\documentclass{article}
\usepackage[utf8]{inputenc}

\usepackage{arxiv}

\RequirePackage{amsthm,amsmath,amsfonts,amssymb}
\RequirePackage{graphicx}

\usepackage{bm}
\usepackage{url}
\usepackage{multirow}
\usepackage{makecell}
\usepackage{tablefootnote}

\usepackage{natbib}
\usepackage{hyperref}

\newtheorem{prop}{Proposition}[section]

\newtheorem{remark}{Remark}[section]

\DeclareMathOperator*{\argmax}{arg\,max}
\DeclareMathOperator*{\argmin}{arg\,min}

\newcommand{\R}{\mathbb{R}}
\newcommand{\E}{\mathbb{E}}
\newcommand{\I}{\bm{I}}
\newcommand{\nv}{\bm{0}}
\newcommand{\Tr}{\text{Tr}}
\newcommand{\Cov}{\text{Cov}}
\newcommand{\Var}{\text{Var}}
\newcommand{\sgn}{\text{\normalfont sign}}
\newcommand{\epsv}{\bm{\varepsilon}}

\newcommand{\smin}{s_{\min}}
\newcommand{\smax}{s_{\max}}

\newcommand{\ns}{n^*}
\newcommand{\vv}{\bm{v}}
\newcommand{\xv}{\bm{x}}
\newcommand{\xsv}{\bm{x^*}}
\newcommand{\xpv}{\bm{x'}}
\newcommand{\xvi}{\bm{x_i}}
\newcommand{\xvj}{\bm{x_j}}
\newcommand{\xsvi}{\bm{x^*_i}}
\newcommand{\X}{\bm{X}}
\newcommand{\Xs}{\bm{X^*}}
\newcommand{\yv}{\bm{y}}
\newcommand{\ytv}{\bm{\tilde{y}}}
\newcommand{\ypv}{\bm{y^{^+}}}
\newcommand{\Ph}{\bm{\Phi}}
\newcommand{\Phs}{\bm{\Phi^*}}
\newcommand{\phiv}{\bm{\varphi}}
\newcommand{\K}{\bm{K}}
\newcommand{\kv}{\bm{k}}
\newcommand{\Ks}{\bm{K^*}}
\newcommand{\Kss}{\bm{K^{**}}}

\newcommand{\Ih}{\bm{\hat{I}}}
\newcommand{\Sigmaalpha}{\bm{\Sigma_\alpha}}

\newcommand{\thetahv}{\bm{\hat{\theta}}}
\newcommand{\thetanv}{\bm{\theta_0}}

\newcommand{\alphah}{\hat{\alpha}}
\newcommand{\alphav}{\bm{\alpha}}
\newcommand{\alphahv}{\bm{\hat{\alpha}}}
\newcommand{\dalphahv}{\bm{\Delta\hat{\alpha}}}
\newcommand{\alphanv}{\bm{\alpha_0}}
\newcommand{\alphahnv}{\bm{\hat{\alpha}}^0}
\newcommand{\alphahfv}{\bm{\hat{\alpha}}_\text{\normalfont KGF}}
\newcommand{\alphahri}{\hat{\alpha}_i^\text{\normalfont KRR}}
\newcommand{\alphahfi}{\hat{\alpha}_i^\text{\normalfont KGF}}
\newcommand{\alphahni}{\hat{\alpha}_i^0}
\newcommand{\alphahrv}{\bm{\hat{\alpha}}_\text{\normalfont KRR}}
\newcommand{\alphahsgfi}{\hat{\alpha}^\text{SGF}_i}
\newcommand{\alphahsgfv}{\bm{\hat{\alpha}}^\text{SGF}}
\newcommand{\alphahinfi}{\hat{\alpha}^\infty_i}
\newcommand{\alphahinfv}{\bm{\hat{\alpha}}^\infty}

\newcommand{\betah}{\hat{\beta}}
\newcommand{\betav}{\bm{\beta}}
\newcommand{\betahv}{\bm{\hat{\beta}}}
\newcommand{\betahsgfi}{\hat{\beta}^\text{SGF}_i}
\newcommand{\betahsgfv}{\bm{\hat{\beta}}^\text{SGF}}
\newcommand{\betahinfi}{\hat{\beta}^\infty_i}
\newcommand{\betahinfv}{\bm{\hat{\beta}}^\infty}

\newcommand{\fh}{\hat{f}}

\newcommand{\fv}{\bm{f}}
\newcommand{\fhv}{\bm{\hat{f}}}
\newcommand{\fsv}{\bm{f^*}}
\newcommand{\fhsv}{\bm{\hat{f}^*}}
\newcommand{\fnv}{\bm{f_0}}
\newcommand{\fhn}{\hat{f}^0}
\newcommand{\fhnv}{\bm{\hat{f}}^0}
\newcommand{\fhf}{\hat{f}_\text{\normalfont KGF}}
\newcommand{\fhfv}{\bm{\hat{f}}_\text{\normalfont KGF}}
\newcommand{\fhr}{\hat{f}_\text{\normalfont KRR}}
\newcommand{\fhrv}{\bm{\hat{f}}_\text{\normalfont KRR}}
\newcommand{\fpv}{\bm{f^{^+}}}
\newcommand{\fhpv}{\bm{\hat{f}^{^+}}}

\newcommand{\etahv}{\bm{\hat{\eta}}}
\newcommand{\etahnv}{\bm{\hat{\eta}_0}}

\newcommand{\U}{\bm{U}}

\newcommand{\Ss}{\bm{S}}
\newcommand{\A}{\bm{A}}
\newcommand{\B}{\bm{B}}
\newcommand{\C}{\bm{C}}
\newcommand{\D}{\bm{D}}
\newcommand{\DK}{\bm{\Delta_K}}

\newcommand{\Hh}{\mathcal{H}}
\newcommand{\Hhh}{\bm{H}}

\newcommand{\fhp}{\hat{f}^{^+}}
\newcommand{\fp}{f^{^+}}
\newcommand{\yp}{y^{^+}}
\newcommand{\fhpinfv}{\bm{\hat{f}}^{\bm{+}\infty}}
\newcommand{\fhpsgfv}{\bm{\hat{f}}^{\bm{+}\text{SGF}}}
\newcommand{\fhpinfi}{\hat{f}^{+\infty}_i}
\newcommand{\fhpsgfi}{\hat{f}^{+\text{SGF}}_i}

\title{Fast Robust Kernel Regression through Sign Gradient Descent with Early Stopping}
\author{%
  Oskar Allerbo \\
  Department of Mathematics\\
  KTH Royal Institute of Technology\\
  \texttt{oallerbo@kth.se} \\
}

\begin{document}
\maketitle

\begin{abstract}
Kernel ridge regression, KRR, is a generalization of linear ridge regression that is non-linear in the data, but linear in the model parameters. 
Here, we introduce an equivalent formulation of the objective function of KRR, which opens up for replacing the ridge penalty with the $\ell_\infty$ and $\ell_1$ penalties.

Using the $\ell_\infty$ and $\ell_1$ penalties, we obtain robust and sparse kernel regression, respectively. 
We study the similarities between explicitly regularized kernel regression and the solutions obtained by early stopping of iterative gradient-based methods, where we connect $\ell_\infty$ regularization to sign gradient descent, $\ell_1$ regularization to forward stagewise regression (also known as coordinate descent), and $\ell_2$ regularization to gradient descent, and, in the last case, theoretically bound for the differences.
We exploit the close relations between $\ell_\infty$ regularization and sign gradient descent, and between $\ell_1$ regularization and coordinate descent to propose computationally efficient methods for robust and sparse kernel regression.

We finally compare robust kernel regression through sign gradient descent to existing methods for robust kernel regression on five real data sets, demonstrating that our method is one to two orders of magnitude faster, without compromised accuracy.
\end{abstract}


\section{Introduction}
Kernel ridge regression, KRR, is a generalization of linear ridge regression that is non-linear in the data, but linear in the model parameters. As for linear ridge regression, KRR has a closed-form solution, but at the cost of solving a system of $n$ linear equations, where $n$ is the number of training observations. The KRR estimate coincides with the posterior mean of kriging, or Gaussian process regression, \citep{krige1951statistical,matheron1963principles} and has successfully been applied within a wide range of applications \citep{zahrt2019prediction, ali2020complete, chen2021optimizing, fan2021well, le2021fingerprinting, safari2021kernel, shahsavar2021experimental, singh2021neural, wu2021increasing, chen2022kernel}.

Robust regression is often implemented by replacing the standard $\ell_2$ loss function with a loss function that is less sensitive to outliers in the data. Robust methods for kernel regression tend to rely either on M-estimators \citep{de2009robustness,wibowo2009robust,debruyne2010robustness,hwang2015robust} or quantile regression \citep{hwang2005simple,takeuchi2006nonparametric,li2007quantile}. In contrast to KRR, these methods are all iterative and thus tend to be computationally heavy, even though more computationally efficient methods for kernel quantile regression have recently been proposed by \citet{zheng2022fast} and \citet{tang2024fastkqr}. 

In the linear case, many alternatives to the ridge penalty have been proposed, including the lasso penalty \citep{tibshirani1996regression}, which is known for creating sparse models. By replacing the ridge penalty of KRR with the lasso penalty \citet{roth2004generalized}, \citet{guigue2005kernel} and \citet{feng2016kernelized} have obtained non-linear regression that is sparse in the observations.
Rather than using the lasso, or $\ell_1$, penalty, \citet{russu2016secure} and \citet{demontis2017infinity} trained kernelized support vector machines with max (or infinity, $\ell_\infty$) norm regularization to obtain models that are robust against adversarial data, which in the context of regression translates to outliers. 

An alternative to explicit regularization is to use an iterative optimization algorithm and to stop the training before the algorithm converges, something that is known as early stopping. A well-known example of this is the connection between the explicitly regularized lasso model and the iterative method forward stagewise regression, which is also known as coordinate descent. The similarities between the solutions are well studied by e.g.\ \citet{efron2004least}, \citet{hastie2007forward} and \citet{tibshirani2015general}. There are also striking similarities between explicitly regularized ridge regression and gradient descent with early stopping \citep{friedman2004gradient,ali2019continuous}. Replacing explicit infinity norm regularization with an iterative optimization method with early stopping does not seem to be as well studied. However, as is shown in this paper, the solutions are similar to those of sign gradient descent with early stopping.

One benefit of replacing explicit regularization with early stopping is that the entire solution path, which consists of all different levels of regularizations between 0 and $\infty$, is obtained by running the iterative optimization algorithm to convergence only once. In contrast, for explicit regularization, the solution has to be calculated once for every considered regularization strength.
In general, no closed-form solutions exist for the lasso and infinity norm regularized problems, and iterative optimization methods have to be run to convergence once for every candidate value of the regularization strength, something that tends to be computationally heavy.

In this paper, we present an equivalent formulation of KRR, which we use to obtain kernel regression with the $\ell_\infty$ and $\ell_1$ norms, in addition to the default, $\ell_2$, norm. We also use the equivalent formulation to solve kernel regression using the three different gradient-based optimization algorithms gradient descent, sign gradient descent, and coordinate descent (forward stagewise regression). We relate each iterative algorithm to one explicitly regularized solution and use gradient-based optimization with early stopping to obtain computationally efficient robust, and sparse, kernel regression. 
Even though robust kernel regression has more evident applications than sparse kernel regression does, since the calculations are analogous for $\ell_\infty$ and $\ell_1$ regularization we include both cases. In the results section, we however focus on robust regression.

The rest of the paper is structured as follows.
In Section \ref{sec:krr_rev}, we review kernel ridge regression, before introducing our equivalent formulation in Section \ref{sec:equ_krr}. In Section \ref{sec:kgf}, we review kernel gradient descent, KGD, and present theoretical comparisons between KGD with infinitesimal step size and KRR. In Section \ref{sec:l1linf}, we generalize KRR into robust and sparse kernel regression by replacing the $\ell_2$ penalty with the $\ell_\infty$ and $\ell_1$ penalties, respectively. Similarly to in Section \ref{sec:kgf}, we also introduce gradient-based optimization algorithms that we relate to the explicitly regularized methods; kernel sign gradient descent for robust kernel regression, and kernel coordinate descent for sparse kernel regression.
Finally, in Section \ref{sec:exps}, we demonstrate our findings with experiments on real and synthetic data.

Our main contributions are listed below.
\begin{itemize}
\item We present an equivalent objective function for KRR and use this formulation to generalize KRR to the $\ell_\infty$ and $\ell_1$ penalties.
\item We theoretically show that kernel sign gradient descent, and kernel coordinate descent, with early stopping, correspond to robust and sparse kernel regression, respectively. We use this to introduce computationally efficient regularization by replacing explicit $\ell_\infty$ and $\ell_1$ regularization by implicit regularization through early stopping.
\item We demonstrate on five real data sets that robust kernel regression through kernel sign gradient descent is one to two orders of magnitude faster than existing robust kernel regression methods, without compromised performance in terms of prediction.
\end{itemize}

All proofs are deferred to Appendix \ref{sec:proofs}.

\section{Review of Kernel Ridge Regression}
\label{sec:krr_rev}
For a positive semi-definite kernel function, $k(\xv,\xpv)\in \R$, and $n$ paired observations, $(\xvi,y_i)_{i=1}^n\in\R^p\times\R$, presented in a design matrix, $\X=[\bm{x_1},\bm{x_2},{\dots} \bm{x_n}]^\top\in\R^{n\times p}$, and a response vector, $\yv \in \R^n$, and for a given regularization strength, $\lambda\geq0$, the objective function of kernel ridge regression, KRR, is given by
\begin{equation}
\label{eq:krr_obj_alpha}
\alphahv=\argmin_{\alphav\in \R^n}\frac12\left\|\yv-\K\alphav\right\|_2^2+\frac\lambda2\|\alphav\|^2_{\K}
\end{equation}
with predictions given by
\begin{equation*}
\begin{bmatrix}\fhv\\\fhsv\end{bmatrix}=\begin{bmatrix}\K\\\Ks\end{bmatrix}\alphahv.
\end{equation*}
Here, $\fhv=\fhv(\X)\in\R^{n}$ and $\fhsv=\fhsv(\Xs,\X)\in\R^{\ns}$ denote model predictions for the training data, $\X$, and new data, $\Xs=[\bm{x^*_1},\bm{x^*_2},{\dots} \bm{x^*_{n^*}}]^\top\in\R^{\ns\times p}$, and $\K=\K(\X)\in \R^{n\times n}$ and $\Ks=\Ks(\Xs,\X)\in \R^{\ns\times n}$ denote two kernel matrices defined according to $\K_{ij}=k(\xvi,\xvj)$ and $\Ks_{ij}=k(\xsvi,\xvj)$. The weighted norm, $\|\vv\|_{\A}$, is defined according to $\|\vv\|_{\A}^2=\vv^\top\A\vv$ for any symmetric positive definite matrix $\A$.

The closed-form solution for $\alphahv$ is given by
\begin{equation}
\label{eq:krr_s_alpha}
\alphahv=\left(\K+\lambda\I\right)^{-1}\yv,
\end{equation}
and consequently
\begin{equation}
\label{eq:krr_s_fs}
\begin{bmatrix}\fhv\\\fhsv\end{bmatrix}=\begin{bmatrix}\K\\\Ks\end{bmatrix}\left(\K+\lambda\I\right)^{-1}\yv.
\end{equation}

An alternative interpretation of KRR is as linear regression for a non-linear feature expansion of $\X$. According to Mercer's Theorem \citep{mercer1909xvi}, every kernel can be written as the inner product of feature expansions of its two arguments: $k(\xv,\xpv)=\phiv(\xv)^\top\phiv(\xpv)$ for $\phiv(\xv)\in \R^q$. Thus, denoting the feature expansions of the design matrix and the new data with $\Ph=\Ph(\X)\in \R^{n\times q}$ and $\Phs=\Phs(\Xs)\in \R^{\ns\times q}$, the two kernel matrices can be expressed as $\K=\Ph\Ph^\top$ and $\Ks=\Phs\Ph^\top$. For $\betav=\Ph^\top\alphav$, Equations \ref{eq:krr_obj_alpha} and \ref{eq:krr_s_fs} become
\begin{equation}
\label{eq:krr_obj_beta}
\begin{aligned}
&\betahv=\argmin_{\betav\in \R^q}\frac12\left\|\yv-\Ph\betav\right\|_2^2+\frac\lambda2\|\betav\|^2_2\\
&\begin{bmatrix}\fhv\\\fhsv\end{bmatrix}=\begin{bmatrix}\Ph\\\Phs\end{bmatrix}\betahv,
\end{aligned}
\end{equation}
which is exactly linear ridge regression for the feature expansion of the kernel.

\section{Equivalent Formulations of Kernel Ridge Regression}
\label{sec:equ_krr}
In this section, we present an equivalent formulation of the objective function in Equation \ref{eq:krr_obj_alpha}, which provides the same solution for $\alphahv$. This formulation opens up for generalizing KRR by using penalties other than the ridge penalty and also provides an interesting connection to functional gradient descent \citep{mason1999boosting}. 

In Appendix \ref{sec:equ_f}, we take this one step further by reformulating Equation \ref{eq:krr_obj_alpha} directly in the model predictions, $[\fv^\top, \fsv{^\top}]^\top$,
by presenting two equivalent objectives, that, when minimized with respect to $[\fv^\top, \fsv{^\top}]^\top$, generate the solution in Equation \ref{eq:krr_s_fs}. In this and the following sections, all calculations are done with respect to $\alphav$, and the corresponding expressions for $[\fv^\top, \fsv{^\top}]^\top$ are obtained through multiplication with $[\K^\top,\Ks{^\top}]^\top$, but, as shown in the appendix, the expressions for $[\fv^\top, \fsv{^\top}]^\top$ can be obtained directly without taking the detour over $\alphav$.

In Proposition \ref{thm:equ_alpha}, we show how we can move the weighted norm in Equation \ref{eq:krr_obj_alpha} from the penalty term to the reconstruction term. 

\begin{prop}
\label{thm:equ_alpha}
\begin{subequations}
\label{eq:equ_alpha}
\begin{align}
\alphahv=&\argmin_{\alphav\in \R^n}\frac12\left\|\yv-\K\alphav\right\|_2^2+\frac\lambda2\|\alphav\|^2_{\K}\label{eq:equ_alpha1}\\
=&\argmin_{\alphav\in \R^n}\frac12\left\|\yv-\K\alphav\right\|_{\K^{-1}}^2+\frac\lambda2\|\alphav\|^2_2\label{eq:equ_alpha2}\\
=&\left(\K+\lambda\I\right)^{-1}\yv.\notag
\end{align}
\end{subequations}
\end{prop}
\begin{remark}Despite the appearance of $\K^{-1}$ in Equation \ref{eq:equ_alpha2}, this inverse never needs to be calculated, since in all calculations, $\K^{-1}$ cancels through multiplication by $\K$.
\end{remark}

The alternative formulation in Equation \ref{eq:equ_alpha2}, where the reconstruction term, rather than the regularization term, is weighted by the kernel matrix, has two interesting implications. First, a standard ridge penalty on the regularization term opens up for using other penalties than the $\ell_2$ norm, such as the $\ell_1$ and  $\ell_\infty$ norms. Although these penalties have previously been used in the kernel setting, it has been done by replacing $\|\alphav\|_{\K}^2$ in Equation \ref{eq:krr_obj_alpha} by $\|\alphav\|_1$ or $\|\alphav\|_\infty$, hence acting as if the objective of KRR were to minimize $\|\yv-\K\alphav\|_2^2+\lambda\|\alphav\|_2^2$, and losing the connection to linear regression in feature space.

Second, the gradient of Equation \ref{eq:equ_alpha2} with respect to $\alphav$ is
\begin{equation}
\label{eq:kgd_grad_alpha}
\K\alphav-\yv+\lambda\alphav.
\end{equation}

Multiplying by $[\K^\top,\Ks{^\top}]^\top$, we obtain
\begin{equation}
\label{eq:kgd_grad_f}
\begin{bmatrix}\K\\\Ks\end{bmatrix}(\fv-\yv)+\lambda\begin{bmatrix}\fv\\\fsv\end{bmatrix},
\end{equation}
which is the gradient used in functional gradient descent. Functional gradient descent is generally derived by differentiating functionals. Here, however, the expression is a simple consequence of the equivalent objective function of KRR.

\section{Kernel Gradient Descent and Kernel Gradient Flow}
\label{sec:kgf}
Before replacing the Euclidean norm in Equation \ref{eq:equ_alpha2} with the $\ell_\infty$ and $\ell_1$ norms in Section \ref{sec:l1linf}, we investigate solving kernel regression iteratively with gradient descent. We also use gradient descent with infinitesimal step size, known as gradient flow, to obtain a closed-form solution which we use for direct comparisons to kernel ridge regression.

The similarities between ridge regression and gradient descent with early stopping are well studied for linear regression \citep{friedman2004gradient,ali2019continuous, allerbo2022elastic}. When starting at zero, optimization time can be thought of as an inverse penalty, where longer optimization time corresponds to weaker regularization. When applying gradient descent to kernel regression, something we refer to as kernel gradient descent, KGD, we replace explicit regularization with implicit regularization through early stopping. That is, we use $\lambda=0$ and consider training time, $t$, as a regularizer.

With $\lambda=0$ in Equation \ref{eq:kgd_grad_alpha}, starting at $\nv$, the KGD update becomes, for step size $\eta$,
\begin{equation}
\label{eq:kgd_update}
\alphahv_{k+1}=\alphahv_k+\eta\cdot\left(\yv-\K\alphahv_k\right),\ \alphahv_0=\nv.
\end{equation}

To compare the regularization injected by early stopping to that of ridge regression, we let the optimization step size go to zero to obtain a closed-form solution, which we refer to as kernel gradient flow, KGF. Then, Equation \ref{eq:kgd_update} can be thought of as the Euler forward formulation of the differential equation in Equation \ref{eq:kgd_diff_eq},
\begin{equation}
\label{eq:kgd_diff_eq}
\frac{d \alphahv(t)}{d t}=\yv-\K\alphahv(t),\ \alphahv(0)=\nv,
\end{equation}
the solution to which is given by Equation \ref{eq:kgf_s}.
\begin{equation}
\label{eq:kgf_s}
\alphahfv(t)=(\I-\exp(-t\K))\K^{-1}\yv,
\end{equation}
where $\exp$ denotes the matrix exponential.
\begin{remark}Note that $(\I-\exp(-t\K))\K^{-1}=\K^{-1}(\I-\exp(-t\K))$ is well-defined even if $\K$ is singular. The matrix exponential is defined through its Taylor expansion and from $\I-\exp(-t\K)=t\K-\frac1{2!}t^2\K^2+\dots$, a matrix $\K$ factors out, that cancels $\K^{-1}$.
\end{remark}
\begin{remark}
\label{rm:nest}
It is possible to generalize KGF to Nesterov accelerated gradient descent with momentum \citep{nesterov1983method,polyak1964some}. In this case $\exp(-t\K)$ in Equation \ref{eq:kgf_s} generalizes to $\exp\left(-\frac t{1-\gamma}\K\right)$, where $\gamma\in [0,1)$ is the strength of the momentum. See Appendix \ref{sec:proofs} for details.
\end{remark}

To facilitate the comparisons between KGF and KRR, we rewrite Equation \ref{eq:krr_s_alpha} as
\begin{equation}
\label{eq:krr_equ}
\alphahrv(\lambda)=\left(\I-\left(\I+1/\lambda\cdot\K\right)^{-1}\right)\K^{-1}\yv,
\end{equation}
where we have used that 
$$\left(\K+\lambda\I\right)^{-1}=\left(\I-\left(\I+1/\lambda\cdot\K\right)^{-1}\right)\K^{-1},$$
which is a consequence of the Woodbury matrix identity.

Since $\exp(-t\K)=\exp(t\K)^{-1}$, the KGF and KRR solutions differ only in the factor
\begin{equation*}
\label{eq:exp_vs_lin}
\exp(t\K) \text{ v.s. } \I+1/\lambda\K.
\end{equation*}
Thus, for $t=1/\lambda$, we can think of the ridge penalty as a first-order Taylor approximation of the gradient flow penalty.

\subsection{Comparisons between Kernel Ridge Regression and Kernel Gradient Flow}
In this section, we compare the KRR and KGF solutions for $\lambda=1/t$.
To do this, we introduce the following notation, where $\kv(\xsv)^\top=\kv(\xsv,\X)^\top\in\R^n$ is the row in $\Ks$ corresponding to $\xsv$:
\begin{equation*}
\begin{aligned}
&\fhfv(\X,t):=\K\alphahfv(t),\quad \fhf(\xsv,t):=\kv(\xsv)^\top\alphahfv(t),\\
&\fhrv(\X,\lambda):=\K\alphahrv(\lambda),\quad \fhr(\xsv,\lambda):=\kv(\xsv)^\top\alphahrv(\lambda),\\
&\alphahnv:=\alphahfv(t=\infty)=\alphahrv(\lambda=0)=\K^{-1}\yv,\\
&\fhnv(\X):=\fhfv(\X,t=\infty)=\fhrv(\X,\lambda=0)=\K\alphahnv=\yv,\\
&\fhn(\xsv):=\fhf(\xsv,t=\infty)=\fhr(\xsv,\lambda=0)=\kv(\xsv)^\top\alphahnv=\kv(\xsv)^\top\K^{-1}\yv,\\
&\alphahfi:=(\alphahfv)_i, \quad \alphahri:=(\alphahrv)_i,\quad \alphahni:=(\alphahnv)_i.
\end{aligned}
\end{equation*}
Even though not explicitly stated, all estimates depend on the training data $(\X,\yv)$.
We further use the notation $\bm{\xi}\sim (\bm{\mu},\bm{\Sigma})$ do denote that the random variable $\bm{\xi}$ follows a distribution with $\E(\bm{\xi})=\bm{\mu}$ and Cov$(\bm{\xi})=\bm{\Sigma}$.

In Proposition \ref{thm:kgf_krr_diff}, we bound the differences between the KGF and KRR solutions in terms of the non-regularized solutions, $\alphanv$, $\fhnv(\X)$ and $\fhn(\xsv)$. In parts \emph{(a)} and \emph{(b)}, where we bound the differences between the parameter and in-sample prediction vectors, no further assumptions are needed. In parts \emph{(c)} and \emph{(d)}, with bounds on individual parameters and predictions, including out-of-sample predictions, some very reasonable assumptions have to be made on the data. For all four bounds, the larger the non-regularized value is, the larger the difference between the KGF and KRR estimates is allowed to be.
\begin{prop}~\\
\label{thm:kgf_krr_diff}
For $t\geq 0$, $\yv\in\R^n$,
\begin{itemize}
\item[(a)]
$\left\|\alphahfv(t)-\alphahrv(1/t)\right\|_2^2\leq 0.0415\cdot\|\alphahnv\|_2^2$,
\item[(b)]
$\left\|\fhfv(\X,t)-\fhrv(\X,1/t)\right\|_2^2\leq 0.0415\cdot\|\fhnv(\X)\|_2^2=0.0415\cdot\|\yv\|_2^2$.
\end{itemize}
For $t\geq 0$, $\yv=\K\alphanv+\epsv,\ \alphanv\sim (\nv, \Sigmaalpha),\ \epsv\sim (\nv, \sigma_\varepsilon^2\I)$, where $\Sigmaalpha$ and $\K$ are simultaneously diagonalizable, 
\begin{itemize}
\item[(c)]
$\E_{\epsv,\alphanv}\left(\left(\alphahfi(t)-\alphahri(1/t)\right)^2\right)\leq0.0415\cdot\E_{\epsv,\alphanv}\left((\alphahni)^2\right)$, $i=1,2,\dots n$,
\item[(d)]
$\E_{\epsv,\alphanv}\left(\left(\fhf(\xsv,t)-\fhr(\xsv,1/t)\right)^2\right)\leq0.0415\cdot\E_{\epsv,\alphanv}\left(\fhn(\xsv)^2\right)$.
\end{itemize}
\end{prop}
Two typical options for $\Sigmaalpha$ are $\sigma_{\alpha}^2\I$ and $\sigma_\beta^2\K^{-1}$, where the second formulation implies $\bm{\beta_0}\sim (\nv,\sigma_\beta^2\I)$ in the feature space formulation of KRR from Equation \ref{eq:krr_obj_beta}.

In Proposition \ref{thm:kgf_krr_y}, we compare the distances to the observation vector, $\yv$, of the KGF and KRR solutions. According to parts \emph{(a)} and \emph{(c)}, for a given regularization, the furthest possible (among all values of $\yv$, or expected $y_i^2$) normalized KGF solution lies closer to the observations than the furthest possible normalized KRR solution does. Analogously, according to parts \emph{(b)} and \emph{(d)}, for a given regularization, the furthest possible KRR solution lies closer to zero than the furthest possible KGF solution does.
These results agree with those of \citet{richards2021comparing}, who state that, for linear regression, when the signal-to-noise ratio is high, gradient descent tends to outperform ridge regression: For low noise, a solution that lies closer to the observations, as provided by KGF, is beneficial.
\begin{prop}~\\
\label{thm:kgf_krr_y}
For $t\geq 0$, $\yv\in \R^n$,
\begin{equation*}
\begin{aligned}
&\text{(a)} && \max_{\yv\neq \nv}\left(\frac{\left\|\fhfv(\X,t)-\yv\right\|_2}{\left\|\yv\right\|_2}\right)\leq\max_{\yv\neq \nv}\left(\frac{\left\|\fhrv(\X,1/t)-\yv\right\|_2}{\left\|\yv\right\|_2}\right),\\
&\text{(b)} && \max_{\yv\neq \nv}\left(\frac{\left\|\fhrv(\X,1/t)\right\|_2}{\left\|\yv\right\|_2}\right)\leq\max_{\yv\neq \nv}\left(\frac{\left\|\fhfv(\X,t)\right\|_2}{\left\|\yv\right\|_2}\right).
\end{aligned}
\end{equation*}
For $t\geq 0$, $\yv=\K\alphanv+\epsv,\ \alphanv\sim (\nv, \Sigmaalpha),\ \epsv\sim (\nv, \sigma_\varepsilon^2\I)$, where $\Sigmaalpha$ and $\K$ are simultaneously diagonalizable, and $i=1,2,\dots n$,
\begin{equation*}
\begin{aligned}
&\text{(c)} && \max_{\E_{\epsv,\alphanv}(y_i^2)\neq 0}\left(\frac{\E_{\epsv,\alphanv}\left(\left(\fhf(\xvi,t)-y_i\right)^2\right)}{\E_{\epsv,\alphanv}\left(y_i^2\right)}\right)\\
& &&\leq\max_{\E_{\epsv,\alphanv}(y_i^2)\neq 0}\left(\frac{\E_{\epsv,\alphanv}\left(\left(\fhr(\xvi,1/t)-y_i\right)^2\right)}{\E_{\epsv,\alphanv}\left(y_i^2\right)}\right),\\
&\text{(d)} && \max_{\E_{\epsv,\alphanv}(y_i^2)\neq 0}\left(\frac{\E_{\epsv,\alphanv}\left(\fhr(\xvi,1/t)^2\right)}{\E_{\epsv,\alphanv}\left(y_i^2\right)}\right)\\
& &&\leq\max_{\E_{\epsv,\alphanv}(y_i^2)\neq 0}\left(\frac{\E_{\epsv,\alphanv}\left(\fhf(\xvi,t)^2\right)}{\E_{\epsv,\alphanv}\left(y_i^2\right)}\right).
\end{aligned}
\end{equation*}
\end{prop}

Finally, if we assume that a true function exists, parameterized by the true $\alphanv$ as
$$\begin{bmatrix}\fnv(\X)\\\fnv(\Xs)\end{bmatrix}=\begin{bmatrix}\K\\\Ks\end{bmatrix}\alphanv,$$
and with observations according to $\yv=\fnv(\X)+\epsv$, we can calculate the expected squared differences between the estimated and true models, something that is often referred to as the risk: $\text{Risk}(\bm{\hat{\theta}};\bm{\theta_0})=\E\left(\|\bm{\hat{\theta}}-\bm{\theta_0}\|_2^2\right)$. In Proposition \ref{thm:kgf_krr_risk}, which is an adaptation of Theorems 1 and 2 by \citet{ali2019continuous}, the KGF estimation and prediction risks are bounded in terms of the corresponding KRR risks. In all three cases, the KGF risk is less than 1.69 times the KRR risk.

\begin{prop}~\\
\label{thm:kgf_krr_risk}
For $t\geq 0$, $\yv=\fnv(\X)+\epsv=\K\alphanv+\epsv$, $\epsv\sim(\bm{0},\sigma_\varepsilon^2\I)$,
\begin{itemize}
\item[(a)]
For any $\alphanv$, $\text{\emph{Risk}}(\alphahfv(t);\alphanv)\leq 1.6862\cdot\text{\emph{Risk}}(\alphahrv(1/t);\alphanv)$.
\item[(b)]
For any $\alphanv$,\\ $\text{\emph{Risk}}(\fhfv(\X,t);\fnv(\X))\leq 1.6862\cdot\text{\emph{Risk}}(\fhrv(\X,1/t);\fnv(\X))$.
\item[(c)]
For $\alphanv\sim \left(\nv,\Sigmaalpha\right)$, where $\Sigmaalpha$ and $\K$ are simultaneously diagonalizable,\\$\E_{\alphanv}\left(\text{\emph{Risk}}(\fhf(\xsv,t);f_0(\xsv))\right)\\\leq 1.6862\cdot\E_{\alphanv}\left(\text{\emph{Risk}}(\fhr(\xsv,1/t);f_0(\xsv))\right)$, 
where $f_0(\xsv)\in \fnv(\Xs)$.
\end{itemize}
\end{prop}

To summarize this section, the KGF and KRR solutions tend to agree well both in terms of the parameter and prediction vectors and in terms of the risks. Which algorithm actually performs better, depends on the specific data, sometimes KGF performs slightly better than KRR and sometimes vice versa. Proposition \ref{thm:kgf_krr_y} suggests that the KGF solution lies closer to the observations than KRR solution does, which can often be an advantage, however not in the presence of extreme outliers. These conclusions are supported by the experiments in Section \ref{sec:exps}.

\section{Kernel Regression with the \texorpdfstring{$\ell_\infty$}{l1} and \texorpdfstring{$\ell_1$}{linf} Norms}
\label{sec:l1linf}
In this section, we replace the squared $\ell_2$ norm of KRR in Equation \ref{eq:equ_alpha2} by the $\ell_\infty$ and $\ell_1$ norms, respectively, to obtain $\ell_\infty$ and $\ell_1$ regularized kernel regression, which we abbreviate as K$\ell_\infty$R and K$\ell_1$R, respectively. We also connect the explicitly regularized algorithms K$\ell_\infty$R and K$\ell_1$R to kernel sign gradient descent, KSGD, and kernel coordinate descent, KCD, with early stopping, similarly to how we related KRR and KGF in Section \ref{sec:kgf}. 
The six algorithms and their abbreviations are stated in Table \ref{tab:abbrs}.

\begin{table}
\caption{Kernel regression algorithms}
\center
\begin{tabular}{|l|l|}
\hline
Algorithm & Abbreviation\\
\hline
Kernel ridge regression ($\ell_2$ penalty) & KRR\\
Kernel regression with the $\ell_\infty$ penalty (robust) & K$\ell_\infty$R\\
Kernel regression with the $\ell_1$ penalty (sparse) & K$\ell_1$R\\
Kernel gradient descent & KGD\\
Kernel sign gradient descent (robust) & KSGD\\
Kernel coordinate descent (sparse) & KCD\\
\hline
\end{tabular}
\label{tab:abbrs}
\end{table}

The objective functions  for K$\ell_\infty$R and K$\ell_1$R are, respectively
\begin{equation}
\label{eq:alphainf}
\argmin_{\alphav\in \R^n}\frac12\left\|\yv-\K\alphav\right\|_{\K^{-1}}^2+\lambda\|\alphav\|_\infty
\end{equation}
and
\begin{equation}
\label{eq:alpha1}
\argmin_{\alphav\in \R^n}\frac12\left\|\yv-\K\alphav\right\|_{\K^{-1}}^2+\lambda\|\alphav\|_1.
\end{equation}
In contrast to KRR, unless the data is uncorrelated, no closed-form solutions exist for Equations \ref{eq:alphainf} and \ref{eq:alpha1}. However, the problems are still convex (strictly convex if $\K$ is strictly positive definite, which is the case for all kernels with infinite-dimensional feature expansions) and solutions can be obtained using the iterative optimization method proximal gradient descent \citep{rockafellar1976monotone}, which, in contrast to standard gradient descent, can handle the discontinuities of the gradients of the $\ell_\infty$ and $\ell_1$ norms.

Applying $\ell_\infty$ regularization, which penalizes the largest parameter(s), promotes a solution with no extreme parameter values. Thus, Equation \ref{eq:alphainf} promotes a solution where the impact of extreme observations is alleviated, which might be beneficial if outliers are present in the data.
Vice versa, $\ell_1$, or lasso, regularization promotes a sparse solution, where parameters are added sequentially to the model, with the most significant parameters included first. 
We further demonstrate this in Section \ref{sec:exp_toy}.

\subsection{Regularization through Early Stopping}
\label{sec:early_stopping}
The regularization strength must be carefully selected for the model to generalize well on new data. While a too large regularization results in poor performance due to a too simple model, a too small regularization might result in a too complex model that incorporates the noise in the training data, and thus generalizes poorly to new data (overfitting). This is sometimes referred to as the bias-variance tradeoff. In practice, the regularization strength is usually selected by some sort of cross-validation, where parts of the data are kept aside during training and then used for validation once the model is trained. By training multiple models, all with different regularization strengths, one can select the model that performs best on the validation data.

However, applying the proximal gradient descent algorithm might be computationally heavy, especially when evaluating several different regularization strengths. For some classes of models, there exist computationally efficient versions of leave-one-out and generalized cross-validation, but, as demonstrated in Appendix \ref{sec:fast_cv}, these do not perform well in this setting. Instead, we use another approach: In Section \ref{sec:kgf}, we showed how the solution of KGF with early stopping is very similar to that of KRR, with a later stopping time, $t$, corresponding to a smaller regularization strength, $\lambda$. Thus, running KGD until convergence once, all different stopping times, or regularization strengths, between $t=0$ (corresponding to $\lambda=\infty$) and $t=t_{\max}$ (corresponding to $\lambda=0$) are obtained through this single execution of the algorithm. 
(For more detailed descriptions of the connections between explicit regularization and early stopping, see the works by e.g.\ \citet{friedman2004gradient}, \citet{tibshirani2015general} or \citet{allerbo2022elastic}.)
Analogously, if we can find iterative optimization algorithms that correspond to the $\ell_\infty$ and $\ell_1$ penalties, instead of solving the problem using proximal gradient descent once for each value of $\lambda$, all different early stopping regularization strengths could be obtained by one single call of that algorithm. 

The similarities between $\ell_1$ regularization (lasso) and the iterative optimization method forward stagewise regression (also known as coordinate descent, or, for infinitesimal step size, coordinate flow) are well studied, for instance by \citet{efron2004least} and \citet{hastie2007forward}. \citet{efron2004least} show that under certain conditions, including for uncorrelated data, the solutions paths of lasso and coordinate flow coincide, but even when these conditions do not hold, the solution paths tend to be very similar, where smaller correlations lead to larger similarities.
Similarly, by comparing Equations \ref{eq:kgf_s} and \ref{eq:krr_equ}, or Equations 6 and 10 by \citet{allerbo2022elastic} for standard linear regression, we note that also for $\ell_2$ regularization and gradient flow, the solution paths exactly coincide for uncorrelated data.

The topic of a gradient-based algorithm similar to $\ell_\infty$ regularization does not seem to be as well studied. However, in Proposition \ref{thm:sgn_inf}, we show that the solutions of $\ell_\infty$ regularization and sign gradient flow coincide for uncorrelated data. This is shown using the constrained form of regularized regression, which, due to Lagrangian duality, is equivalent to the penalized form.

Before stating the proposition, we first review the three gradient-based algorithms and their corresponding infinitesimal step-size versions.
For a loss function, $L(\cdot)$, that quantifies the reconstruction error, we denote the gradient and its maximum component at time step $k$ as
\begin{equation*}
\begin{aligned}
\bm{g}_k&:=\frac{\partial L(\bm{\theta}_k;\X,\yv)}{\partial \bm{\theta}_k}\\
m_k&:=\argmax_d(|\bm{g}_k|)_d,
\end{aligned}
\end{equation*}
where the absolute value of the gradient is evaluated element-wise.
The update rules of gradient descent/flow, sign gradient descent/flow, and coordinate descent/flow are stated in Equation \ref{eq:gd_updates}, where $\eta$ denotes the optimization step size, and where $\bm{I_m}(t)$ is a diagonal matrix that defines which gradient component(s) that is updated at time $t$ (for details about $\bm{I_m}$, see \citet{allerbo2022elastic}).
\begin{equation}
\label{eq:gd_updates}
\begin{aligned}
\text{Gradient Descent: }& &\bm{\theta}_{k+1}&=\bm{\theta}_k-\eta\cdot\bm{g}_k\\
\text{Gradient Flow: }& &\frac{\partial \bm{\theta}(t)}{\partial t}&=-\bm{g}(t)\\
\text{Sign Gradient Descent: }& &\bm{\theta}_{k+1}&=\bm{\theta}_k-\eta\cdot\sgn\left(\bm{g}_k\right)\\
\text{Sign Gradient Flow: }& &\frac{\partial \bm{\theta}(t)}{\partial t}&=-\sgn(\bm{g}(t))\\
\text{Coordinate Descent: }& &\left(\bm{\theta}_{k+1}\right)_{m_k}&=\left(\bm{\theta}_k\right)_{m_k}-\eta\cdot\sgn\left(\bm{g}_k\right)_{m_k}\\
\text{Coordinate Flow: }& &\frac{\partial \bm{\theta}(t)}{\partial t}&=-\bm{I_m}(t)\cdot\sgn(\bm{g}(t)).\\
\end{aligned}
\end{equation}
\begin{remark}
For coordinate descent (flow), in each iteration (at each time), only the coordinate(s) corresponding to the maximal absolute gradient value is updated.
\end{remark}
\begin{remark}
The name ``coordinate descent'' is sometimes also used for other, related, algorithms.
\end{remark}

\begin{prop}~\\
\label{thm:sgn_inf}
\begin{enumerate}
\item[(a)]
Let $\betahinfv(c)$ denote the solution to
$$\min_{\betav\in\R^p}\left\|\yv-\X\betav\right\|_{2}^2 \text{ s.t. } \|\betav\|_\infty\leq c,$$
and let $\betahsgfv(t)$ denote the solution to
$$\frac{d\betahv(t)}{dt}=-\sgn\left(\frac{\partial}{\partial\betahv(t)}\left(\left\|\yv-\X\betahv(t)\right\|_{2}^2\right)\right),\ \betahv(0)=\nv.$$
When $\X^\top\X$ is diagonal, with elements $\{s_{ii}\}_{i=1}^p$, the two solutions decompose element-wise and coincide for $c=t$:
$$\betahinfi(t)=\betahsgfi(t)=\sgn((\X^\top\yv)_i/s_{ii})\cdot\min(t,|(\X^\top\yv)_i/s_{ii}|).$$
\item[(b)]
Let $\alphahinfv(c)$ denote the solution to
$$\min_{\alphav\in\R^n}\left\|\yv-\K\alphav\right\|_{\K^{-1}}^2 \text{ s.t. } \|\alphav\|_\infty\leq c,$$
and let $\alphahsgfv(t)$ denote the solution to
$$\frac{d\alphahv(t)}{dt}=-\sgn\left(\frac{\partial}{\partial\alphahv(t)}\left(\left\|\yv-\K\alphahv(t)\right\|_{\K^{-1}}^2\right)\right),\ \alphahv(0)=\nv.$$
When $\K$ is diagonal, with elements $\{k_{ii}\}_{i=1}^n$, the two solutions decompose element-wise and coincide for $c=t$:
$$\alphahinfi(t)=\alphahsgfi(t)=\sgn(y_i/k_{ii})\cdot\min(t,|y_i/k_{ii}|).$$
\end{enumerate}
\end{prop}
\begin{remark}
For infinitesimal step size, sign gradient descent, Adam \citep{kingma2014adam}, and RMSProp \citep{hinton2012neural} coincide, up to constants of order $10^{-8}$ added for numerical stability \citep{balles2018dissecting}. Thus, in the uncorrelated case, $\ell_\infty$ regularization also coincides with Adam and RMSProp with infinitesimal step size.
\end{remark}
\begin{remark}
Similarly to ridge regression and gradient flow, and lasso and coordinate flow, the exact coincidence of $\ell_\infty$ regularization and sign gradient flow only holds for uncorrelated data. However, due to continuity, unless the correlations are very large, the solution paths still tend to be very similar.
\end{remark}

\subsection{Properties of the Early Stopping Algorithms}
One might wonder whether the fact that the exact coincidence of the solution paths for explicit and early stopping regularization holds only for uncorrelated data, implies that the properties of the explicitly regularized methods do not transfer to the early stopping methods on uncorrelated data.
Or is it still the case that KSGD, just like K$\ell_\infty$R, provides a solution that is robust to outliers and that KCD, like K$\ell_1$R, promotes a sparse solution that includes the most significant observations first? In Proposition \ref{thm:other_objs}, we show that this is indeed the case by using the equivalent interpretation of kernel regression as linear regression in feature space from Equation \ref{eq:krr_obj_beta}.

\begin{prop}~\\
\label{thm:other_objs}
Solving $\min_{\alphav}\left\|\yv-\K\alphav\right\|_{\K^{-1}}^2$ for $\alphav$ with 
\begin{itemize}
\item gradient descent, is equivalent of solving $\min_{\betav}\left\|\yv-\Ph\betav\right\|_2^2$ for $\betav$ with gradient descent.
\item sign gradient descent, is equivalent of solving $\min_{\betav}\left\|\yv-\Ph\betav\right\|_1$ for $\betav$ with gradient descent.
\item coordinate descent, is equivalent of solving $\min_{\betav}\left\|\yv-\Ph\betav\right\|_\infty$ for $\betav$ with gradient descent.
\end{itemize}
\end{prop}
According to Proposition \ref{thm:other_objs}, KSGD and KCD correspond to feature space regression with the $\ell_1$ and $\ell_\infty$ loss functions, respectively. (Note that while we previously discussed different norms for the \emph{penalty term}, Proposition \ref{thm:other_objs} considers different norms for \emph{loss function}.)
Regression with the $\ell_1$ loss function, which is known as least absolute deviations, is a commonly used method for robust regression.
Vice versa, regression with the $\ell_\infty$ loss function minimizes the maximum residual, i.e.\ more extreme observations contribute more to the solution than less extreme observations do. 
~\\

Even for convex problems, a too large optimization step size will cause the solution to diverge. In Proposition \ref{thm:conv}, we present upper bounds for the step sizes, which guarantee that the solutions do not diverge.
\begin{prop}~\\
\label{thm:conv}
When solving $\min_{\alphav}\left\|\yv-\K\alphav\right\|_{\K^{-1}}^2$ for $\alphav$ with 
\begin{itemize}
\item 
gradient descent, i.e.\ when $\alphahv_{k+1}=\alphahv_k+\eta\cdot\left(\yv-\K\alphahv_k\right)$, then 
$$\eta\leq \frac2{\smax(\K)} \implies \|\yv-\K\alphahv_{k+1}\|_2\leq\|\yv-\K\alphahv_k\|_2,$$
where $\smax(\K)$ denotes the largest singular value of $\K$.
\item
sign gradient descent, i.e.\ when $\alphahv_{k+1}=\alphahv_k+\eta\cdot\sgn\left(\yv-\K\alphahv_k\right)$, then
$$\eta\leq \min_{1\leq i\leq n}\left(\left|\yv-\K\alphahv_k\right|_i\right) \implies \|\yv-\K\alphahv_{k+1}\|_1\leq\|\yv-\K\alphahv_k\|_1,$$
where the absolute value is taken element-wise.
\item 
coordinate descent, i.e.\ when $\left(\alphahv_{k+1}\right)_{m_k}=\left(\alphahv_k\right)_{m_k}+\eta\cdot\sgn\left(\yv-\K\alphahv_k\right)_{m_k}$, then
$$\eta\leq \min_{1\leq i\leq n}\left(\left|\yv-\K\alphahv_k\right|_i\right) \implies \|\yv-\K\alphahv_{k+1}\|_\infty\leq\|\yv-\K\alphahv_k\|_\infty,$$
where $m_k$ denotes the coordinate with the largest absolute value, and where the absolute value is taken element-wise.
\end{itemize}
\end{prop}
\begin{remark}
Note that depending on which algorithm is used, the convergence is obtained for different norms, similar to in Proposition \ref{thm:other_objs}.
\end{remark}

\begin{figure}[t]
\center
\includegraphics[width=\textwidth]{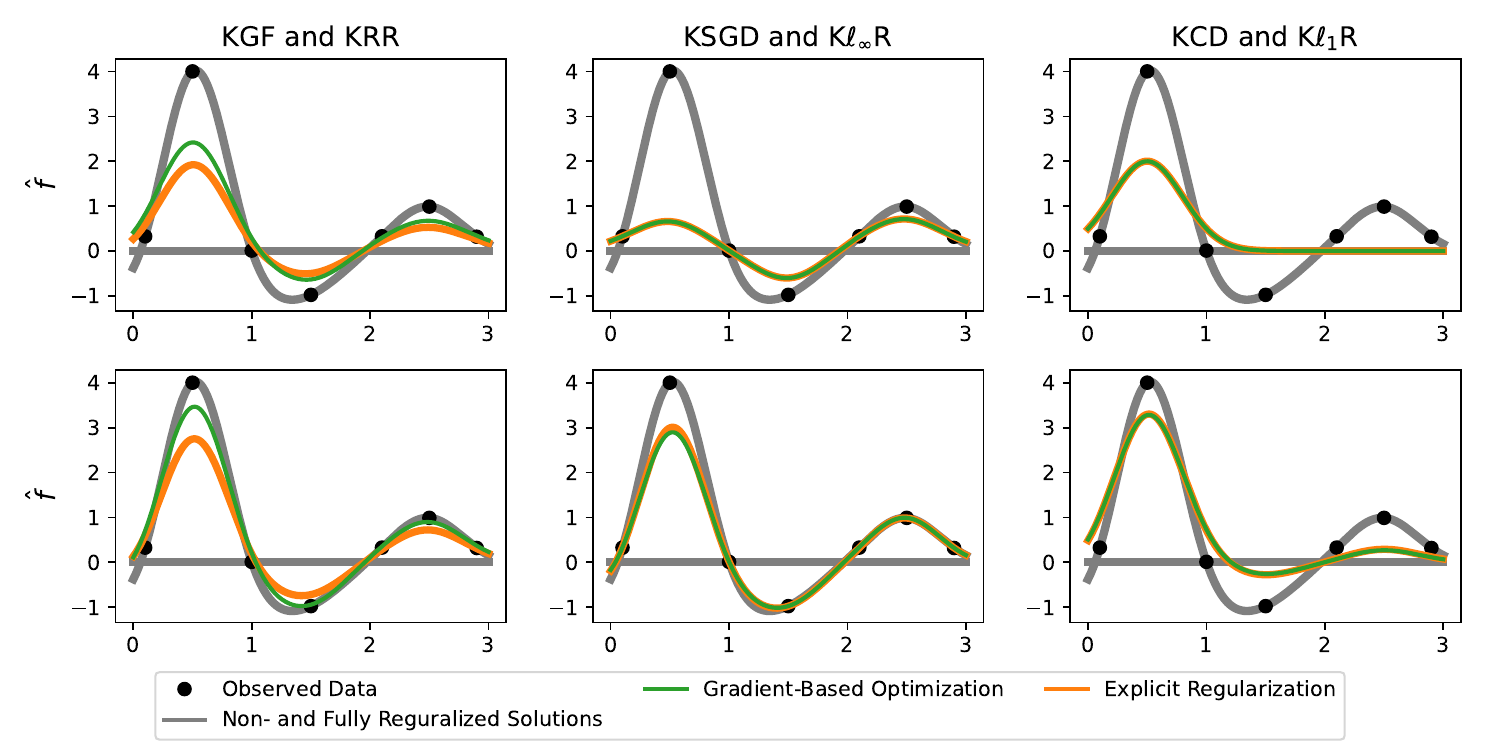}
\caption{Comparisons of the effects of KGF/KRR, KSGD/K$\ell_\infty$R, and KCD/K$\ell_1$R on $\fhv(\xv)$.
In the top panel, a larger regularization, or a shorter training time, is used than in the bottom panel.
For KGF/KRR, we use $t=1/\lambda$, while for the other two cases, $t$ and $\lambda$ are chosen so that the functions coincide as well as possible.\\
As proposed by Proposition \ref{thm:kgf_krr_diff}, the KGF and KRR solutions differ most where the non-regularized solution is large, and as suggested by Proposition \ref{thm:kgf_krr_y}, the KGF solution tends to lie closer to the observations than the KRR solution does.\\
For KSGD/K$\ell_\infty$R, the more extreme observations tend to be penalized harder, resulting in a more robust solution that is less sensitive to outliers.
For KCD/K$\ell_1$R, some observations do not contribute to the solution, resulting in peaks at the more significant observations.\\
The solutions obtained through early stopping are very similar to those obtained through explicit regularization, although not exactly identical.
}
\label{fig:comp_af}
\end{figure}

\section{Experiments}
\label{sec:exps}
In this section, we demonstrate our methods on synthetic and real data. We first demonstrate the six methods in Table \ref{tab:abbrs} on simple synthetic data in Section \ref{sec:exp_toy}. Then, in Section \ref{sec:exp_real}, we compare KSGD and K$\ell_\infty$R to four other robust kernel regression methods on five real data sets.
K$\ell_\infty$R and K$\ell_1$R were implemented by solving Equations \ref{eq:alphainf} and \ref{eq:alpha1} using proximal gradient descent, with an optimization step size of $0.01$, a step size that was also used for KGD, KSGD, and KCD.
We consistently used the Gaussian kernel, $k(\xv,\xpv)=\exp\left(-\frac{\|\xv-\xpv\|^2_2}{2\sigma^2}\right)$; in Appendix \ref{sec:more_exps}, we also present results for four additional kernels.
\subsection{Demonstrations on Synthetic Data}
\label{sec:exp_toy}
To get a better intuition of the six methods discussed in Sections \ref{sec:kgf} and \ref{sec:l1linf}, in this section, we demonstrate them on simple, synthetic data.

In Figures \ref{fig:comp_af} and \ref{fig:comp_a}, the similarities between explicit regularization and gradient-based optimization with early stopping are demonstrated. The solutions are very similar, although not exactly identical.
In Figure \ref{fig:comp_af}, we observe that for KGF and KRR, all residuals are treated equally and all parts of the function are updated at a pace that is proportional to the distance to the solution at convergence.
For KSGD and K$\ell_\infty$, which are less sensitive to outliers, all parts of the function are initially updated at the same pace, and, in contrast to more moderate observations, the most extreme observations are not fully incorporated until at the end of the training.
For KCD and K$\ell_1$R, which minimize the maximum residual, the most extreme observations are incorporated first, while more moderate observations are initially ignored.

\begin{figure}[btp]
\center
\includegraphics[width=\textwidth]{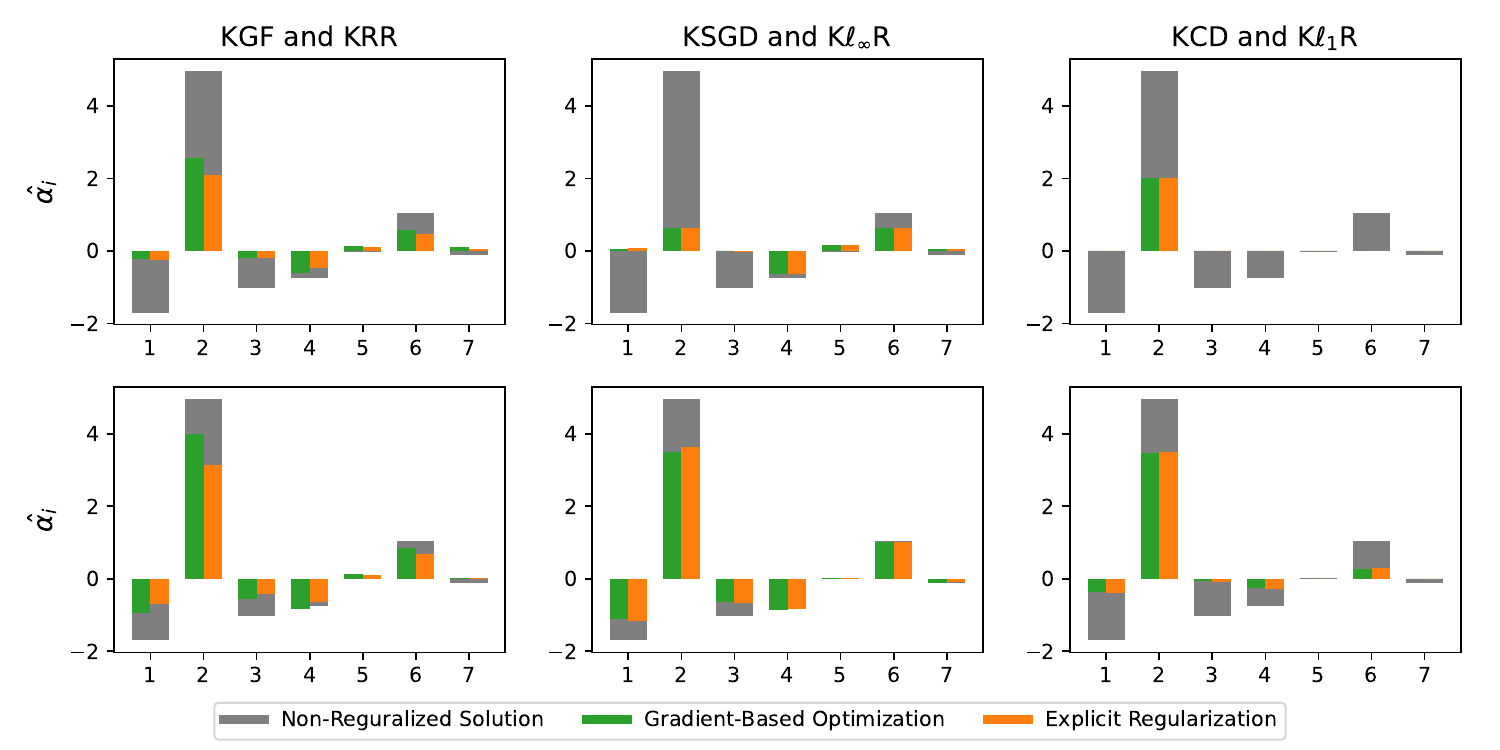}
\caption{Comparisons of the effects of KGF/KRR, KSGD/K$\ell_\infty$R, and KCD/K$\ell_1$R on the seven elements in $\alphahv$.
In the top panel, a larger regularization, or a shorter training time, is used than in the bottom panel.
For KGF/KRR, we use $t=1/\lambda$, while, for the two other cases, $t$ and $\lambda$ are chosen so that the functions coincide as well as possible.\\
For KSGD/K$\ell_\infty$R, the more extreme observations tend to be penalized harder, resulting in no extreme $\hat{\alpha}_i$'s.
For KCD/K$\ell_1$R, some observations do not contribute to the solution, resulting in the corresponding $\hat{\alpha}_i$'s being 0.\\
The solutions obtained through early stopping are very similar to those obtained through explicit regularization, although not exactly identical.
}
\label{fig:comp_a}
\end{figure}

In Figure \ref{fig:syn_expl}, we further compare KSGD/K$\ell_\infty$R and KCD/K$\ell_1$R to KGD/KRR.
For robust regression, to obtain data with outliers, 100 observations were sampled according to 
\begin{equation}
\label{eq:data_sin}
x\sim\mathcal{U}(-10,10),\ y=\sin\left(\frac\pi2\cdot x\right)+\mathcal{C}(0,0.1),
\end{equation}
where $\mathcal{U}(a,b)$ denotes the uniform distribution on $(a,b)$, and $\mathcal{C}(0,\gamma)$ the centered Cauchy distribution with scale parameter $\gamma$. 

To demonstrate sparse regression on a signal that is mostly zero, with a narrow but distinct peak, 100 observations were sampled according to 
\begin{equation}
\label{eq:data_gauss}
x\sim\mathcal{U}(-10,10),\ y=e^{-5\cdot x^2}+\mathcal{N}(0,0.1^2),
\end{equation}
where again $\mathcal{U}(a,b)$ denotes the uniform distribution and $\mathcal{N}(\mu,\sigma^2)$ denotes the normal distribution.
On the first (robust) data set, we compare KSGD, K$\ell_\infty$R, KGD, and KRR, and on the second (sparse) data set, we compare KCD, K$\ell_1$R, KGD, and KRR.
In both cases, bandwidth and regularization were selected by 10-fold cross-validation evaluating $30\times 30$ logarithmically spaced candidate values.

As expected, KGD/KRR are more sensitive to the outliers than KSGD/K$\ell_\infty$R are, with KGD being more sensitive than KRR, something that probably can be attributed to the fact that the KGF solution tends to lie closer to the observations than the KRR solution does, as suggested by Proposition \ref{thm:kgf_krr_y}.
For the second data set, to perform well at the peak, KGD/KRR must also incorporate the noise in the regions where the true signal is zero. KCD/K$\ell_1$R include the most significant observations first and are thus able to perform well at the peak without incorporating the noise in the zero regions.

\begin{figure}[t]
\center
\includegraphics[width=\textwidth]{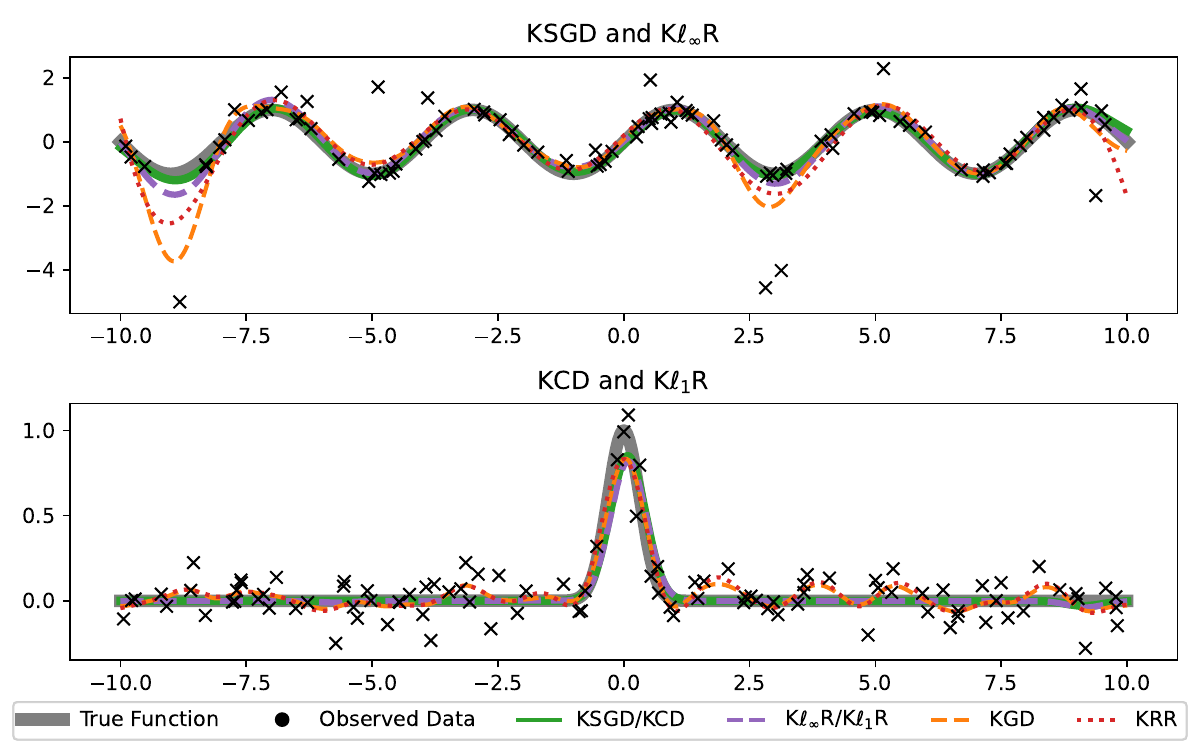}
\caption{Modelling the data generated by Equations \ref{eq:data_sin} (top) and \ref{eq:data_gauss} (bottom) using KSGD/K$\ell_\infty$R (top) or KCD/K$\ell_1$R (bottom) KGD/KRR (both panels). 
For the first data set, KSGD/K$\ell_\infty$R are less affected by the outliers than KGD/KRR are, and KGD is more affected by the outliers than KRR is.
For the second data set, in contrast to KGD/KRR, KCD/K$\ell_1$R are able to model the peak without incorporating the noise.}
\label{fig:syn_expl}
\end{figure}

\subsection{Robust Kernel Regression on Real Data}
\label{sec:exp_real}
In this section, we use five real data sets to compare robust kernel regression through KSGD and K$\ell_\infty$R to four state-of-the-art methods for robust kernel regression, and to KGD and KRR. The competing methods, for which we used the default implementations, are presented in Table \ref{tab:methods}, and the data sets in Table \ref{tab:datasets}.

Each data set was standardized to zero mean and unit variance. Then 50 random splits were created by randomly selecting 100 observations, which in turn were split 80\%/20\% into training and testing data.

Hyper-parameter selection was performed through 10-fold cross-validation, with 30 logarithmically spaced values for the kernel bandwidth, $\sigma$, and regularization strength, $\lambda$, respectively. Moreover, all four methods in Table \ref{tab:methods} come with an additional hyper-parameter, for which we evaluated three different values. For quantile regression, where the parameter $\tau\in(0,1)$ defines which quantile to use, we used $\tau \in [0.25,\ 0.5,\ 0.75]$. For the two M-estimators, which both come with a parameter $k$, which tunes the sensitivity toward outliers, we used $k\in [0.5k_d,\ k_d, 2k_d]$, where $k_d$ is the default tuning constant. For Huber regression, $k_d=1.345\hat{\sigma}$, and for Tukey bisquare regression, $k_d=4.685\hat{\sigma}$, where $\hat{\sigma}$ denotes the standard deviation of the residuals.

For the early stopping algorithms, rather than using an explicit regularization parameter $\lambda$, the stopping times were selected by monitoring performance on a validation data set, for which 10\% of the training set was used. Once the performance started to decrease on the validation set, the algorithm was terminated.
The experiments were run on a cluster with AMD EPYC Zen 2, 2.25 GHz cores, using one core per experiment.

\begin{table}[b]
\caption{Robust kernel regression models used.}
\centering
\begin{tabular}{|l|l|}
\hline
Method & Abbreviation\\
\hline
\hline
Huber loss implemented by iteratively reweighted least squares & KMR-H\\
Tukey loss implemented by iteratively reweighted least squares & KMR-T\\
Fast quantile regression implemented by \citet{zheng2022fast} & KQR-Z\\
Fast quantile regression implemented by \citet{tang2024fastkqr} & KQR-T\\
\hline
\end{tabular}
\label{tab:methods}
\end{table}

In Tables \ref{tab:real_100_0} and \ref{tab:real_100_1}, we present the computation time and $R^2$ on test data for the experiments.\footnote{$R^2:=1-\|\yv-\fhv\|_2^2/\|\yv-\bar{y}\|_2^2\leq 1$, where $\yv$ and $\bar{y}$ denotes the vector of observations and its mean, and $\fhv$ denotes the vector of model predictions, measures the proportion of the variance in $\yv$ that is explained by the model. $R^2=1$ corresponds to a perfect fit, while a model that predicts $\hat{f}_i=\bar{y}$ for all $i$ results in $R^2=0$; thus a negative value of $R^2$ is possible, but it corresponds to a model that performs worse than always predicting the mean of the observed data.}
The difference between the two tables is that in Table \ref{tab:real_100_1}, the presence of outliers is amplified by multiplying each element, $y_i$, in the response vector by $(1+|\varepsilon_i|)$, where $\varepsilon_i\sim\mathcal{C}(0,0.01)$ is Cauchy distributed. We note that amplifying the outliers only affects the relative performance between the robust and non-robust methods, not the relative performance among the robust methods.

For all five data sets, KSGD and K$\ell_\infty$R perform at least as well as the competing four robust methods in terms of test $R^2$. However, KSGD is significantly faster than the competing methods, performing between one and two orders of magnitude faster.

Comparing Tables \ref{tab:real_100_0} and \ref{tab:real_100_1}, we note that KGD and KRR are much more sensitive to outliers than the robust methods are. We also note with amplified outliers, KRR tends to perform better than KGD in terms of test $R^2$, while in Table \ref{tab:real_100_0}, the tendency is the opposite. This can probably be attributed to the fact that the KGF solution tends to lie closer to the observations than the KRR solution does, as suggested by Proposition \ref{thm:kgf_krr_y}. This is often beneficial, however not in the presence of outliers.

\begin{table}[t]
\caption{Real data sets, with corresponding reference when existing, used for comparing the methods. }
\centering
\begin{tabular}{|l|l|}
\hline
Data set & Size, $n\times p$\\
\hline
\hline
Sound pressure of airfoils\tablefootnote{The data set is available at \url{https://archive.ics.uci.edu/dataset/291/airfoil+self+noise}.} & $1502\times 6$\\
\hline
House values in California \citep{pace1997sparse}\tablefootnote{The data set is available at \url{https://www.dcc.fc.up.pt/~ltorgo/Regression/cal_housing.html}.} & $20640\times  9$\\
\hline
Energy consumption in steel production \citep{ve2021efficient}\tablefootnote{The data set is available at \url{https://archive.ics.uci.edu/dataset/851/steel+industry+energy+consumption}.} & $35040\times 7$\\
\hline
Critical temperature of superconductors \citep{hamidieh2018data}\tablefootnote{The data set is available at \url{https://archive.ics.uci.edu/dataset/464/superconductivty+data}.} & $21263\times  82$\\
\hline
Daily temperature in the U.K. in the year 2000 \citep{wood2017generalized}\tablefootnote{The data set is available at \url{https://www.maths.ed.ac.uk/~swood34}.} & $45568\times 5$\\
\hline
\end{tabular}
\label{tab:datasets}
\end{table}

\begin{table}
\caption{The 2.5th, 50th, and 97.5th percentiles of computation time and test $R^2$ for the different methods and data sets, \textbf{without} amplified outliers. The six robust methods perform very similarly in terms of test $R^2$, while KSGD performs one to two orders of magnitude faster.}
\center
\begin{tabular}{|l|l|l|l|}
\hline
Data & Method & \makecell[l]{Computation Time [s]\\50\%,\ (2.5\%,\ 97.5\%)} & \makecell[l]{Test $R^2$\\50\%,\ (2.5\%,\ 97.5\%)}\\
\hline
\multirow{9}{*}{\makecell[l]{Airfoil Sound\\Pressure}}
& KSGD            & $9.36,\ (5.41, 13.8)$ & $0.53 ,\ (-0.47  , 0.82 )$ \\
& K$\ell_\infty$R & $383,\ (370, 404)$ & $0.56 ,\ (-0.54  , 0.87 )$ \\
\cline{2-4}
& KMR-H           & $172,\ (158, 192)$ & $0.51 ,\ (-0.64  , 0.84 )$ \\
& KMR-T           & $161,\ (151, 178)$ & $0.53 ,\ (-0.45  , 0.84 )$ \\
& KQR-A           & $289,\ (275, 311)$ & $0.55 ,\ (-0.55  , 0.84 )$ \\
& KQR-B           & $314,\ (185, 345)$ & $0.41 ,\ (-0.13  , 0.72 )$ \\
\cline{2-4}
& KGD             & $14.7,\ (10.9, 18.4)$ & $0.57 ,\ (-0.49  , 0.85 )$ \\
& KRR             & $2.92,\ (2.87, 3.04)$ & $0.53 ,\ (-0.44  , 0.84 )$ \\
\hline
\multirow{9}{*}{\makecell[l]{California\\House Values}}
& KSGD            & $4.14,\ (2.20, 8.39)$ & $0.63 ,\ (-0.09  , 0.86 )$ \\
& K$\ell_\infty$R & $325,\ (311, 348)$ & $0.64 ,\ (-0.17  , 0.88 )$ \\
\cline{2-4}
& KMR-H           & $137,\ (124, 153)$ & $0.60 ,\ (-0.30  , 0.89 )$ \\
& KMR-T           & $120,\ (112, 135)$ & $0.58 ,\ (-0.31  , 0.89 )$ \\
& KQR-A           & $271,\ (259, 294)$ & $0.61 ,\ (-0.19  , 0.86 )$ \\
& KQR-B           & $206,\ (115, 218)$ & $0.33 ,\ (-0.10  , 0.62 )$ \\
\cline{2-4}
& KGD             & $12.6,\ (9.20, 15.7)$ & $0.60 ,\ (-0.17  , 0.89 )$ \\
& KRR             & $2.95,\ (2.90, 3.14)$ & $0.60 ,\ (-0.34  , 0.88 )$ \\
\hline
\multirow{9}{*}{\makecell[l]{Steel Energy\\Consumption}}
& KSGD            & $12.0,\ (9.85, 20.3)$ & $0.98 ,\ (0.92   , 0.99 )$ \\
& K$\ell_\infty$R & $416,\ (401, 436)$ & $0.98 ,\ (0.94   , 0.99 )$ \\
\cline{2-4}
& KMR-H           & $166,\ (146, 199)$ & $0.99 ,\ (0.98   , 1.00 )$ \\
& KMR-T           & $156,\ (148, 191)$ & $0.99 ,\ (0.98   , 1.00 )$ \\
& KQR-A           & $267,\ (256, 332)$ & $0.99 ,\ (0.97   , 1.00 )$ \\
& KQR-B           & $541,\ (475, 620)$ & $0.90 ,\ (0.60   , 0.98 )$ \\
\cline{2-4}
& KGD             & $22.2,\ (16.8, 26.0)$ & $0.99 ,\ (0.94   , 1.00 )$ \\
& KRR             & $2.95,\ (2.90, 3.29)$ & $0.99 ,\ (0.97   , 1.00 )$ \\
\hline
\multirow{9}{*}{\makecell[l]{Superconductor\\Critical\\Temperature}}
& KSGD            & $12.6,\ (8.54, 17.7)$ & $0.59 ,\ (-0.01  , 0.84 )$ \\
& K$\ell_\infty$R & $409,\ (401, 425)$ & $0.64 ,\ (0.14   , 0.89 )$ \\
\cline{2-4}
& KMR-H           & $169,\ (155, 193)$ & $0.66 ,\ (0.05   , 0.86 )$ \\
& KMR-T           & $160,\ (135, 189)$ & $0.65 ,\ (0.04   , 0.89 )$ \\
& KQR-A           & $288,\ (255, 329)$ & $0.61 ,\ (-0.04  , 0.88 )$ \\
& KQR-B           & $482,\ (338, 632)$ & $0.46 ,\ (-0.08  , 0.69 )$ \\
\cline{2-4}
& KGD             & $14.8,\ (10.9, 17.4)$ & $0.67 ,\ (0.21   , 0.89 )$ \\
& KRR             & $3.44,\ (3.37, 3.88)$ & $0.66 ,\ (0.17   , 0.89 )$ \\
\hline
\multirow{9}{*}{\makecell[l]{U.K.\\Temperature}}
& KSGD            & $6.17,\ (2.99, 9.76)$ & $0.40 ,\ (-0.07  , 0.69 )$ \\
& K$\ell_\infty$R & $370,\ (358, 388)$ & $0.37 ,\ (-0.06  , 0.67 )$ \\
\cline{2-4}
& KMR-H           & $182,\ (161, 206)$ & $0.38 ,\ (-1.13  , 0.68 )$ \\
& KMR-T           & $180,\ (157, 192)$ & $0.39 ,\ (-1.03  , 0.67 )$ \\
& KQR-A           & $310,\ (282, 340)$ & $0.36 ,\ (-1.24  , 0.66 )$ \\
& KQR-B           & $321,\ (258, 343)$ & $0.42 ,\ (-0.06  , 0.66 )$ \\
\cline{2-4}
& KGD             & $11.3,\ (8.31, 15.1)$ & $0.40 ,\ (-0.31  , 0.71 )$ \\
& KRR             & $2.97,\ (2.87, 3.14)$ & $0.40 ,\ (-1.13  , 0.68 )$ \\
\hline
\end{tabular}
\label{tab:real_100_0}
\end{table}

\begin{table}
\caption{The 2.5th, 50th, and 97.5th percentiles of computation time and test $R^2$ for the different methods and data sets, \textbf{with} amplified outliers. The six robust methods perform very similarly in terms of test $R^2$, while KSGD performs one to two orders of magnitude faster. The two non-robust methods, KGD/KRR, do not perform as well as the robust methods.}
\center
\begin{tabular}{|l|l|l|l|}
\hline
Data & Method & \makecell[l]{Computation Time [s]\\50\%,\ (2.5\%,\ 97.5\%)} & \makecell[l]{Test $R^2$\\50\%,\ (2.5\%,\ 97.5\%)}\\
\hline
\multirow{9}{*}{\makecell[l]{Airfoil Sound\\Pressure}}
& KSGD            & $10.5,\ (5.29, 13.8)$ & $0.44 ,\ (-0.05  , 0.74 )$ \\
& K$\ell_\infty$R & $395,\ (375, 431)$ & $0.42 ,\ (-0.76  , 0.75 )$ \\
\cline{2-4}
& KMR-H           & $168,\ (101, 199)$ & $0.38 ,\ (-1.12  , 0.71 )$ \\
& KMR-T           & $154,\ (115, 182)$ & $0.36 ,\ (-1.08  , 0.73 )$ \\
& KQR-A           & $304,\ (284, 346)$ & $0.38 ,\ (-1.10  , 0.74 )$ \\
& KQR-B           & $346,\ (287, 714)$ & $0.25 ,\ (-5.95  , 0.61 )$ \\
\cline{2-4}
& KGD             & $12.4,\ (3.23, 16.5)$ & $0.25 ,\ (-82.33 , 0.71 )$ \\
& KRR             & $2.99,\ (2.88, 3.22)$ & $0.27 ,\ (-1.61  , 0.70 )$ \\
\hline
\multirow{9}{*}{\makecell[l]{California\\House Values}}
& KSGD            & $4.65,\ (2.51, 8.68)$ & $0.46 ,\ (-0.06  , 0.80 )$ \\
& K$\ell_\infty$R & $352,\ (326, 384)$ & $0.42 ,\ (-0.50  , 0.76 )$ \\
\cline{2-4}
& KMR-H           & $130,\ (80.1, 170)$ & $0.36 ,\ (-0.74  , 0.81 )$ \\
& KMR-T           & $118,\ (91.1, 143)$ & $0.37 ,\ (-0.63  , 0.77 )$ \\
& KQR-A           & $297,\ (268, 327)$ & $0.30 ,\ (-0.44  , 0.82 )$ \\
& KQR-B           & $259,\ (160, 283)$ & $0.23 ,\ (-2.20  , 0.56 )$ \\
\cline{2-4}
& KGD             & $11.0,\ (4.98, 14.2)$ & $0.17 ,\ (-4.39  , 0.76 )$ \\
& KRR             & $2.99,\ (2.92, 3.35)$ & $0.29 ,\ (-0.22  , 0.77 )$ \\
\hline
\multirow{9}{*}{\makecell[l]{Steel Energy\\Consumption}}
& KSGD            & $14.5,\ (11.1, 23.3)$ & $0.89 ,\ (0.06   , 0.97 )$ \\
& K$\ell_\infty$R & $426,\ (404, 458)$ & $0.85 ,\ (-1.13  , 0.98 )$ \\
\cline{2-4}
& KMR-H           & $168,\ (122, 220)$ & $0.86 ,\ (-1.59  , 0.98 )$ \\
& KMR-T           & $174,\ (141, 214)$ & $0.86 ,\ (-0.54  , 0.98 )$ \\
& KQR-A           & $308,\ (263, 340)$ & $0.88 ,\ (-2.11  , 0.98 )$ \\
& KQR-B           & $707,\ (578, 1.220)$ & $0.67 ,\ (-1.68  , 0.95 )$ \\
\cline{2-4}
& KGD             & $13.3,\ (5.84, 17.8)$ & $0.74 ,\ (-33.96 , 0.97 )$ \\
& KRR             & $2.96,\ (2.87, 3.34)$ & $0.78 ,\ (-3.34  , 0.97 )$ \\
\hline
\multirow{9}{*}{\makecell[l]{Superconductor\\Critical\\Temperature}}
& KSGD            & $14.3,\ (10.7, 21.3)$ & $0.45 ,\ (-0.22  , 0.81 )$ \\
& K$\ell_\infty$R & $416,\ (401, 441)$ & $0.48 ,\ (-0.49  , 0.84 )$ \\
\cline{2-4}
& KMR-H           & $159,\ (114, 202)$ & $0.48 ,\ (-0.35  , 0.79 )$ \\
& KMR-T           & $146,\ (118, 201)$ & $0.48 ,\ (-0.33  , 0.82 )$ \\
& KQR-A           & $290,\ (258, 326)$ & $0.44 ,\ (-0.62  , 0.83 )$ \\
& KQR-B           & $586,\ (406, 1.190)$ & $0.28 ,\ (-0.96  , 0.69 )$ \\
\cline{2-4}
& KGD             & $12.1,\ (5.00, 15.4)$ & $0.37 ,\ (-1.41  , 0.83 )$ \\
& KRR             & $3.46,\ (3.37, 3.89)$ & $0.47 ,\ (-0.51  , 0.82 )$ \\
\hline
\multirow{9}{*}{\makecell[l]{U.K.\\Temperature}}
& KSGD            & $7.29,\ (3.14, 11.6)$ & $0.33 ,\ (-0.13  , 0.63 )$ \\
& K$\ell_\infty$R & $385,\ (368, 423)$ & $0.25 ,\ (-1.13  , 0.62 )$ \\
\cline{2-4}
& KMR-H           & $164,\ (99.7, 213)$ & $0.29 ,\ (-3.08  , 0.66 )$ \\
& KMR-T           & $156,\ (117, 186)$ & $0.27 ,\ (-2.31  , 0.65 )$ \\
& KQR-A           & $305,\ (282, 339)$ & $0.22 ,\ (-3.07  , 0.63 )$ \\
& KQR-B           & $328,\ (277, 532)$ & $0.20 ,\ (-0.64  , 0.60 )$ \\
\cline{2-4}
& KGD             & $9.97,\ (3.15, 15.7)$ & $0.11 ,\ (-20.91 , 0.61 )$ \\
& KRR             & $2.92,\ (2.87, 3.16)$ & $0.17 ,\ (-2.46  , 0.61 )$ \\
\hline
\end{tabular}
\label{tab:real_100_1}
\end{table}

\section{Conclusions}
We introduced an equivalent formulation of kernel ridge regression and used it to define kernel regression with the $\ell_\infty$ and $\ell_1$ penalties, and for solving kernel regression with gradient-based optimization methods. 
We introduced the methods kernel sign gradient descent and kernel coordinate descent and utilized the similarities between $\ell_\infty$ regularization and sign gradient descent, and between $\ell_1$ regularization and coordinate descent (forward stagewise regression), to obtain computationally efficient algorithms for robust and sparse kernel regression, respectively. We demonstrated on five real data sets that our implementation of robust regression is one to two orders of magnitude faster than existing robust kernel regression methods.

Our generalizations of kernel ridge regression, together with regularization through early stopping, enable computationally efficient, kernelized robust, and sparse, regression.
Although we investigated only kernel regression with the $\ell_2$, $\ell_1$, and $\ell_\infty$ penalties, many other penalties could be considered, such as the adaptive lasso \citep{zou2006adaptive}, the group lasso \citep{yuan2006model}, the exclusive lasso \citep{zhou2010exclusive} and OSCAR \citep{bondell2008simultaneous}. This is, however, left for future work.

Code is available at \url{https://github.com/allerbo/fast_robust_kernel_regression}.

\section*{Acknowledgments}
The computations were enabled by resources provided by the National Academic Infrastructure for Supercomputing in Sweden (NAISS), partially funded by the Swedish Research Council through grant agreement no. 2022-06725.

\newpage
\begin{appendix}
\section{Calculations for \texorpdfstring{$[\fhv^\top, \fhsv{^\top}]^\top$}{[f, f*]}}
\label{sec:equ_f}
In this section, we revisit some of the calculations in the main manuscript and reformulate them directly in terms of $[\fhv^\top, \fhsv{^\top}]^\top$, rather than obtaining them by multiplying $\alphav$ by $[\K^\top, \Ks{^\top}]^\top$.

\subsection{Equivalent Formulations of Kernel Ridge Regression}
In Proposition \ref{thm:equ_f}, we present three different objective functions, that all render the same solution for $[\fhv^\top, \fhsv{^\top}]^\top$ as in Equation \ref{eq:krr_s_fs}. When expressing KRR in terms of $\alphav\in\R^n$, it is enough to include $\K\in\R^{n\times n}$ in the expression to capture the dynamics of the kernel. However, when expressing KRR in terms of $[\fhv^\top, \fhsv{^\top}]^\top\in \R^{n+\ns}$, we need to introduce the larger kernel matrix
$$\Kss:=\begin{bmatrix}\K(\X,\X) & \K(\X,\Xs)\\ \K(\Xs,\X) & \K(\Xs,\Xs)\end{bmatrix}
\in \R^{(n+\ns)\times(n+\ns)}$$
in order to let the kernel affect not only $\fv\in \R^n$, but also $\fsv\in \R^{\ns}$.
Furthermore, we need the extended response vector $[\yv^\top,\ytv^\top]^\top$, where $\ytv$ is a copy of $\fsv$, so that $\ytv-\fsv=\nv$. 

\begin{prop}
\label{thm:equ_f}
\begin{subequations}
\label{eq:equ_f}
\begin{align}
\begin{bmatrix}\fhv\\\fhsv\end{bmatrix}=&\argmin_{[\fv^\top,\fsv{^\top}]^\top\in \R^{n+\ns}}\frac12\left\|\yv-\fv\right\|_2^2+\frac\lambda2\left\|\begin{bmatrix}\fv\\\fsv\end{bmatrix}\right\|^2_{(\Kss)^{-1}}\label{eq:equ_f1}\\
=&\argmin_{[\fv^\top,\fsv{^\top}]^\top\in \R^{n+\ns}}\frac12\left\|\begin{bmatrix}\yv\\\ytv\end{bmatrix}-\begin{bmatrix}\fv\\\fsv\end{bmatrix}\right\|_2^2+\frac\lambda2\left\|\begin{bmatrix}\fv\\\fsv\end{bmatrix}\right\|^2_{(\Kss)^{-1}}\label{eq:equ_f1a}\\
=&\argmin_{[\fv^\top,\fsv{^\top}]^\top\in \R^{n+\ns}}\frac12\left\|\begin{bmatrix}\yv\\\ytv\end{bmatrix}-\begin{bmatrix}\fv\\\fsv\end{bmatrix}\right\|_{\Kss}^2+\frac\lambda2\left\|\begin{bmatrix}\fv\\\fsv\end{bmatrix}\right\|^2_2\label{eq:equ_f2}\\
=&\begin{bmatrix}\K\\\Ks\end{bmatrix}\left(\K+\lambda\I\right)^{-1}\yv.\notag
\end{align}
\end{subequations}
\end{prop}
\begin{remark}
We assert that $\fsv$ does not affect the reconstruction error by requiring that $\ytv-\fsv=\nv$, We may, however, not define $\ytv$ to equal $\fsv$, since we do not want $\ytv$ to be considered when differentiating Equations \ref{eq:equ_f1a} and \ref{eq:equ_f2} with respect to $\fsv$.
\end{remark}
\begin{remark}
The term $\left\|\begin{bmatrix}\fv\\\fsv\end{bmatrix}\right\|^2_{(\Kss)^{-1}}$, might appear a bit peculiar, including the inverse of the extended kernel matrix (which never needs to be calculated due to cancellations). However, it is very closely related to the expression for the reproducing kernel Hilbert space norm, $\|f\|^2_{\Hh_k}$, obtained when expressing the functions $f$ and $k$ in terms of the orthogonal basis given by Mercer's Theorem:
$f(\xv)=\sum_{i=1}^\infty f_i\phi_i(\xv)$ and $k(\xv,\xpv)=\sum_{i=1}^\infty k_i\phi_i(\xv)\phi_i(\xpv)$, with $\|f\|^2_{\Hh_k}:=\sum_{i=1}^\infty \frac{f_i^2}{k_i}$. Defining the vector $\bm{\tilde{f}}$ and the diagonal matrix $\bm{\tilde{K}}$ according to $\bm{\tilde{f}}_i=f_i$ and $\bm{\tilde{K}}_{ii}=k_i$, we obtain $\|f\|^2_{\Hh_k}=\sum_{i=1}^\infty \frac{f_i^2}{k_i}=\bm{\tilde{f}}^\top\bm{\tilde{K}}^{-1}\bm{\tilde{f}}=\|\bm{\tilde{f}}\|_{\bm{\tilde{K}}^{-1}}^2$.
\end{remark}

\subsection{Kernel Gradient Descent and Kernel Gradient Flow}
The gradient of Equation \ref{eq:equ_f2} with respect to $[\fv^\top,\fsv{^\top}]^\top$ is 
$$\begin{bmatrix}\K\\\Ks\end{bmatrix}(\fv-\yv)+\lambda\begin{bmatrix}\fv\\\fsv\end{bmatrix},$$
which coincides with Equation \ref{eq:kgd_grad_f}, as expected. Thus, the KGD update in\\ $[\fhv^\top,\fhsv{^\top}]^\top$ is
\begin{equation*}
\label{eq:kgd_update_f}
\begin{bmatrix}\fhv_{k+1}\\\fhsv_{k+1}\end{bmatrix}=\begin{bmatrix}\fhv_k\\\fhsv_k\end{bmatrix}+\eta\cdot\begin{bmatrix}\K\\\Ks\end{bmatrix}\left(\yv-\fhv_k\right),\ \begin{bmatrix}\fhv_0\\\fhsv_0\end{bmatrix}=\nv,
\end{equation*}
with the corresponding differential equation
\begin{equation}
\label{eq:kgd_diff_eq_f}
\frac{\partial}{\partial t}\left(\begin{bmatrix}\fhv(t)\\\fhsv(t)\end{bmatrix}\right)=\begin{bmatrix}\K\\\Ks\end{bmatrix}\left(\yv-\fhv(t)\right),\ \begin{bmatrix}\fhv(0)\\\fhsv(0)\end{bmatrix}=\nv,
\end{equation}
whose solution is given by Equation \ref{eq:kgf_s_f}.
\begin{equation}
\label{eq:kgf_s_f}
\begin{bmatrix}\fhv(t)\\\fhsv(t)\end{bmatrix}=\begin{bmatrix}\I\\\Ks\K^{-1}\end{bmatrix}(\I-\exp(-t\K))\yv.
\end{equation}

\subsection{Kernel Regression with the \texorpdfstring{$\ell_\infty$}{linf} and \texorpdfstring{$\ell_1$}{l1} Norms}
If we replace the squared $\ell_2$ norm in Equation \ref{eq:equ_f2} by the $\ell_\infty$ and $\ell_1$ norms, we obtain 
\begin{equation}
\label{eq:finf}
\argmin_{[\fv^\top,\fsv{^\top}]^\top\in \R^{n+\ns}}\frac12\left\|\begin{bmatrix}\yv\\\ytv\end{bmatrix}-\begin{bmatrix}\fv\\\fsv\end{bmatrix}\right\|_{\Kss}^2+\lambda\left\|\begin{bmatrix}\fv\\\fsv\end{bmatrix}\right\|_\infty
\end{equation}
and
\begin{equation}
\label{eq:f1}
\argmin_{[\fv^\top,\fsv{^\top}]^\top\in \R^{n+\ns}}\frac12\left\|\begin{bmatrix}\yv\\\ytv\end{bmatrix}-\begin{bmatrix}\fv\\\fsv\end{bmatrix}\right\|_{\Kss}^2+\lambda\left\|\begin{bmatrix}\fv\\\fsv\end{bmatrix}\right\|_1.
\end{equation}

In contrast to ridge regression, where the formulations in $\alphav$ and $[\fv^\top,\fsv{^\top}]^\top$ are equivalent in the sense that the latter can always be obtained from the first through multiplication with $[\K^\top, \Ks{^\top}]^\top$, the solutions of Equations \ref{eq:finf} and \ref{eq:f1} cannot be obtained from the solutions of Equations \ref{eq:alphainf} and \ref{eq:alpha1}.

Since $\ell_\infty$ regularization promotes a solution with no extreme parameters, many of the predictions in the solution to Equation \ref{eq:finf} are exactly equal.
For $\ell_1$ regularization, which promotes a sparse solution with the most significant parameters included first, in the solution to Equation \ref{eq:f1}, many of the predictions equal exactly zero, with only the most extreme predictions being non-zero.
In Figure \ref{fig:comp_f}, we exemplify this for explicit regularization for two different regularization strengths. Even though it is technically possible to use gradient-based optimization with early stopping for $[\fv^\top,\fsv{^\top}]^\top$, since the reconstruction error, and thus the gradient, is affected only by $\fv$ and not by $\fsv$, this does not work well for the sign gradient and coordinate descent algorithms.

\begin{figure}
\center
\includegraphics[width=\textwidth]{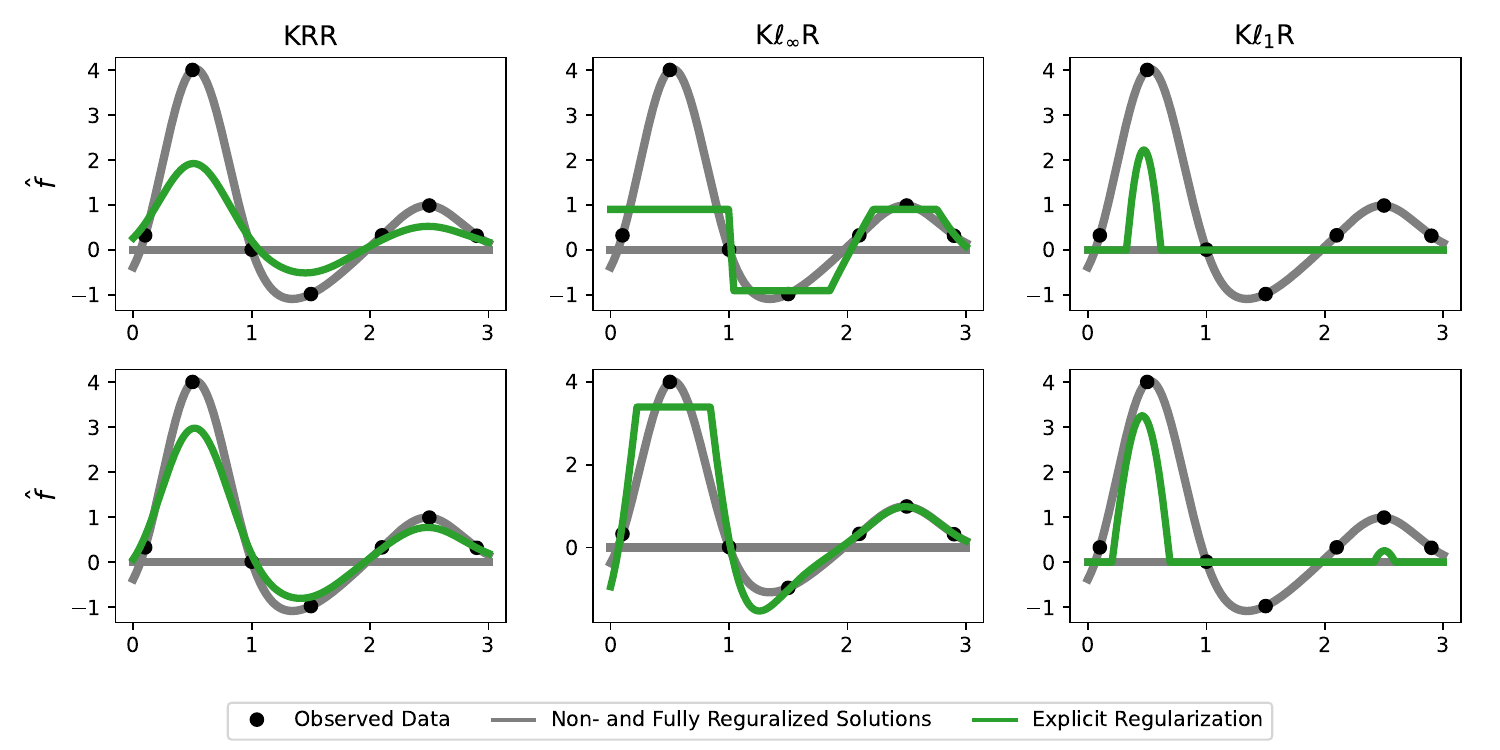}
\caption{Comparisons of the effects of $\ell_2$, $\ell_\infty$, and $\ell_1$ regularization on $[\fv^\top,\fsv{^\top}]^\top$.\\
In the top panel, a larger regularization is used than in the bottom panel.\\
For $\ell_\infty$ regularization, the absolute values of many predictions are exactly equal. To obtain this, compared to the non-regularized solution, some predictions are shifted away from zero, while some predictions are shifted toward zero.
For $\ell_1$ regularization, many predictions are set exactly to zero, with peaks at some extreme observations.}
\label{fig:comp_f}
\end{figure}

Finally, for completeness, we present the analog of Proposition \ref{thm:sgn_inf} for\\ $[\fv^\top,\fsv{^\top}]^\top$ in Proposition \ref{thm:sgn_inf_f}.

\begin{prop}~\\
\label{thm:sgn_inf_f}
For $\fpv:=[\fv^\top,\fsv{^\top}]^\top$ and $\ypv:=[\yv^\top,\ytv^\top]^\top$, where $\ytv-\fsv=\nv$, let $\fhpinfv(c)$ denote the solution to
$$\min_{\fpv\in\R^{n+\ns}}\left\|\ypv-\fpv\right\|_{\Kss}^2 \text{ s.t. } \|\fpv\|_\infty\leq c,$$
and let $\fhpsgfv(t)$ denote the solution to
$$\frac{\partial\fhpv(t)}{\partial t}=-\sgn\left(\frac{\partial}{\partial\fhpv(t)}\left(\left\|\ypv-\fhpv(t)\right\|_{\Kss}^2\right)\right),\ \fhpv(0)=\nv.$$
When $\Kss$ is diagonal, with elements $\{k_{ii}\}_{i=1}^{n+\ns}$, $k_{ii}>0$, the two solutions decompose element-wise and coincide for $c=t$:
$$\fhpinfi(t)=\fhpsgfi(t)=
\begin{cases}
\sgn(y_i)\cdot\min(t,|y_i|) &\text{if }i\leq n\\
0 &\text{if }i> n.
\end{cases}$$
\end{prop}

\section{Computationally Efficient Cross-Validation}
\label{sec:fast_cv}
For models of the type $\fhv=\Hhh\yv$, where $\Hhh\in\R^{n\times n}$ is a matrix that does not depend on $\yv$, the fast leave-one-out cross-validation, LOOCV, and generalized cross-validation, GCV, statistics are defined according to
\begin{equation*}
\begin{aligned}
&\text{LOOCV}:=\frac1n\sum_{i=1}^n\left(\frac{y_i-\fh(\xvi)}{1-H_{ii}}\right)^2\\
&\text{GCV}:=\frac1n\sum_{i=1}^n\left(\frac{y_i-\fh(\xvi)}{1-\text{Tr}(\Hhh)/n}\right)^2.
\end{aligned}
\end{equation*}
These two methods tend to be faster than k-fold cross-validation and provide similar performance. For KRR, where $\fhv=\K(\K+\lambda\I)^{-1}\yv$, we trivially obtain $\Hhh=\K(\K+\lambda\I)^{-1}$. For other penalties than the ridge penalty, it is however not trivial to obtain $\Hhh$. \citet{tibshirani1996regression} noted that the lasso objective, $\frac12\left\|\yv-\X\betav\right\|_2^2+\lambda\|\betav\|_1$, can be reformulated as $\frac12\left\|\yv-\X\betav\right\|_2^2+\frac\lambda2\betav^\top\B\betav$, where $\B\in \R^{p\times p}$ is a diagonal matrix with $B_{ii}=2/|\beta_i|$. If $\B$ is treated as a constant, this formulation becomes a weighted ridge regression problem which is minimized by $\betahv=\X(\X^\top\X+\lambda\B)^{-1}\X^\top\yv$, and thus, for a given $\betahv$, the author proposed approximating $\Hhh$ as $\X(\X^\top\X+\lambda\bm{\hat{B}})^{-1}\X^\top$, where $\bm{\hat{B}}$ is a diagonal matrix with $\hat{B}_{ii}=2/|\betahv|$.

Applying these ideas in the kernel formulation, we obtain, for an $\A$ such that $\frac12\alphav^\top\A\alphav=\|\alphav\|$,
\begin{equation*}
\begin{aligned}
\alphahv&=\argmin_{\alphav\in \R^n}\frac12\left\|\yv-\K\alphav\right\|_{\K^{-1}}^2+\lambda\|\alphav\|\\
&=\argmin_{\alphav\in \R^n}\frac12\left\|\yv-\K\alphav\right\|_{\K^{-1}}^2+\frac{\lambda}2\alphav^\top\A\alphav=(\K+\lambda\A)^{-1}\yv.
\end{aligned}
\end{equation*}
Thus $\Hhh$ can be approximated by $\K(\K+\lambda\bm{\hat{A}})^{-1}$, where $\bm{\hat{A}}\in\R^{n\times n}$ is a matrix such that $\frac12\alphahv^\top\bm{\hat{A}}\alphahv=\|\alphahv\|$. For the $\ell_1$ norm, $\|\alphahv\|_1=\sum_{i=1}^n|\hat{\alpha}_i|$, this is obtained if $\bm{\hat{A}}$ is a diagonal matrix with $\hat{A}_{ii}=2/|\hat{\alpha}_i|$.
For the $\ell_\infty$ norm, $\|\alphahv\|_\infty=\max_i|\hat{\alpha}_i|$, it is obtained if $\bm{\hat{A}}$ is a matrix with only zeros, except for $\hat{A}_{mm}=2/|\hat{\alpha}_m|$, where $m=\argmax_i|\hat{\alpha}_i|$.

In Tables \ref{tab:real_lgcv_100_0} and \ref{tab:real_lgcv_100_1}, we present the analogs of Tables \ref{tab:real_100_0} and \ref{tab:real_100_1} for LOOCV and GCV. As expected, the two methods perform faster than 10-fold cross-validation. However, for K$\ell_\infty$R, the performance in terms of test $R^2$ is much lower than for 10-fold cross-validation. For KRR, LOOCV performs well in terms of test $R^2$.

\begin{table}
\caption{The 2.5th, 50th, and 97.5th percentiles of computation time and test $R^2$ for K$\ell_\infty$R in combination with LOOCV and GCV on the different data sets, \textbf{without} amplified outliers. Compared to 10-fold cross validation LOOCV and GCV perform faster, but the cost of decreased predictive performance, especially for K$\ell_\infty$R.}
\center
\begin{tabular}{|l|l|l|l|l|}
\hline
Data & Method & \makecell[l]{Type of\\Cross-\\Validation} & \makecell[l]{Computation\\Time [s]\\50\%,\ (2.5\%,\ 97.5\%)} & \makecell[l]{Test $R^2$\\50\%,\ (2.5\%,\ 97.5\%)}\\
\hline
\multirow{4}{*}{\makecell[l]{Airfoil Sound\\Pressure}}
& \multirow{2}{*}{K$\ell_\infty$R}
& LOOCV & $49.1,\ (48.2, 57.2)$ & $-0.01,\ (-0.38  , 0.04 )$ \\
& & GCV & $50.0,\ (48.7, 57.9)$ & $-0.01,\ (-0.37  , 0.15 )$ \\
\cline{2-5}
& \multirow{2}{*}{KRR}
& LOOCV & $1.21,\ (0.891, 3.11)$ & $0.55 ,\ (-0.29  , 0.84 )$ \\
& & GCV & $1.43,\ (0.910, 3.35)$ & $0.33 ,\ (-4.74  , 0.75 )$ \\
\hline
\multirow{4}{*}{\makecell[l]{California\\House Values}}
& \multirow{2}{*}{K$\ell_\infty$R}
& LOOCV & $43.1,\ (41.2, 50.3)$ & $-0.02,\ (-0.36  , 0.05 )$ \\
& & GCV & $43.5,\ (41.7, 52.8)$ & $-0.03,\ (-0.37  , 0.00 )$ \\
\cline{2-5}
& \multirow{2}{*}{KRR}
& LOOCV & $1.26,\ (0.928, 2.25)$ & $0.60 ,\ (-0.13  , 0.88 )$ \\
& & GCV & $1.26,\ (0.841, 2.16)$ & $0.47 ,\ (-0.39  , 0.83 )$ \\
\hline
\multirow{4}{*}{\makecell[l]{Steel Energy\\Consumption}}
& \multirow{2}{*}{K$\ell_\infty$R}
& LOOCV & $54.1,\ (51.7, 63.1)$ & $0.00 ,\ (-0.38  , 0.71 )$ \\
& & GCV & $54.3,\ (51.5, 68.5)$ & $0.57 ,\ (-0.08  , 0.98 )$ \\
\cline{2-5}
& \multirow{2}{*}{KRR}
& LOOCV & $0.891,\ (0.793, 1.63)$ & $0.99 ,\ (0.97   , 1.00 )$ \\
& & GCV & $0.892,\ (0.797, 1.57)$ & $0.98 ,\ (0.71   , 1.00 )$ \\
\hline
\multirow{4}{*}{\makecell[l]{Superconductor\\Critical\\Temperature}}
& \multirow{2}{*}{K$\ell_\infty$R}
& LOOCV & $52.4,\ (51.2, 56.0)$ & $-0.00,\ (-0.42  , 0.25 )$ \\
& & GCV & $52.8,\ (51.6, 56.1)$ & $0.44 ,\ (0.01   , 0.74 )$ \\
\cline{2-5}
& \multirow{2}{*}{KRR}
& LOOCV & $1.79,\ (1.02, 2.91)$ & $0.69 ,\ (0.12   , 0.87 )$ \\
& & GCV & $2.30,\ (1.15, 3.83)$ & $0.46 ,\ (-2.71  , 0.79 )$ \\
\hline
\multirow{4}{*}{\makecell[l]{U.K.\\Temperature}}
& \multirow{2}{*}{K$\ell_\infty$R}
& LOOCV & $48.5,\ (46.6, 64.7)$ & $-0.05,\ (-0.38  , 0.03 )$ \\
& & GCV & $48.6,\ (46.8, 61.3)$ & $-0.03,\ (-0.38  , 0.10 )$ \\
\cline{2-5}
& \multirow{2}{*}{KRR}
& LOOCV & $1.01,\ (0.860, 1.59)$ & $0.38 ,\ (-0.30  , 0.68 )$ \\
& & GCV & $0.993,\ (0.796, 1.65)$ & $0.28 ,\ (-6.81  , 0.71 )$ \\
\hline
\end{tabular}
\label{tab:real_lgcv_100_0}
\end{table}

\begin{table}
\caption{The 2.5th, 50th, and 97.5th percentiles of computation time and test $R^2$ for K$\ell_\infty$R in combination with LOOCV and GCV on the different data sets, \textbf{with} amplified outliers. Compared to 10-fold cross validation LOOCV and GCV perform faster, but the cost of decreased predictive performance, especially for K$\ell_\infty$R.}
\center
\begin{tabular}{|l|l|l|l|l|}
\hline
Data & Method & \makecell[l]{Type of\\Cross-\\Validation} & \makecell[l]{Computation\\Time [s]\\50\%,\ (2.5\%,\ 97.5\%)} & \makecell[l]{Test $R^2$\\50\%,\ (2.5\%,\ 97.5\%)}\\
\hline
\multirow{4}{*}{\makecell[l]{Airfoil Sound\\Pressure}}
& \multirow{2}{*}{K$\ell_\infty$R}
& LOOCV & $52.8,\ (48.9, 60.4)$ & $-0.02,\ (-14.57 , 0.01 )$ \\
& & GCV & $56.1,\ (49.7, 64.9)$ & $-0.02,\ (-14.33 , 0.08 )$ \\
\cline{2-5}
& \multirow{2}{*}{KRR}
& LOOCV & $0.840,\ (0.760, 0.925)$ & $0.30 ,\ (-1.12  , 0.66 )$ \\
& & GCV & $0.844,\ (0.754, 0.997)$ & $-0.05,\ (-189.78, 0.69 )$ \\
\hline
\multirow{4}{*}{\makecell[l]{California\\House Values}}
& \multirow{2}{*}{K$\ell_\infty$R}
& LOOCV & $44.7,\ (42.4, 50.7)$ & $-0.02,\ (-0.21  , 0.01 )$ \\
& & GCV & $45.0,\ (43.0, 50.3)$ & $-0.03,\ (-1.28  , 0.29 )$ \\
\cline{2-5}
& \multirow{2}{*}{KRR}
& LOOCV & $0.985,\ (0.832, 3.58)$ & $0.29 ,\ (-0.62  , 0.77 )$ \\
& & GCV & $1.01,\ (0.833, 4.69)$ & $0.16 ,\ (-37.61 , 0.72 )$ \\
\hline
\multirow{4}{*}{\makecell[l]{Steel Energy\\Consumption}}
& \multirow{2}{*}{K$\ell_\infty$R}
& LOOCV & $54.5,\ (52.3, 68.9)$ & $-0.04,\ (-0.94  , 0.12 )$ \\
& & GCV & $55.1,\ (52.3, 65.7)$ & $0.01 ,\ (-0.97  , 0.88 )$ \\
\cline{2-5}
& \multirow{2}{*}{KRR}
& LOOCV & $0.869,\ (0.753, 1.04)$ & $0.79 ,\ (-0.80  , 0.97 )$ \\
& & GCV & $0.889,\ (0.762, 1.05)$ & $0.13 ,\ (-405.97, 0.96 )$ \\
\hline
\multirow{4}{*}{\makecell[l]{Superconductor\\Critical\\Temperature}}
& \multirow{2}{*}{K$\ell_\infty$R}
& LOOCV & $52.9,\ (51.3, 61.0)$ & $-0.04,\ (-0.90  , 0.02 )$ \\
& & GCV & $53.3,\ (51.5, 62.1)$ & $0.47 ,\ (-0.19  , 0.78 )$ \\
\cline{2-5}
& \multirow{2}{*}{KRR}
& LOOCV & $1.04,\ (0.993, 1.23)$ & $0.54 ,\ (-0.46  , 0.75 )$ \\
& & GCV & $1.06,\ (0.987, 1.29)$ & $0.13 ,\ (-37.16 , 0.62 )$ \\
\hline
\multirow{4}{*}{\makecell[l]{U.K.\\Temperature}}
& \multirow{2}{*}{K$\ell_\infty$R}
& LOOCV & $51.2,\ (48.2, 59.1)$ & $-0.05,\ (-4.55  , 0.00 )$ \\
& & GCV & $51.2,\ (48.3, 58.7)$ & $-0.05,\ (-1.12  , 0.00 )$ \\
\cline{2-5}
& \multirow{2}{*}{KRR}
& LOOCV & $0.906,\ (0.818, 1.13)$ & $0.18 ,\ (-2.03  , 0.56 )$ \\
& & GCV & $0.887,\ (0.817, 1.14)$ & $-0.01,\ (-61.69 , 0.58 )$ \\
\hline
\end{tabular}
\label{tab:real_lgcv_100_1}
\end{table}

\section{Additional Experiments}
\label{sec:more_exps}

In this section, we extend the results of Tables \ref{tab:real_100_0} and \ref{tab:real_100_1} by repeating the experiments for the four kernels in Table \ref{tab:kerns}. Since, at the time of writing, the KQRT-T implementation only supported the Gaussian kernel, this algorithm is omitted in this section. The results are presented in Tables \ref{tab:real_0.5_0} to \ref{tab:real_10_1} and agree with those of Section \ref{sec:exp_real}.
\begin{table}
\caption{Additional kernels used}
\center
\begin{tabular}{|l|l|}
\hline
Name & Equation \\
\hline
Matérn, $\nu=\frac12$ (Laplace) & $\exp\left(-\frac{\|\xv-\xpv\|_2}\sigma\right)$\\
Matérn, $\nu=\frac32$  & $\left(1+\frac{\sqrt{3}\cdot\|\xv-\xpv\|_2}\sigma\right)\cdot\exp\left(-\frac{\sqrt{3}\cdot\|\xv-\xpv\|_2}\sigma\right)$\\
Matérn, $\nu=\frac52$  & $\left(1+\frac{\sqrt{5}\cdot\|\xv-\xpv\|_2}\sigma+\frac{5\cdot\|\xv-\xpv\|^2_2}{3\cdot\sigma^2}\right)\cdot\exp\left(-\frac{\sqrt{5}\cdot\|\xv-\xpv\|_2}\sigma\right)$\\
Cauchy & $\left(1+\frac{\|\xv-\xpv\|_2^2}{\sigma^2}\right)^{-1}$\\
\hline
\end{tabular}
\label{tab:kerns}
\end{table}

\begin{table}
\caption{The 2.5th, 50th, and 97.5th percentiles of computation time and test $R^2$ for the different methods and data sets, \textbf{without} amplified outliers, for $\nu=0.5$ (Laplace kernel). The five robust methods perform very similarly in terms of test $R^2$, while KSGD performs one to two orders of magnitude faster.}
\center
\begin{tabular}{|l|l|l|l|}
\hline
Data & Method & \makecell[l]{Computation Time [s]\\50\%,\ (2.5\%,\ 97.5\%)} & \makecell[l]{Test $R^2$\\50\%,\ (2.5\%,\ 97.5\%)}\\
\hline
\multirow{8}{*}{\makecell[l]{Airfoil Sound\\Pressure}}
& KSGD            & $9.81,\ (4.94, 14.7)$ & $0.49 ,\ (-0.39  , 0.78 )$ \\
& K$\ell_\infty$R & $383,\ (372, 416)$ & $0.50 ,\ (-0.39  , 0.81 )$ \\
\cline{2-4}
& KMR-H           & $113,\ (96.3, 127)$ & $0.52 ,\ (-0.50  , 0.83 )$ \\
& KMR-T           & $84.2,\ (73.5, 96.8)$ & $0.50 ,\ (-0.44  , 0.83 )$ \\
& KQR-A           & $235,\ (219, 253)$ & $0.50 ,\ (-0.44  , 0.83 )$ \\
\cline{2-4}
& KGD             & $20.2,\ (12.7, 26.7)$ & $0.51 ,\ (-0.39  , 0.84 )$ \\
& KRR             & $2.64,\ (2.59, 3.24)$ & $0.54 ,\ (-0.38  , 0.84 )$ \\
\hline
\multirow{8}{*}{\makecell[l]{California\\House Values}}
& KSGD            & $3.58,\ (2.63, 4.97)$ & $0.62 ,\ (-0.09  , 0.87 )$ \\
& K$\ell_\infty$R & $331,\ (312, 361)$ & $0.63 ,\ (-0.10  , 0.84 )$ \\
\cline{2-4}
& KMR-H           & $95.6,\ (81.9, 118)$ & $0.66 ,\ (-0.20  , 0.86 )$ \\
& KMR-T           & $69.3,\ (57.7, 86.6)$ & $0.65 ,\ (-0.21  , 0.87 )$ \\
& KQR-A           & $222,\ (191, 241)$ & $0.67 ,\ (-0.18  , 0.86 )$ \\
\cline{2-4}
& KGD             & $17.9,\ (12.4, 25.1)$ & $0.66 ,\ (-0.13  , 0.85 )$ \\
& KRR             & $2.97,\ (2.64, 4.08)$ & $0.66 ,\ (-0.12  , 0.85 )$ \\
\hline
\multirow{8}{*}{\makecell[l]{Steel Energy\\Consumption}}
& KSGD            & $17.9,\ (14.6, 28.4)$ & $0.98 ,\ (0.85   , 0.99 )$ \\
& K$\ell_\infty$R & $464,\ (408, 478)$ & $0.98 ,\ (0.93   , 0.99 )$ \\
\cline{2-4}
& KMR-H           & $119,\ (91.7, 139)$ & $0.99 ,\ (0.95   , 1.00 )$ \\
& KMR-T           & $97.5,\ (78.7, 107)$ & $0.98 ,\ (0.95   , 1.00 )$ \\
& KQR-A           & $254,\ (208, 268)$ & $0.99 ,\ (0.95   , 1.00 )$ \\
\cline{2-4}
& KGD             & $26.2,\ (19.3, 33.4)$ & $0.98 ,\ (0.93   , 0.99 )$ \\
& KRR             & $2.98,\ (2.64, 3.34)$ & $0.99 ,\ (0.95   , 1.00 )$ \\
\hline
\multirow{8}{*}{\makecell[l]{Superconductor\\Critical\\Temperature}}
& KSGD            & $18.0,\ (13.1, 22.1)$ & $0.58 ,\ (0.04   , 0.88 )$ \\
& K$\ell_\infty$R & $443,\ (404, 491)$ & $0.65 ,\ (0.10   , 0.90 )$ \\
\cline{2-4}
& KMR-H           & $129,\ (106, 161)$ & $0.65 ,\ (-0.01  , 0.87 )$ \\
& KMR-T           & $99.4,\ (76.0, 124)$ & $0.67 ,\ (0.06   , 0.89 )$ \\
& KQR-A           & $275,\ (213, 335)$ & $0.65 ,\ (0.06   , 0.90 )$ \\
\cline{2-4}
& KGD             & $15.3,\ (10.6, 22.8)$ & $0.67 ,\ (0.12   , 0.90 )$ \\
& KRR             & $3.38,\ (3.16, 4.13)$ & $0.65 ,\ (0.14   , 0.88 )$ \\
\hline
\multirow{8}{*}{\makecell[l]{U.K.\\Temperature}}
& KSGD            & $7.43,\ (4.34, 11.5)$ & $0.43 ,\ (0.04   , 0.67 )$ \\
& K$\ell_\infty$R & $382,\ (363, 420)$ & $0.48 ,\ (0.02   , 0.70 )$ \\
\cline{2-4}
& KMR-H           & $118,\ (96.8, 142)$ & $0.46 ,\ (0.01   , 0.68 )$ \\
& KMR-T           & $87.2,\ (69.6, 110)$ & $0.46 ,\ (0.01   , 0.69 )$ \\
& KQR-A           & $250,\ (226, 297)$ & $0.47 ,\ (-0.10  , 0.69 )$ \\
\cline{2-4}
& KGD             & $18.6,\ (14.2, 25.4)$ & $0.47 ,\ (0.01   , 0.69 )$ \\
& KRR             & $2.74,\ (2.61, 3.03)$ & $0.46 ,\ (0.02   , 0.69 )$ \\
\hline
\end{tabular}
\label{tab:real_0.5_0}
\end{table}

\begin{table}
\caption{The 2.5th, 50th, and 97.5th percentiles of computation time and test $R^2$ for the different methods and data sets, \textbf{with} amplified outliers, for $\nu=0.5$ (Laplace kernel). The five robust methods perform very similarly in terms of test $R^2$, while KSGD performs one to two orders of magnitude faster. The two non-robust methods, KGD/KRR, do not perform as well as the robust methods.}
\center
\begin{tabular}{|l|l|l|l|}
\hline
Data & Method & \makecell[l]{Computation Time [s]\\50\%,\ (2.5\%,\ 97.5\%)} & \makecell[l]{Test $R^2$\\50\%,\ (2.5\%,\ 97.5\%)}\\
\hline
\multirow{8}{*}{\makecell[l]{Airfoil Sound\\Pressure}}
& KSGD            & $11.7,\ (4.46, 15.2)$ & $0.41 ,\ (-0.07  , 0.70 )$ \\
& K$\ell_\infty$R & $386,\ (375, 417)$ & $0.43 ,\ (-0.73  , 0.73 )$ \\
\cline{2-4}
& KMR-H           & $82.1,\ (38.3, 116)$ & $0.41 ,\ (-0.39  , 0.75 )$ \\
& KMR-T           & $70.9,\ (36.9, 91.1)$ & $0.40 ,\ (-0.14  , 0.76 )$ \\
& KQR-A           & $251,\ (223, 283)$ & $0.43 ,\ (-0.07  , 0.74 )$ \\
\cline{2-4}
& KGD             & $15.0,\ (3.83, 22.1)$ & $0.31 ,\ (-57.39 , 0.75 )$ \\
& KRR             & $2.64,\ (2.60, 2.89)$ & $0.32 ,\ (-1.62  , 0.72 )$ \\
\hline
\multirow{8}{*}{\makecell[l]{California\\House Values}}
& KSGD            & $4.05,\ (2.77, 6.99)$ & $0.49 ,\ (-0.20  , 0.77 )$ \\
& K$\ell_\infty$R & $344,\ (326, 382)$ & $0.49 ,\ (-0.26  , 0.79 )$ \\
\cline{2-4}
& KMR-H           & $72.2,\ (38.5, 111)$ & $0.46 ,\ (-0.32  , 0.80 )$ \\
& KMR-T           & $58.4,\ (35.8, 79.1)$ & $0.48 ,\ (-0.22  , 0.80 )$ \\
& KQR-A           & $226,\ (202, 266)$ & $0.48 ,\ (-0.30  , 0.79 )$ \\
\cline{2-4}
& KGD             & $12.7,\ (3.83, 18.7)$ & $0.34 ,\ (-2.69  , 0.80 )$ \\
& KRR             & $2.73,\ (2.66, 3.18)$ & $0.37 ,\ (-0.33  , 0.80 )$ \\
\hline
\multirow{8}{*}{\makecell[l]{Steel Energy\\Consumption}}
& KSGD            & $22.1,\ (17.5, 34.1)$ & $0.89 ,\ (0.07   , 0.97 )$ \\
& K$\ell_\infty$R & $433,\ (403, 516)$ & $0.84 ,\ (-0.76  , 0.97 )$ \\
\cline{2-4}
& KMR-H           & $81.8,\ (45.7, 125)$ & $0.87 ,\ (-1.19  , 0.97 )$ \\
& KMR-T           & $72.2,\ (46.7, 107)$ & $0.88 ,\ (-0.07  , 0.97 )$ \\
& KQR-A           & $263,\ (230, 330)$ & $0.88 ,\ (-1.58  , 0.97 )$ \\
\cline{2-4}
& KGD             & $11.6,\ (6.34, 17.5)$ & $0.75 ,\ (-31.91 , 0.96 )$ \\
& KRR             & $2.73,\ (2.65, 2.92)$ & $0.79 ,\ (-0.36  , 0.96 )$ \\
\hline
\multirow{8}{*}{\makecell[l]{Superconductor\\Critical\\Temperature}}
& KSGD            & $21.2,\ (15.0, 27.7)$ & $0.47 ,\ (-0.10  , 0.82 )$ \\
& K$\ell_\infty$R & $445,\ (400, 476)$ & $0.50 ,\ (-0.28  , 0.85 )$ \\
\cline{2-4}
& KMR-H           & $112,\ (55.7, 163)$ & $0.48 ,\ (-0.38  , 0.84 )$ \\
& KMR-T           & $89.1,\ (52.6, 158)$ & $0.49 ,\ (-0.30  , 0.84 )$ \\
& KQR-A           & $282,\ (222, 332)$ & $0.48 ,\ (-0.10  , 0.84 )$ \\
\cline{2-4}
& KGD             & $11.5,\ (4.98, 17.9)$ & $0.44 ,\ (-0.96  , 0.82 )$ \\
& KRR             & $3.26,\ (3.13, 3.76)$ & $0.46 ,\ (-0.44  , 0.81 )$ \\
\hline
\multirow{8}{*}{\makecell[l]{U.K.\\Temperature}}
& KSGD            & $9.33,\ (5.15, 14.9)$ & $0.36 ,\ (-0.07  , 0.63 )$ \\
& K$\ell_\infty$R & $383,\ (367, 461)$ & $0.27 ,\ (-0.47  , 0.63 )$ \\
\cline{2-4}
& KMR-H           & $78.5,\ (38.0, 135)$ & $0.32 ,\ (-0.10  , 0.65 )$ \\
& KMR-T           & $65.9,\ (36.2, 98.6)$ & $0.35 ,\ (-0.20  , 0.64 )$ \\
& KQR-A           & $249,\ (217, 312)$ & $0.37 ,\ (-0.25  , 0.69 )$ \\
\cline{2-4}
& KGD             & $13.8,\ (3.81, 22.6)$ & $0.18 ,\ (-19.57 , 0.62 )$ \\
& KRR             & $2.66,\ (2.62, 3.22)$ & $0.16 ,\ (-1.83  , 0.63 )$ \\
\hline
\end{tabular}
\label{tab:real_0.5_1}
\end{table}

\begin{table}
\caption{The 2.5th, 50th, and 97.5th percentiles of computation time and test $R^2$ for the different methods and data sets, \textbf{without} amplified outliers, for $\nu=1.5$. The five robust methods perform very similarly in terms of test $R^2$, while KSGD performs one to two orders of magnitude faster.}
\center
\begin{tabular}{|l|l|l|l|}
\hline
Data & Method & \makecell[l]{Computation Time [s]\\50\%,\ (2.5\%,\ 97.5\%)} & \makecell[l]{Test $R^2$\\50\%,\ (2.5\%,\ 97.5\%)}\\
\hline
\multirow{8}{*}{\makecell[l]{Airfoil Sound\\Pressure}}
& KSGD            & $10.2,\ (5.42, 14.7)$ & $0.50 ,\ (-0.42  , 0.83 )$ \\
& K$\ell_\infty$R & $411,\ (375, 435)$ & $0.53 ,\ (-0.52  , 0.84 )$ \\
\cline{2-4}
& KMR-H           & $184,\ (145, 208)$ & $0.54 ,\ (-0.41  , 0.81 )$ \\
& KMR-T           & $163,\ (135, 186)$ & $0.55 ,\ (-0.43  , 0.85 )$ \\
& KQR-A           & $329,\ (264, 344)$ & $0.52 ,\ (-0.42  , 0.85 )$ \\
\cline{2-4}
& KGD             & $17.4,\ (11.9, 21.2)$ & $0.56 ,\ (-0.50  , 0.85 )$ \\
& KRR             & $3.03,\ (2.95, 3.52)$ & $0.56 ,\ (-0.41  , 0.85 )$ \\
\hline
\multirow{8}{*}{\makecell[l]{California\\House Values}}
& KSGD            & $3.74,\ (2.08, 7.04)$ & $0.63 ,\ (-0.11  , 0.85 )$ \\
& K$\ell_\infty$R & $330,\ (314, 351)$ & $0.64 ,\ (-0.19  , 0.89 )$ \\
\cline{2-4}
& KMR-H           & $134,\ (120, 148)$ & $0.62 ,\ (-0.19  , 0.87 )$ \\
& KMR-T           & $110,\ (98.8, 118)$ & $0.61 ,\ (-0.25  , 0.87 )$ \\
& KQR-A           & $279,\ (245, 289)$ & $0.61 ,\ (-0.07  , 0.88 )$ \\
\cline{2-4}
& KGD             & $14.6,\ (10.6, 18.2)$ & $0.64 ,\ (-0.35  , 0.89 )$ \\
& KRR             & $3.24,\ (3.06, 3.41)$ & $0.63 ,\ (-0.10  , 0.85 )$ \\
\hline
\multirow{8}{*}{\makecell[l]{Steel Energy\\Consumption}}
& KSGD            & $15.0,\ (12.4, 23.4)$ & $0.98 ,\ (0.90   , 1.00 )$ \\
& K$\ell_\infty$R & $420,\ (404, 449)$ & $0.98 ,\ (0.94   , 0.99 )$ \\
\cline{2-4}
& KMR-H           & $162,\ (147, 175)$ & $0.99 ,\ (0.97   , 1.00 )$ \\
& KMR-T           & $145,\ (133, 157)$ & $0.99 ,\ (0.97   , 1.00 )$ \\
& KQR-A           & $278,\ (244, 289)$ & $0.99 ,\ (0.97   , 1.00 )$ \\
\cline{2-4}
& KGD             & $22.9,\ (18.8, 26.2)$ & $0.99 ,\ (0.94   , 1.00 )$ \\
& KRR             & $3.11,\ (3.04, 3.28)$ & $0.99 ,\ (0.97   , 1.00 )$ \\
\hline
\multirow{8}{*}{\makecell[l]{Superconductor\\Critical\\Temperature}}
& KSGD            & $16.2,\ (11.4, 21.6)$ & $0.58 ,\ (0.07   , 0.86 )$ \\
& K$\ell_\infty$R & $415,\ (399, 495)$ & $0.65 ,\ (0.17   , 0.91 )$ \\
\cline{2-4}
& KMR-H           & $179,\ (142, 228)$ & $0.64 ,\ (0.16   , 0.90 )$ \\
& KMR-T           & $156,\ (119, 204)$ & $0.65 ,\ (0.14   , 0.88 )$ \\
& KQR-A           & $306,\ (252, 373)$ & $0.61 ,\ (-0.00  , 0.90 )$ \\
\cline{2-4}
& KGD             & $16.3,\ (12.1, 22.4)$ & $0.68 ,\ (0.09   , 0.90 )$ \\
& KRR             & $4.06,\ (3.53, 4.52)$ & $0.65 ,\ (0.13   , 0.89 )$ \\
\hline
\multirow{8}{*}{\makecell[l]{U.K.\\Temperature}}
& KSGD            & $6.36,\ (3.65, 9.72)$ & $0.42 ,\ (0.01   , 0.72 )$ \\
& K$\ell_\infty$R & $367,\ (358, 379)$ & $0.43 ,\ (-0.11  , 0.71 )$ \\
\cline{2-4}
& KMR-H           & $162,\ (144, 181)$ & $0.44 ,\ (-0.27  , 0.68 )$ \\
& KMR-T           & $148,\ (129, 164)$ & $0.44 ,\ (-0.25  , 0.72 )$ \\
& KQR-A           & $285,\ (259, 333)$ & $0.39 ,\ (-0.23  , 0.70 )$ \\
\cline{2-4}
& KGD             & $12.8,\ (9.70, 18.0)$ & $0.43 ,\ (-0.20  , 0.71 )$ \\
& KRR             & $3.05,\ (2.98, 3.51)$ & $0.46 ,\ (-0.25  , 0.68 )$ \\
\hline
\end{tabular}
\label{tab:real_1.5_0}
\end{table}

\begin{table}
\caption{The 2.5th, 50th, and 97.5th percentiles of computation time and test $R^2$ for the different methods and data sets, \textbf{with} amplified outliers, for $\nu=1.5$. The five robust methods perform very similarly in terms of test $R^2$, while KSGD performs one to two orders of magnitude faster. The two non-robust methods, KGD/KRR, do not perform as well as the robust methods.}
\center
\begin{tabular}{|l|l|l|l|}
\hline
Data & Method & \makecell[l]{Computation Time [s]\\50\%,\ (2.5\%,\ 97.5\%)} & \makecell[l]{Test $R^2$\\50\%,\ (2.5\%,\ 97.5\%)}\\
\hline
\multirow{8}{*}{\makecell[l]{Airfoil Sound\\Pressure}}
& KSGD            & $10.9,\ (6.35, 14.8)$ & $0.40 ,\ (-0.25  , 0.74 )$ \\
& K$\ell_\infty$R & $387,\ (378, 412)$ & $0.44 ,\ (-0.77  , 0.74 )$ \\
\cline{2-4}
& KMR-H           & $141,\ (82.1, 169)$ & $0.39 ,\ (-0.14  , 0.74 )$ \\
& KMR-T           & $131,\ (93.0, 153)$ & $0.41 ,\ (-0.10  , 0.77 )$ \\
& KQR-A           & $289,\ (271, 304)$ & $0.39 ,\ (-0.15  , 0.75 )$ \\
\cline{2-4}
& KGD             & $13.6,\ (3.52, 18.5)$ & $0.25 ,\ (-80.63 , 0.72 )$ \\
& KRR             & $3.03,\ (2.99, 4.03)$ & $0.27 ,\ (-1.61  , 0.71 )$ \\
\hline
\multirow{8}{*}{\makecell[l]{California\\House Values}}
& KSGD            & $4.78,\ (2.53, 8.28)$ & $0.48 ,\ (-0.04  , 0.81 )$ \\
& K$\ell_\infty$R & $363,\ (326, 394)$ & $0.48 ,\ (-0.55  , 0.80 )$ \\
\cline{2-4}
& KMR-H           & $109,\ (69.8, 146)$ & $0.44 ,\ (-0.62  , 0.81 )$ \\
& KMR-T           & $98.2,\ (75.5, 127)$ & $0.48 ,\ (-0.54  , 0.82 )$ \\
& KQR-A           & $290,\ (245, 326)$ & $0.38 ,\ (-0.81  , 0.83 )$ \\
\cline{2-4}
& KGD             & $11.9,\ (5.15, 14.8)$ & $0.24 ,\ (-4.82  , 0.75 )$ \\
& KRR             & $3.12,\ (3.06, 3.47)$ & $0.33 ,\ (-0.28  , 0.78 )$ \\
\hline
\multirow{8}{*}{\makecell[l]{Steel Energy\\Consumption}}
& KSGD            & $20.8,\ (15.9, 31.9)$ & $0.89 ,\ (0.06   , 0.97 )$ \\
& K$\ell_\infty$R & $495,\ (413, 530)$ & $0.83 ,\ (-0.59  , 0.98 )$ \\
\cline{2-4}
& KMR-H           & $161,\ (123, 211)$ & $0.88 ,\ (-1.32  , 0.98 )$ \\
& KMR-T           & $169,\ (134, 204)$ & $0.87 ,\ (-0.02  , 0.98 )$ \\
& KQR-A           & $344,\ (264, 373)$ & $0.87 ,\ (-1.30  , 0.98 )$ \\
\cline{2-4}
& KGD             & $13.8,\ (6.79, 18.3)$ & $0.74 ,\ (-28.49 , 0.97 )$ \\
& KRR             & $3.09,\ (3.02, 3.75)$ & $0.80 ,\ (-0.03  , 0.97 )$ \\
\hline
\multirow{8}{*}{\makecell[l]{Superconductor\\Critical\\Temperature}}
& KSGD            & $18.1,\ (13.9, 23.1)$ & $0.46 ,\ (-0.29  , 0.83 )$ \\
& K$\ell_\infty$R & $426,\ (409, 472)$ & $0.48 ,\ (-0.50  , 0.85 )$ \\
\cline{2-4}
& KMR-H           & $156,\ (95.1, 201)$ & $0.49 ,\ (-0.43  , 0.82 )$ \\
& KMR-T           & $137,\ (96.7, 200)$ & $0.50 ,\ (-0.44  , 0.86 )$ \\
& KQR-A           & $289,\ (247, 325)$ & $0.44 ,\ (-0.55  , 0.83 )$ \\
\cline{2-4}
& KGD             & $12.5,\ (4.72, 16.7)$ & $0.37 ,\ (-1.33  , 0.83 )$ \\
& KRR             & $3.61,\ (3.56, 4.25)$ & $0.46 ,\ (-0.46  , 0.82 )$ \\
\hline
\multirow{8}{*}{\makecell[l]{U.K.\\Temperature}}
& KSGD            & $9.54,\ (4.72, 14.7)$ & $0.35 ,\ (-0.08  , 0.63 )$ \\
& K$\ell_\infty$R & $460,\ (367, 491)$ & $0.27 ,\ (-0.60  , 0.62 )$ \\
\cline{2-4}
& KMR-H           & $153,\ (92.4, 226)$ & $0.33 ,\ (-0.70  , 0.66 )$ \\
& KMR-T           & $146,\ (103, 204)$ & $0.30 ,\ (-0.71  , 0.65 )$ \\
& KQR-A           & $363,\ (266, 393)$ & $0.32 ,\ (-0.88  , 0.66 )$ \\
\cline{2-4}
& KGD             & $11.6,\ (4.35, 17.7)$ & $0.14 ,\ (-24.32 , 0.61 )$ \\
& KRR             & $3.11,\ (3.01, 4.03)$ & $0.21 ,\ (-1.45  , 0.61 )$ \\
\hline
\end{tabular}
\label{tab:real_1.5_1}
\end{table}

\begin{table}
\caption{The 2.5th, 50th, and 97.5th percentiles of computation time and test $R^2$ for the different methods and data sets, \textbf{without} amplified outliers, for $\nu=2.5$. The five robust methods perform very similarly in terms of test $R^2$, while KSGD performs one to two orders of magnitude faster.}
\center
\begin{tabular}{|l|l|l|l|}
\hline
Data & Method & \makecell[l]{Computation Time [s]\\50\%,\ (2.5\%,\ 97.5\%)} & \makecell[l]{Test $R^2$\\50\%,\ (2.5\%,\ 97.5\%)}\\
\hline
\multirow{8}{*}{\makecell[l]{Airfoil Sound\\Pressure}}
& KSGD            & $9.84,\ (5.66, 15.7)$ & $0.50 ,\ (-0.51  , 0.82 )$ \\
& K$\ell_\infty$R & $409,\ (371, 455)$ & $0.53 ,\ (-0.50  , 0.85 )$ \\
\cline{2-4}
& KMR-H           & $190,\ (152, 209)$ & $0.49 ,\ (-0.61  , 0.82 )$ \\
& KMR-T           & $177,\ (144, 197)$ & $0.52 ,\ (-0.43  , 0.82 )$ \\
& KQR-A           & $327,\ (270, 352)$ & $0.51 ,\ (-0.60  , 0.83 )$ \\
\cline{2-4}
& KGD             & $16.7,\ (11.5, 20.6)$ & $0.56 ,\ (-0.50  , 0.86 )$ \\
& KRR             & $3.27,\ (3.18, 3.68)$ & $0.54 ,\ (-0.43  , 0.85 )$ \\
\hline
\multirow{8}{*}{\makecell[l]{California\\House Values}}
& KSGD            & $4.35,\ (2.34, 7.86)$ & $0.60 ,\ (0.02   , 0.86 )$ \\
& K$\ell_\infty$R & $343,\ (317, 383)$ & $0.63 ,\ (-0.11  , 0.89 )$ \\
\cline{2-4}
& KMR-H           & $155,\ (122, 176)$ & $0.62 ,\ (-0.18  , 0.87 )$ \\
& KMR-T           & $130,\ (106, 147)$ & $0.59 ,\ (-0.29  , 0.87 )$ \\
& KQR-A           & $301,\ (255, 328)$ & $0.60 ,\ (-0.07  , 0.88 )$ \\
\cline{2-4}
& KGD             & $14.4,\ (10.4, 19.0)$ & $0.63 ,\ (-0.14  , 0.89 )$ \\
& KRR             & $3.56,\ (3.29, 3.84)$ & $0.62 ,\ (-0.06  , 0.85 )$ \\
\hline
\multirow{8}{*}{\makecell[l]{Steel Energy\\Consumption}}
& KSGD            & $14.7,\ (12.2, 23.6)$ & $0.98 ,\ (0.91   , 0.99 )$ \\
& K$\ell_\infty$R & $439,\ (406, 481)$ & $0.98 ,\ (0.95   , 1.00 )$ \\
\cline{2-4}
& KMR-H           & $178,\ (151, 200)$ & $0.99 ,\ (0.98   , 1.00 )$ \\
& KMR-T           & $167,\ (147, 178)$ & $0.99 ,\ (0.97   , 1.00 )$ \\
& KQR-A           & $290,\ (251, 311)$ & $0.99 ,\ (0.98   , 1.00 )$ \\
\cline{2-4}
& KGD             & $23.1,\ (18.6, 27.0)$ & $0.99 ,\ (0.95   , 1.00 )$ \\
& KRR             & $3.36,\ (3.27, 3.64)$ & $0.99 ,\ (0.98   , 1.00 )$ \\
\hline
\multirow{8}{*}{\makecell[l]{Superconductor\\Critical\\Temperature}}
& KSGD            & $14.9,\ (10.9, 19.9)$ & $0.58 ,\ (0.01   , 0.87 )$ \\
& K$\ell_\infty$R & $431,\ (406, 491)$ & $0.63 ,\ (0.16   , 0.91 )$ \\
\cline{2-4}
& KMR-H           & $185,\ (158, 231)$ & $0.64 ,\ (0.05   , 0.91 )$ \\
& KMR-T           & $166,\ (135, 209)$ & $0.64 ,\ (0.10   , 0.88 )$ \\
& KQR-A           & $315,\ (264, 366)$ & $0.62 ,\ (-0.05  , 0.89 )$ \\
\cline{2-4}
& KGD             & $15.5,\ (12.6, 20.5)$ & $0.68 ,\ (0.16   , 0.90 )$ \\
& KRR             & $4.15,\ (3.77, 4.42)$ & $0.65 ,\ (0.13   , 0.89 )$ \\
\hline
\multirow{8}{*}{\makecell[l]{U.K.\\Temperature}}
& KSGD            & $7.07,\ (3.47, 11.4)$ & $0.40 ,\ (-0.01  , 0.73 )$ \\
& K$\ell_\infty$R & $394,\ (362, 460)$ & $0.41 ,\ (-0.14  , 0.70 )$ \\
\cline{2-4}
& KMR-H           & $183,\ (162, 228)$ & $0.41 ,\ (-0.58  , 0.66 )$ \\
& KMR-T           & $177,\ (152, 205)$ & $0.43 ,\ (-0.58  , 0.68 )$ \\
& KQR-A           & $314,\ (284, 380)$ & $0.35 ,\ (-0.95  , 0.69 )$ \\
\cline{2-4}
& KGD             & $12.6,\ (9.63, 17.3)$ & $0.42 ,\ (-0.22  , 0.67 )$ \\
& KRR             & $3.53,\ (3.24, 4.04)$ & $0.42 ,\ (-0.58  , 0.66 )$ \\
\hline
\end{tabular}
\label{tab:real_2.5_0}
\end{table}

\begin{table}
\caption{The 2.5th, 50th, and 97.5th percentiles of computation time and test $R^2$ for the different methods and data sets, \textbf{with} amplified outliers, for $\nu=2.5$. The five robust methods perform very similarly in terms of test $R^2$, while KSGD performs one to two orders of magnitude faster. The two non-robust methods, KGD/KRR, do not perform as well as the robust methods.}
\center
\begin{tabular}{|l|l|l|l|}
\hline
Data & Method & \makecell[l]{Computation Time [s]\\50\%,\ (2.5\%,\ 97.5\%)} & \makecell[l]{Test $R^2$\\50\%,\ (2.5\%,\ 97.5\%)}\\
\hline
\multirow{8}{*}{\makecell[l]{Airfoil Sound\\Pressure}}
& KSGD            & $10.8,\ (6.34, 14.3)$ & $0.41 ,\ (-0.25  , 0.71 )$ \\
& K$\ell_\infty$R & $389,\ (376, 410)$ & $0.43 ,\ (-0.67  , 0.75 )$ \\
\cline{2-4}
& KMR-H           & $150,\ (90.2, 177)$ & $0.39 ,\ (-0.52  , 0.68 )$ \\
& KMR-T           & $140,\ (105, 160)$ & $0.37 ,\ (-0.57  , 0.79 )$ \\
& KQR-A           & $292,\ (272, 310)$ & $0.39 ,\ (-0.86  , 0.70 )$ \\
\cline{2-4}
& KGD             & $13.4,\ (3.63, 17.6)$ & $0.23 ,\ (-92.96 , 0.68 )$ \\
& KRR             & $3.26,\ (3.19, 3.35)$ & $0.27 ,\ (-1.61  , 0.70 )$ \\
\hline
\multirow{8}{*}{\makecell[l]{California\\House Values}}
& KSGD            & $4.89,\ (2.78, 10.2)$ & $0.46 ,\ (-0.05  , 0.82 )$ \\
& K$\ell_\infty$R & $352,\ (326, 391)$ & $0.43 ,\ (-0.41  , 0.80 )$ \\
\cline{2-4}
& KMR-H           & $119,\ (80.8, 157)$ & $0.40 ,\ (-0.70  , 0.81 )$ \\
& KMR-T           & $114,\ (88.8, 133)$ & $0.44 ,\ (-0.55  , 0.80 )$ \\
& KQR-A           & $282,\ (266, 328)$ & $0.34 ,\ (-0.78  , 0.82 )$ \\
\cline{2-4}
& KGD             & $11.8,\ (5.21, 15.1)$ & $0.20 ,\ (-3.93  , 0.79 )$ \\
& KRR             & $3.34,\ (3.29, 4.09)$ & $0.31 ,\ (-0.27  , 0.79 )$ \\
\hline
\multirow{8}{*}{\makecell[l]{Steel Energy\\Consumption}}
& KSGD            & $17.4,\ (13.2, 25.6)$ & $0.90 ,\ (0.06   , 0.97 )$ \\
& K$\ell_\infty$R & $416,\ (401, 573)$ & $0.85 ,\ (-0.70  , 0.98 )$ \\
\cline{2-4}
& KMR-H           & $149,\ (116, 173)$ & $0.87 ,\ (-1.92  , 0.98 )$ \\
& KMR-T           & $151,\ (132, 167)$ & $0.87 ,\ (-0.02  , 0.98 )$ \\
& KQR-A           & $275,\ (259, 311)$ & $0.86 ,\ (-2.41  , 0.98 )$ \\
\cline{2-4}
& KGD             & $13.8,\ (6.41, 19.5)$ & $0.76 ,\ (-31.99 , 0.97 )$ \\
& KRR             & $3.34,\ (3.30, 3.59)$ & $0.80 ,\ (-5.16  , 0.97 )$ \\
\hline
\multirow{8}{*}{\makecell[l]{Superconductor\\Critical\\Temperature}}
& KSGD            & $16.6,\ (12.4, 22.5)$ & $0.46 ,\ (-0.40  , 0.82 )$ \\
& K$\ell_\infty$R & $417,\ (399, 478)$ & $0.49 ,\ (-0.36  , 0.85 )$ \\
\cline{2-4}
& KMR-H           & $163,\ (99.4, 211)$ & $0.49 ,\ (-0.45  , 0.80 )$ \\
& KMR-T           & $146,\ (105, 204)$ & $0.49 ,\ (-0.41  , 0.85 )$ \\
& KQR-A           & $287,\ (258, 370)$ & $0.48 ,\ (-0.47  , 0.82 )$ \\
\cline{2-4}
& KGD             & $12.3,\ (4.91, 17.2)$ & $0.39 ,\ (-1.28  , 0.83 )$ \\
& KRR             & $3.84,\ (3.76, 4.70)$ & $0.47 ,\ (-0.51  , 0.80 )$ \\
\hline
\multirow{8}{*}{\makecell[l]{U.K.\\Temperature}}
& KSGD            & $8.31,\ (4.60, 12.8)$ & $0.33 ,\ (-0.12  , 0.62 )$ \\
& K$\ell_\infty$R & $384,\ (368, 455)$ & $0.25 ,\ (-1.11  , 0.60 )$ \\
\cline{2-4}
& KMR-H           & $149,\ (98.7, 196)$ & $0.27 ,\ (-2.30  , 0.65 )$ \\
& KMR-T           & $139,\ (107, 179)$ & $0.30 ,\ (-2.36  , 0.66 )$ \\
& KQR-A           & $291,\ (272, 354)$ & $0.26 ,\ (-1.52  , 0.69 )$ \\
\cline{2-4}
& KGD             & $10.8,\ (3.44, 17.3)$ & $0.11 ,\ (-20.02 , 0.61 )$ \\
& KRR             & $3.28,\ (3.22, 3.88)$ & $0.19 ,\ (-2.27  , 0.61 )$ \\
\hline
\end{tabular}
\label{tab:real_2.5_1}
\end{table}

\begin{table}
\caption{The 2.5th, 50th, and 97.5th percentiles of computation time and test $R^2$ for the different methods and data sets, \textbf{without} amplified outliers, for the Cauchy kernel. The five robust methods perform very similarly in terms of test $R^2$, while KSGD performs one to two orders of magnitude faster.}
\center
\begin{tabular}{|l|l|l|l|}
\hline
Data & Method & \makecell[l]{Computation Time [s]\\50\%,\ (2.5\%,\ 97.5\%)} & \makecell[l]{Test $R^2$\\50\%,\ (2.5\%,\ 97.5\%)}\\
\hline
\multirow{8}{*}{\makecell[l]{Airfoil Sound\\Pressure}}
& KSGD            & $10.6,\ (6.23, 15.9)$ & $0.49 ,\ (-0.46  , 0.81 )$ \\
& K$\ell_\infty$R & $403,\ (375, 445)$ & $0.53 ,\ (-0.54  , 0.82 )$ \\
\cline{2-4}
& KMR-H           & $189,\ (156, 209)$ & $0.48 ,\ (-0.59  , 0.84 )$ \\
& KMR-T           & $177,\ (145, 196)$ & $0.52 ,\ (-0.43  , 0.82 )$ \\
& KQR-A           & $326,\ (273, 348)$ & $0.53 ,\ (-0.53  , 0.83 )$ \\
\cline{2-4}
& KGD             & $17.7,\ (12.2, 22.6)$ & $0.54 ,\ (-0.51  , 0.87 )$ \\
& KRR             & $2.26,\ (2.20, 2.57)$ & $0.52 ,\ (-0.42  , 0.84 )$ \\
\hline
\multirow{8}{*}{\makecell[l]{California\\House Values}}
& KSGD            & $4.83,\ (2.93, 7.88)$ & $0.63 ,\ (0.02   , 0.89 )$ \\
& K$\ell_\infty$R & $352,\ (313, 405)$ & $0.64 ,\ (-0.19  , 0.89 )$ \\
\cline{2-4}
& KMR-H           & $154,\ (131, 175)$ & $0.60 ,\ (-0.13  , 0.87 )$ \\
& KMR-T           & $122,\ (106, 145)$ & $0.59 ,\ (-0.31  , 0.86 )$ \\
& KQR-A           & $288,\ (253, 334)$ & $0.56 ,\ (-0.21  , 0.88 )$ \\
\cline{2-4}
& KGD             & $17.2,\ (13.5, 23.2)$ & $0.62 ,\ (-0.26  , 0.88 )$ \\
& KRR             & $2.51,\ (2.29, 2.88)$ & $0.61 ,\ (-0.11  , 0.87 )$ \\
\hline
\multirow{8}{*}{\makecell[l]{Steel Energy\\Consumption}}
& KSGD            & $21.4,\ (16.4, 28.5)$ & $0.98 ,\ (0.89   , 0.99 )$ \\
& K$\ell_\infty$R & $454,\ (421, 518)$ & $0.98 ,\ (0.94   , 0.99 )$ \\
\cline{2-4}
& KMR-H           & $186,\ (150, 208)$ & $1.00 ,\ (0.98   , 1.00 )$ \\
& KMR-T           & $176,\ (145, 195)$ & $0.99 ,\ (0.97   , 1.00 )$ \\
& KQR-A           & $302,\ (258, 331)$ & $0.99 ,\ (0.97   , 1.00 )$ \\
\cline{2-4}
& KGD             & $24.4,\ (20.1, 29.0)$ & $0.99 ,\ (0.94   , 1.00 )$ \\
& KRR             & $2.36,\ (2.26, 2.68)$ & $0.99 ,\ (0.97   , 1.00 )$ \\
\hline
\multirow{8}{*}{\makecell[l]{Superconductor\\Critical\\Temperature}}
& KSGD            & $23.7,\ (17.4, 31.2)$ & $0.57 ,\ (0.05   , 0.86 )$ \\
& K$\ell_\infty$R & $465,\ (400, 500)$ & $0.66 ,\ (0.19   , 0.91 )$ \\
\cline{2-4}
& KMR-H           & $188,\ (147, 231)$ & $0.63 ,\ (0.10   , 0.90 )$ \\
& KMR-T           & $165,\ (130, 209)$ & $0.64 ,\ (0.18   , 0.88 )$ \\
& KQR-A           & $313,\ (253, 351)$ & $0.59 ,\ (-0.01  , 0.89 )$ \\
\cline{2-4}
& KGD             & $18.3,\ (13.6, 24.6)$ & $0.68 ,\ (0.10   , 0.89 )$ \\
& KRR             & $2.76,\ (2.64, 3.25)$ & $0.65 ,\ (0.17   , 0.89 )$ \\
\hline
\multirow{8}{*}{\makecell[l]{U.K.\\Temperature}}
& KSGD            & $7.46,\ (4.06, 11.4)$ & $0.44 ,\ (-0.02  , 0.72 )$ \\
& K$\ell_\infty$R & $410,\ (364, 421)$ & $0.42 ,\ (-0.14  , 0.69 )$ \\
\cline{2-4}
& KMR-H           & $198,\ (158, 220)$ & $0.42 ,\ (-1.12  , 0.65 )$ \\
& KMR-T           & $189,\ (144, 202)$ & $0.44 ,\ (-1.11  , 0.67 )$ \\
& KQR-A           & $338,\ (269, 356)$ & $0.37 ,\ (-1.22  , 0.69 )$ \\
\cline{2-4}
& KGD             & $14.6,\ (9.97, 19.1)$ & $0.41 ,\ (-0.18  , 0.70 )$ \\
& KRR             & $2.37,\ (2.25, 2.65)$ & $0.45 ,\ (-1.12  , 0.67 )$ \\
\hline
\end{tabular}
\label{tab:real_10_0}
\end{table}

\begin{table}
\caption{The 2.5th, 50th, and 97.5th percentiles of computation time and test $R^2$ for the different methods and data sets, \textbf{with} amplified outliers, for the Cauchy kernel. The five robust methods perform very similarly in terms of test $R^2$, while KSGD performs one to two orders of magnitude faster. The two non-robust methods, KGD/KRR, do not perform as well as the robust methods.}
\center
\begin{tabular}{|l|l|l|l|}
\hline
Data & Method & \makecell[l]{Computation Time [s]\\50\%,\ (2.5\%,\ 97.5\%)} & \makecell[l]{Test $R^2$\\50\%,\ (2.5\%,\ 97.5\%)}\\
\hline
\multirow{8}{*}{\makecell[l]{Airfoil Sound\\Pressure}}
& KSGD            & $12.6,\ (7.98, 17.6)$ & $0.42 ,\ (-0.05  , 0.74 )$ \\
& K$\ell_\infty$R & $386,\ (371, 411)$ & $0.42 ,\ (-0.79  , 0.75 )$ \\
\cline{2-4}
& KMR-H           & $151,\ (94.5, 178)$ & $0.41 ,\ (-0.85  , 0.68 )$ \\
& KMR-T           & $140,\ (105, 158)$ & $0.33 ,\ (-0.60  , 0.81 )$ \\
& KQR-A           & $282,\ (265, 310)$ & $0.42 ,\ (-1.13  , 0.75 )$ \\
\cline{2-4}
& KGD             & $14.2,\ (3.66, 19.7)$ & $0.23 ,\ (-83.22 , 0.70 )$ \\
& KRR             & $2.24,\ (2.20, 2.30)$ & $0.28 ,\ (-2.40  , 0.70 )$ \\
\hline
\multirow{8}{*}{\makecell[l]{California\\House Values}}
& KSGD            & $5.59,\ (3.24, 9.36)$ & $0.49 ,\ (-0.05  , 0.81 )$ \\
& K$\ell_\infty$R & $338,\ (322, 362)$ & $0.46 ,\ (-0.54  , 0.78 )$ \\
\cline{2-4}
& KMR-H           & $112,\ (69.8, 143)$ & $0.40 ,\ (-0.76  , 0.80 )$ \\
& KMR-T           & $107,\ (77.4, 121)$ & $0.41 ,\ (-0.59  , 0.81 )$ \\
& KQR-A           & $268,\ (251, 291)$ & $0.29 ,\ (-0.80  , 0.82 )$ \\
\cline{2-4}
& KGD             & $13.9,\ (6.24, 16.8)$ & $0.20 ,\ (-3.98  , 0.77 )$ \\
& KRR             & $2.37,\ (2.28, 2.42)$ & $0.30 ,\ (-0.21  , 0.76 )$ \\
\hline
\multirow{8}{*}{\makecell[l]{Steel Energy\\Consumption}}
& KSGD            & $22.7,\ (19.6, 34.0)$ & $0.89 ,\ (0.06   , 0.97 )$ \\
& K$\ell_\infty$R & $416,\ (402, 485)$ & $0.83 ,\ (-0.75  , 0.98 )$ \\
\cline{2-4}
& KMR-H           & $152,\ (113, 199)$ & $0.86 ,\ (-1.63  , 0.98 )$ \\
& KMR-T           & $152,\ (127, 184)$ & $0.86 ,\ (-0.02  , 0.98 )$ \\
& KQR-A           & $272,\ (252, 320)$ & $0.88 ,\ (-1.72  , 0.98 )$ \\
\cline{2-4}
& KGD             & $14.0,\ (6.36, 18.6)$ & $0.74 ,\ (-34.70 , 0.97 )$ \\
& KRR             & $2.31,\ (2.20, 2.39)$ & $0.80 ,\ (-5.36  , 0.97 )$ \\
\hline
\multirow{8}{*}{\makecell[l]{Superconductor\\Critical\\Temperature}}
& KSGD            & $27.6,\ (22.0, 34.7)$ & $0.47 ,\ (-0.09  , 0.82 )$ \\
& K$\ell_\infty$R & $469,\ (406, 502)$ & $0.47 ,\ (-0.39  , 0.84 )$ \\
\cline{2-4}
& KMR-H           & $182,\ (111, 236)$ & $0.49 ,\ (-0.39  , 0.84 )$ \\
& KMR-T           & $173,\ (114, 250)$ & $0.49 ,\ (-0.54  , 0.85 )$ \\
& KQR-A           & $329,\ (270, 358)$ & $0.48 ,\ (-0.87  , 0.83 )$ \\
\cline{2-4}
& KGD             & $15.4,\ (8.11, 19.6)$ & $0.37 ,\ (-1.24  , 0.83 )$ \\
& KRR             & $2.73,\ (2.59, 3.18)$ & $0.46 ,\ (-0.49  , 0.81 )$ \\
\hline
\multirow{8}{*}{\makecell[l]{U.K.\\Temperature}}
& KSGD            & $9.66,\ (5.27, 17.3)$ & $0.38 ,\ (-0.08  , 0.62 )$ \\
& K$\ell_\infty$R & $428,\ (373, 456)$ & $0.29 ,\ (-0.93  , 0.63 )$ \\
\cline{2-4}
& KMR-H           & $163,\ (111, 225)$ & $0.28 ,\ (-3.21  , 0.66 )$ \\
& KMR-T           & $152,\ (115, 202)$ & $0.29 ,\ (-3.33  , 0.66 )$ \\
& KQR-A           & $338,\ (273, 369)$ & $0.22 ,\ (-1.46  , 0.64 )$ \\
\cline{2-4}
& KGD             & $12.3,\ (3.28, 19.1)$ & $0.10 ,\ (-19.91 , 0.59 )$ \\
& KRR             & $2.33,\ (2.23, 2.65)$ & $0.14 ,\ (-2.69  , 0.59 )$ \\
\hline
\end{tabular}
\label{tab:real_10_1}
\end{table}

\newpage
\clearpage
\newpage

\section{Proofs}
\label{sec:proofs}

\begin{proof}[Proof of Proposition \ref{thm:equ_alpha}]~\\
Differentiating Equation \ref{eq:equ_alpha2} with respect to $\alphav$ and setting the gradient to $\nv$, we obtain

\begin{equation*}
\begin{aligned}
&\nv=\frac{\partial}{\partial \alphahv}\left(\frac12\left\|\yv-\K\alphahv\right\|_{\K^{-1}}^2+\frac\lambda2\|\alphahv\|^2_2\right)=-\yv+\K\alphahv+\lambda\alphahv=-\yv+(\K+\lambda\I)\alphahv\\
&\iff \alphahv=(\K+\lambda\I)^{-1}\yv.
\end{aligned}
\end{equation*}
\end{proof}

\begin{proof}[Proof of Equation \ref{eq:kgf_s}]~\\
With $\etahv(t)=\K\alphahv(t)-\yv$, we obtain $\frac{d\etahv(t)}{dt}=\K\frac{d\alphahv(t)}{dt}$ and Equation \ref{eq:kgd_diff_eq} can be written as
$$\frac{d\etahv(t)}{dt}=-\K\etahv(t)\iff \etahv(t)=\exp\left(-t\K\right)\etahnv.$$
Now,
$$\alphahv(0)=\bm{0}\implies\etahnv=\etahv(0)=-\yv\implies \etahv(t)=-\exp(-t\K)\yv.$$
Solving for $\alphahv(t)$, we obtain
$$\alphahv(t)=\K^{-1}\left(\I-\exp(-t\K)\right)\yv=\left(\I-\exp(-t\K)\right)\K^{-1}\yv.$$
\end{proof}

\begin{proof}[Proof of Remark \ref{rm:nest}]~\\
The differential equation in Equation \ref{eq:kgd_diff_eq} can be obtained by writing
$$\alphahv_{k+1}=\alphahv_k+\eta\cdot\left(\yv-\K\alphahv_k\right)$$
as 
$$\alphahv(t+\Delta t)=\alphahv(t)+\Delta t\cdot\left(\yv-\K\alphahv(t)\right),$$
rearranging and letting $\Delta t\to 0$, to obtain
$$\frac{d\alphahv(t)}{dt}=\lim_{\Delta t\to 0}\left(\frac{\alphahv(t+\Delta t)-\alphahv(t)}{\Delta t}\right)=\yv-\K\alphahv(t).$$
For momentum and Nesterov accelerated gradient, the update rule for $\alphahv$ in Equation \ref{eq:kgd_update} generalizes to
\begin{equation*}
\begin{aligned}
\alphahv(t+\Delta t)=&\alphahv(t)+\gamma\cdot\left(\alphahv(t)-\alphahv(t-\Delta t)\right)\\
&+\Delta t\cdot\left(\yv-\K\left(\alphahv(t)+\mathbb{I}_{\text{NAG}}\cdot\gamma\cdot\left(\alphahv(t)-\alphahv(t-\Delta t)\right)\right)\right),
\end{aligned}
\end{equation*}
where $\gamma\in[0,1)$ is the strength of the momentum and $\mathbb{I}_{\text{NAG}} \in \{0,1\}$ decides whether to use Nesterov accelerated gradient or not. Rearranging, we obtain
\begin{equation*}
\begin{aligned}
&\frac{\alphahv(t+\Delta t)-\alphahv(t)}{\Delta t}-\gamma\cdot\frac{\alphahv(t)-\alphahv(t-\Delta t)}{\Delta t}=\\
&\yv-\K\left(\alphahv(t)+\mathbb{I}_{\text{NAG}}\cdot\gamma\cdot\left(\alphahv(t)-\alphahv(t-\Delta t)\right)\right),
\end{aligned}
\end{equation*}
and, when $\Delta t \to 0$,
$$(1-\gamma)\cdot\frac{d\alphahv(t)}{d t}=\yv-\K\alphahv(t).$$
Solving the differential equations in the same way as for Equation \ref{eq:kgf_s}, we obtain
$$\alphahv(t)=\left(\I-\exp\left(-\frac t{1-\gamma}\K\right)\right)\K^{-1}\yv.$$

\end{proof}

\begin{proof}[Proof of Proposition \ref{thm:kgf_krr_diff}]~\\
By using the fact that $\K$, $\K^{-1}$, and $\DK:=\left(\left(\I+t\K\right)^{-1}-\exp(-t\K)\right)$ commute, we obtain
\begin{equation*}
\begin{aligned}
\alphahfv(t)-\alphahrv(1/t)&=\left(\I-\exp(-t\K)\right)\K^{-1}\yv-\left(\I-\left(\I+t\K\right)^{-1}\right)\K^{-1}\yv\\
&=\underbrace{\left(\left(\I+t\K\right)^{-1}-\exp(-t\K)\right)}_{=:\DK}\K^{-1}\yv=\DK\K^{-1}\yv,\\
\end{aligned}
\end{equation*}

\begin{equation*}
\begin{aligned}
\fhfv(\X,t)-\fhrv(\X,1/t)&=\K\left(\alphahfv-\alphahrv\right)=\K\DK\K^{-1}\yv\\
&=\DK\K\K^{-1}\yv=\DK\yv,\\
\fhf(\xsv,t)-\fhr(\xsv,1/t)&=\kv(\xsv)^\top\left(\alphahfv-\alphahrv\right)\\
&=\kv(\xsv)^\top\DK\K^{-1}\yv=\kv(\xsv)^\top\K^{-1}\DK\yv,\\
\end{aligned}
\end{equation*}

We denote the singular value decomposition of the symmetric matrix $\K$ by $\K=\U\Ss\U^\top$, where $\Ss$ is a diagonal matrix with diagonal elements $\Ss_{ii}=s_i$. Then $\DK=\U\D\U^\top$, where $\D$ is a diagonal matrix with entries 
$$\D_{ii}=d_i=\frac1{1+t\cdot s_i}-e^{t\cdot s_i}.$$
Since for any vector $\xv$ and any diagonal matrix $\D$,
\begin{equation}
\label{eq:move_middle}
\xv^\top\D\xv=\sum_i x_i^2\cdot d_i\leq\sum_ix_i^2\cdot\max_i d_i=\xv^\top\xv\cdot\max_i d_i,
\end{equation}
and since $\K^{-1}$ and $\DK$ are both symmetric,
\begin{equation*}
\begin{aligned}
&\|\alphahfv(t)-\alphahrv(1/t)\|_2^2=\yv^\top\K^{-1}\DK\DK\K^{-1}\yv=\yv^\top\K^{-1}\U\D^2\U^\top\K^{-1}\yv\\
&\leq\yv^\top\K^{-1}\underbrace{\U\U^\top}_{=\I}\K^{-1}\yv\cdot \max_i d_i^2=\|\K^{-1}\yv\|_2^2\cdot\max_i d_i^2,
\end{aligned}
\end{equation*}
and
\begin{equation*}
\begin{aligned}
\|\fhfv(\X,t)-\fhrv(\X,1/t)\|_2^2&=\yv^\top\DK\DK\yv=\yv^\top\U\D^2\U^\top\yv\\
&\leq\yv^\top\underbrace{\U\U^\top}_{=\I}\yv\cdot \max_i d_i^2=\|\yv\|_2^2\cdot\max_i d_i^2.
\end{aligned}
\end{equation*}

For out-of-sample predictions, calculations become different, since $\kv(\xsv)$ and $\K$ do not commute. Hence we need to take expectation over $\yv$. We first note that for $\yv=\K\alphanv+\epsv$, where $\alphanv\sim(\bm{0},\Sigmaalpha)$, and $\epsv\sim(\bm{0},\sigma_\varepsilon^2\I)$, and where $\K$ and $\Sigmaalpha$ commute,
$$\E_{\epsv,\alphanv}\left(\yv\yv^\top\right)=\K\E\left(\alphanv\alphanv^\top\right)\K+2\K\E(\alphanv)\E(\epsv)^\top+\E(\epsv\epsv^\top)=\Sigmaalpha\K^2+\sigma_\varepsilon^2\I.$$
Now,
\begin{equation*}
\begin{aligned}
&\E_{\epsv,\alphanv}\left(\left(\fhf(\xsv,t)-\fhr(\xsv,1/t)\right)^2\right)\\
&=\E_{\epsv,\alphanv}\left(\kv(\xsv)^\top\K^{-1}\DK\yv\yv^\top\DK\K^{-1}\kv(\xsv)\right)\\
&=\kv(\xsv)^\top\K^{-1}\DK\left(\Sigmaalpha\K^2+\sigma_\varepsilon^2\I\right)\DK\K^{-1}\kv(\xsv)\\
&=\kv(\xsv)^\top\K^{-1}\Sigmaalpha^{1/2}\K\U\D^2\U^\top\K\Sigmaalpha^{1/2}\K^{-1}\kv(\xsv)\\
&+ \sigma_\varepsilon^2\cdot\kv(\xsv)^\top\K^{-1}\U\D^2\U^\top\K^{-1}\kv(\xsv)\\
&\leq\left(\kv(\xsv)^\top\K^{-1}\Sigmaalpha^{1/2}\K\underbrace{\U\U^\top}_{=\I}\K\Sigmaalpha^{1/2}\K^{-1}\kv(\xsv)\right.\\
&\left.+ \sigma_\varepsilon^2\cdot\kv(\xsv)^\top\K^{-1}\underbrace{\U\U^\top}_{=\I}\K^{-1}\kv(\xsv)\right)\cdot \max_i d_i^2\\
&=\kv(\xsv)^\top\K^{-1}\left(\Sigmaalpha\K^2+\sigma_\varepsilon^2\I\right)\K^{-1}\kv(\xsv)\cdot \max_i d_i^2\\
&=\E_{\epsv,\alphanv}\left(\kv(\xsv)^\top\K^{-1}\yv\yv^\top\K^{-1}\kv(\xsv)\right)\cdot \max_i d_i^2\\
&=\E_{\epsv,\alphanv}\left((\kv(\xsv)^\top\K^{-1}\yv)^2\right)\cdot \max_i d_i^2
=\E_{\epsv,\alphanv}\left(\fhn(\xsv)^2\right)\cdot \max_i d_i^2,\\
\end{aligned}
\end{equation*}
where we have again used Equation \ref{eq:move_middle}, and the fact that $\K$, $\Sigmaalpha^{1/2}$ and $\DK$ commute.

If we repeat the calculations with $\kv(\xsv)$ replaced by $\bm{e_i}\in \R^n$, where $(\bm{e_i})_i=1$ and remaining elements equal 0, so that $\hat{\alpha}_i=\bm{e_i}^\top\alphahv$, we obtain
\begin{equation*}
\begin{aligned}
\E_{\epsv,\alphanv}\left(\left(\alphahfi(t)-\alphahri(1/t)\right)^2\right)x_i d_i^2
&=\E_{\epsv,\alphanv}\left((\bm{e_i}^\top\K^{-1}\yv)^2\right)\cdot \max_i d_i^2\\
&=\E_{\epsv,\alphanv}\left((\alphahni)^2\right)\cdot \max_i d_i^2.
\end{aligned}
\end{equation*}
What is left to do, is to calculate $\max_id_i^2$. Let $x=t\cdot s_i$. Then, for $x\geq0$ (which holds, since both $t$ and $s_i$ are non-negative), $d_i^2=\left(\frac1{1+x}-e^{-x}\right)^2$ is obviously larger than 0, and equals 0 for $x=0$ and $x=+\infty$. To find the maximum of the expression we differentiate and set the derivative to zero.

$$\frac{\partial}{\partial x}\left(\left(\frac1{1+x}-e^{-x}\right)^2\right)=2\left(e^{-x}-\frac1{(1+x)^2}\right)\left(\frac1{1+x}-e^{-x}\right).$$
For $i\in\{1,2\}$ we obtain
\begin{equation*}
\begin{aligned}
&e^{-x}-\frac1{(1+x)^i}=0 \iff -\frac{1+x}ie^{-\frac{1+x}i}=-\frac{e^{-\frac1i}}i\\
\iff& -\frac{1+x}i=W_k\left(-\frac{e^{-\frac1i}}i\right)
\iff x=-1-i\cdot W_k\left(-\frac{e^{-\frac1i}}i\right)
\end{aligned}
\end{equation*}
where $W_k$ is the $k$-th branch of the Lambert W function, for $k\in\{-1,0\}$. The only combination of $i$ and $k$ that yields an $x\neq0$ is $i=2$, $k=-1$, for which we obtain $x\approx 2.51287$ and $\left.\left(\frac1{1+x}-e^{-x}\right)^2\right|_{x=2.51287}\approx0.04166.$

\end{proof}

\begin{proof}[Proof of Proposition \ref{thm:kgf_krr_y}]~\\
In all four parts, we will use the inequality $e^x\geq 1+x$ and the equalities $\max_{\vv\neq \nv}\frac{\|\A\vv\|_2}{\|\vv\|_2}=\|\A\|_2=\smax(\A)$, where $\smax(\A)$ denotes the largest singular value. We will denote the largest and smallest singular values of $\K$ by $\smax=\smax(\K)$ and $\smin=\smin(\K).$

For part \emph{(a)}, we have
\begin{equation*}
\begin{aligned}
&\max_{\yv\neq \nv}\left(\frac{\left\|\fhfv(\X,t)-\yv\right\|_2}{\left\|\yv\right\|_2}\right)
=\max_{\yv\neq \nv}\left(\frac{\left\|\K\K^{-1}(\I-\exp(-t\K))\yv-\yv\right\|_2}{\left\|\yv\right\|_2}\right)\\
&=\max_{\yv\neq \nv}\left(\frac{\left\|\exp(-t\K))\yv\right\|_2}{\left\|\yv\right\|_2}\right)
=\left\|\exp(-t\K))\right\|_2=e^{-t\smin},
\end{aligned}
\end{equation*}
where $\smin$ denotes the smallest singular value of $\K$.

Similarly,
\begin{equation*}
\begin{aligned}
&\max_{\yv\neq \nv}\left(\frac{\left\|\fhrv(\X,1/t)-\yv\right\|_2}{\left\|\yv\right\|_2}\right)
=\max_{\yv\neq \nv}\left(\frac{\left\|\K\K^{-1}(\I-(\I+t\K)^{-1})\yv-\yv\right\|_2}{\left\|\yv\right\|_2}\right)\\
&=\max_{\yv\neq \nv}\left(\frac{\left\|(\I+t\K)^{-1}\yv\right\|_2}{\left\|\yv\right\|_2}\right)
=\left\|(\I+t\K)^{-1}\right\|_2=\frac1{1+t\smin},
\end{aligned}
\end{equation*}
where $e^{-t\smin}\leq \frac1{1+t\smin}$, since $e^x\geq 1+x$.

~\\
For part \emph{(b)},
\begin{equation*}
\begin{aligned}
&\max_{\yv\neq \nv}\left(\frac{\left\|\fhrv(\X,1/t)\right\|_2}{\left\|\yv\right\|_2}\right)
=\max_{\yv\neq \nv}\left(\frac{\left\|\K\K^{-1}(\I-(\I+t\K)^{-1})\yv\right\|_2}{\left\|\yv\right\|_2}\right)\\
&=\max_{\yv\neq \nv}\left(\frac{\left\|(\I-(\I+t\K)^{-1})\yv\right\|_2}{\left\|\yv\right\|_2}\right)
=\left\|\I-(\I+t\K)^{-1}\right\|_2=1-\frac1{1+t\smax}
\end{aligned}
\end{equation*}
and
\begin{equation*}
\begin{aligned}
&\max_{\yv\neq \nv}\left(\frac{\left\|\fhfv(\X,t)\right\|_2}{\left\|\yv\right\|_2}\right)
=\max_{\yv\neq \nv}\left(\frac{\left\|\K\K^{-1}(\I-\exp(-t\K))\yv\right\|_2}{\left\|\yv\right\|_2}\right)\\
&=\max_{\yv\neq \nv}\left(\frac{\left\|(\I-\exp(-t\K)))\yv\right\|_2}{\left\|\yv\right\|_2}\right)
=\left\|\I-\exp(-t\K))\right\|_2=1-e^{-t\smax},
\end{aligned}
\end{equation*}
where $1-\frac1{1+t\smax}\leq 1-e^{-t\smax}$, since $e^x\geq 1+x$.

~\\
For parts \emph{(c)} and \emph{(d)}, we first note that for $\yv=\K\alphanv+\epsv$, where $\alphanv\sim(\bm{0},\Sigmaalpha)$, and $\epsv\sim(\bm{0},\sigma_\varepsilon^2\I)$, and where $\K$ and $\Sigmaalpha$ commute,
$$\E_{\epsv,\alphanv}\left(\yv\yv^\top\right)=\K\E\left(\alphanv\alphanv^\top\right)\K+2\K\E(\alphanv)\E(\epsv)^\top+\E(\epsv\epsv^\top)=\Sigmaalpha\K^2+\sigma_\varepsilon^2\I,$$
and that this quantity commutes with other functions of $\K$.
Furthermore, let $\bm{e_i}\in \R^n$, denote the i-th base vector, where $(\bm{e_i})_i=1$ and remaining elements equal 0, so that $\bm{e_i}^\top\vv=v_i$.

For part \emph{(c)},
\begin{equation*}
\begin{aligned}
&\max_{\E_{\epsv,\alphanv}(y_i^2)\neq 0}\left(\frac{\E_{\epsv,\alphanv}\left(\left(\fhf(\xvi,t)-y_i\right)^2\right)}{\E_{\epsv,\alphanv}\left(y_i^2\right)}\right)\\
&=\max_{\E_{\epsv,\alphanv}(y_i^2)\neq 0}\left(\frac{\E_{\epsv,\alphanv}\left(\left(\bm{e_i}^\top\exp(-t\K)\yv\right)^2\right)}{\E_{\epsv,\alphanv}\left(\left(\bm{e_i}^\top\yv\right)^2\right)}\right)\\
&=\max_{\E_{\epsv,\alphanv}(y_i^2)\neq 0}\left(\frac{\E_{\epsv,\alphanv}\left(\bm{e_i}^\top\exp(-t\K)\yv\yv^\top\exp(-t\K)\bm{e_i}\right)}{\E_{\epsv,\alphanv}\left(\bm{e_i}^\top\yv\yv^\top\bm{e_i}\right)}\right)\\
&=\max_{\E_{\epsv,\alphanv}(y_i^2)\neq 0}\left(\frac{\bm{e_i}^\top\exp(-t\K)\left(\Sigmaalpha\K^2+\sigma_\varepsilon^2\I\right)\exp(-t\K)\bm{e_i}}{\bm{e_i}^\top\left(\Sigmaalpha\K^2+\sigma_\varepsilon^2\I\right)\bm{e_i}}\right)\\
&=\max_{\left\|\sqrt{\Sigmaalpha\K^2+\sigma_\varepsilon^2\I}\bm{e_i}\right\|_2\neq 0}\left(\frac{\left\|\exp(-t\K)\sqrt{\Sigmaalpha\K^2+\sigma_\varepsilon^2\I}\bm{e_i}\right\|_2}{\left\|\sqrt{\Sigmaalpha\K^2+\sigma_\varepsilon^2\I}\bm{e_i}\right\|_2}\right)^2
=e^{-2t\smin}
\end{aligned}
\end{equation*}
and
\begin{equation*}
\begin{aligned}
&\max_{\E_{\epsv,\alphanv}(y_i^2)\neq 0}\left(\frac{\E_{\epsv,\alphanv}\left(\left(\fhr(\xvi,1/t)-\yv\right)^2\right)}{\E_{\epsv,\alphanv}\left(y_i^2\right)}\right)\\
&=\max_{\E_{\epsv,\alphanv}(y_i^2)\neq 0}\left(\frac{\E_{\epsv,\alphanv}\left(\left(\bm{e_i}^\top(\I+t\K)^{-1}\yv\right)^2\right)}{\E_{\epsv,\alphanv}\left(\left(\bm{e_i}^\top\yv\right)^2\right)}\right)\\
&=\max_{\E_{\epsv,\alphanv}(y_i^2)\neq 0}\left(\frac{\E_{\epsv,\alphanv}\left(\bm{e_i}^\top(\I+t\K)^{-1}\yv\yv^\top(\I+t\K)^{-1}\bm{e_i}\right)}{\E_{\epsv,\alphanv}\left(\bm{e_i}^\top\yv\yv^\top\bm{e_i}\right)}\right)\\
&=\max_{\E_{\epsv,\alphanv}(y_i^2)\neq 0}\left(\frac{\bm{e_i}^\top(\I+t\K)^{-1}\left(\Sigmaalpha\K^2+\sigma_\varepsilon^2\I\right)(\I+t\K)^{-1}\bm{e_i}}{\bm{e_i}^\top\left(\Sigmaalpha\K^2+\sigma_\varepsilon^2\I\right)\bm{e_i}}\right)\\
&=\max_{\left\|\sqrt{\Sigmaalpha\K^2+\sigma_\varepsilon^2\I}\bm{e_i}\right\|_2\neq 0}\left(\frac{\left\|(\I+t\K)^{-1}\sqrt{\Sigmaalpha\K^2+\sigma_\varepsilon^2\I}\bm{e_i}\right\|_2}{\left\|\sqrt{\Sigmaalpha\K^2+\sigma_\varepsilon^2\I}\bm{e_i}\right\|_2}\right)^2
=\frac1{(1+t\smin)^2},
\end{aligned}
\end{equation*}
where $e^{-2t\smin}\leq \frac1{(1+t\smin)^2}$, since $e^x\geq 1+x$.

For part \emph{(d)},
\begin{equation*}
\begin{aligned}
&\max_{\E_{\epsv,\alphanv}(y_i^2)\neq 0}\left(\frac{\E_{\epsv,\alphanv}\left(\left(\fhr(\xvi,1/t)\right)^2\right)}{\E_{\epsv,\alphanv}\left(y_i^2\right)}\right)\\
&=\max_{\E_{\epsv,\alphanv}(y_i^2)\neq 0}\left(\frac{\E_{\epsv,\alphanv}\left(\left(\bm{e_i}^\top(\I-(\I+t\K)^{-1})\yv\right)^2\right)}{\E_{\epsv,\alphanv}\left(\left(\bm{e_i}^\top\yv\right)^2\right)}\right)\\
&=\max_{\E_{\epsv,\alphanv}(y_i^2)\neq 0}\left(\frac{\E_{\epsv,\alphanv}\left(\bm{e_i}^\top((\I-\I+t\K)^{-1})\yv\yv^\top(\I-(\I+t\K)^{-1})\bm{e_i}\right)}{\E_{\epsv,\alphanv}\left(\bm{e_i}^\top\yv\yv^\top\bm{e_i}\right)}\right)\\
&=\max_{\E_{\epsv,\alphanv}(y_i^2)\neq 0}\left(\frac{\bm{e_i}^\top(\I-(\I+t\K)^{-1})\left(\Sigmaalpha\K^2+\sigma_\varepsilon^2\I\right)(\I-(\I+t\K)^{-1})\bm{e_i}}{\bm{e_i}^\top\left(\Sigmaalpha\K^2+\sigma_\varepsilon^2\I\right)\bm{e_i}}\right)\\
&=\max_{\left\|\sqrt{\Sigmaalpha\K^2+\sigma_\varepsilon^2\I}\bm{e_i}\right\|_2\neq 0}\left(\frac{\left\|(\I-(\I+t\K)^{-1})\sqrt{\Sigmaalpha\K^2+\sigma_\varepsilon^2\I}\bm{e_i}\right\|_2}{\left\|\sqrt{\Sigmaalpha\K^2+\sigma_\varepsilon^2\I}\bm{e_i}\right\|_2}\right)^2\\
&=\left(1-\frac1{(1+t\smax)}\right)^2,
\end{aligned}
\end{equation*}
and
\begin{equation*}
\begin{aligned}
&\max_{\E_{\epsv,\alphanv}(y_i^2)\neq 0}\left(\frac{\E_{\epsv,\alphanv}\left(\left(\fhf(\xvi,t)-\yv\right)^2\right)}{\E_{\epsv,\alphanv}\left(y_i^2\right)}\right)\\
&=\max_{\E_{\epsv,\alphanv}(y_i^2)\neq 0}\left(\frac{\E_{\epsv,\alphanv}\left(\left(\bm{e_i}^\top(\I-\exp(-t\K))\yv\right)^2\right)}{\E_{\epsv,\alphanv}\left(\left(\bm{e_i}^\top\yv\right)^2\right)}\right)\\
&=\max_{\E_{\epsv,\alphanv}(y_i^2)\neq 0}\left(\frac{\E_{\epsv,\alphanv}\left(\bm{e_i}^\top(\I-\exp(-t\K))\yv\yv^\top(\I-\exp(-t\K))\bm{e_i}\right)}{\E_{\epsv,\alphanv}\left(\bm{e_i}^\top\yv\yv^\top\bm{e_i}\right)}\right)\\
&=\max_{\E_{\epsv,\alphanv}(y_i^2)\neq 0}\left(\frac{\bm{e_i}^\top(\I-\exp(-t\K))\left(\Sigmaalpha\K^2+\sigma_\varepsilon^2\I\right)(\I-\exp(-t\K))\bm{e_i}}{\bm{e_i}^\top\left(\Sigmaalpha\K^2+\sigma_\varepsilon^2\I\right)\bm{e_i}}\right)\\
&=\max_{\left\|\sqrt{\Sigmaalpha\K^2+\sigma_\varepsilon^2\I}\bm{e_i}\right\|_2\neq 0}\left(\frac{\left\|(\I-\exp(-t\K))\sqrt{\Sigmaalpha\K^2+\sigma_\varepsilon^2\I}\bm{e_i}\right\|_2}{\left\|\sqrt{\Sigmaalpha\K^2+\sigma_\varepsilon^2\I}\bm{e_i}\right\|_2}\right)^2\\
&=(1-e^{-t\smax})^2,
\end{aligned}
\end{equation*}
where $\left(1-\frac1{1+t\smax}\right)^2\leq \left(1-e^{-t\smax}\right)^2$, since $e^x\geq 1+x$.

\end{proof}

\begin{proof}[Proof of Proposition \ref{thm:kgf_krr_risk}]~\\
Unless otherwise stated, all expectations and covariances are with respect to the random variable $\epsv$. We let $\K=\U\Ss\U^\top$ denote the eigenvalue decomposition of $\K$.

Before starting the calculations, we show some intermediary results that will be needed:

\begin{equation*}
\begin{aligned}
&\E(\alphahfv(t))-\alphanv=\left(\I-\exp\left(-t\K\right)\right)\K^{-1}\E(\yv)-\alphanv\\
&=\left(\I-\exp\left(-t\K\right)\right)\alphanv-\alphanv=-\exp\left(-t\K\right)\alphanv.\\
&\Cov(\alphahfv(t))=\E\left(\alphahfv(t)\left(\alphahfv(t)\right)^\top\right)-\E(\alphahfv(t))\E(\alphahfv(t))^\top\\
&=\left(\I-\exp\left(-t\K\right)\right)\alphanv\alphanv^\top\left(\I-\exp\left(-t\K\right)\right) +\sigma_\varepsilon^2\K^{-2}\left(\I-\exp\left(-t\K\right)\right)^2\\
&-\left(\I-\exp\left(-t\K\right)\right)\alphanv\alphanv^\top\left(\I-\exp\left(-t\K\right)\right)\\
&=\sigma_\varepsilon^2\K^{-2}\left(\I-\exp\left(-t\K\right)\right)^2.\\
&\E(\alphahrv(1/t))-\alphanv\\
&=\left(\K+1/t\cdot\I\right)^{-1}\E(\yv)-\alphanv=\left(\K+1/t\cdot\I\right)^{-1}\K\alphanv-\alphanv\\
&=\left(\left(\K+1/t\cdot\I\right)^{-1}\K-\I\right)\alphanv=\left(\K+1/t\cdot\I\right)^{-1}\left(\K-\K-1/t\cdot\I\right)\alphanv\\
&=-1/t\cdot\left(\K+1/t\cdot\I\right)^{-1}\alphanv.\\
&\Cov(\alphahrv(1/t))=\E\left(\alphahrv(1/t)\left(\alphahrv(1/t)\right)^\top\right)\\
&-\E(\alphahrv(1/t))\E(\alphahrv(1/t))^\top\\
&=\left(\K+1/t\cdot\I\right)^{-1}\left(\K\alphanv\alphanv^\top\K+\sigma_\varepsilon^2\I\right)\left(\K+1/t\cdot\I\right)^{-1}\\
&-\left(\K+1/t\cdot\I\right)^{-1}\K\alphanv\alphanv^\top\K\left(\K+1/t\cdot\I\right)^{-1}\\
&=\sigma_\varepsilon^2\left(\K+1/t\cdot\I\right)^{-2}.
\end{aligned}
\end{equation*}

\begin{equation*}
\begin{aligned}
&\E\left(\fhf(\xsv,t)\right)-f_0(\xsv)=\kv(\xsv)^\top\left(\E(\alphahfv(t))-\alphanv\right)=\\
&-\kv(\xsv)^\top\exp\left(-t\K\right)\alphanv.\\
&\E\left(\fhr(\xsv,1/t)\right)-f_0(\xsv)=\kv(\xsv)^\top\left(\E(\alphahrv(1/t))-\alphanv\right)=\\
&-1/t\cdot\kv(\xsv)^\top\left(\K+1/t\cdot\I\right)^{-1}\alphanv.\\
&\Var(\fhf(\xsv,t))=\E\left(\fhf(\xsv,t)^2\right)-\E\left(\fhf(\xsv,t)\right)^2\\
&=\kv(\xsv)^\top\left(\I-\exp\left(-t\K\right)\right)\alphanv\alphanv^\top\left(\I-\exp\left(-t\K\right)\right)\kv(\xsv)\\
&+\sigma_\varepsilon^2\kv(\xsv)^\top\K^{-2}\left(\I-\exp\left(-t\K\right)\right)^2\kv(\xsv)\\
&-\kv(\xsv)^\top\left(\I-\exp\left(-t\K\right)\right)\alphanv\alphanv^\top\left(\I-\exp\left(-t\K\right)\right)\kv(\xsv)\\
&=\sigma_\varepsilon^2\kv(\xsv)^\top\K^{-2}\left(\I-\exp\left(-t\K\right)\right)^2\kv(\xsv).\\
&\Var(\fhr(\xsv,1/t))=\E\left(\fhr(\xsv,1/t)^2\right)-\E\left(\fhr(\xsv,1/t)\right)^2\\
&=\kv(\xsv)^\top\left(\K+1/t\cdot\I\right)^{-1}\left(\K\alphanv\alphanv^\top\K+\sigma_\varepsilon^2\I\right)\left(\K+1/t\cdot\I\right)^{-1}\kv(\xsv)\\
&-\kv(\xsv)^\top\left(\K+1/t\cdot\I\right)^{-1}\K\alphanv\alphanv^\top\K\left(\K+1/t\cdot\I\right)^{-1}\kv(\xsv)\\
&=\sigma_\varepsilon^2\kv(\xsv)^\top\left(\K+1/t\cdot\I\right)^{-2}\kv(\xsv).
\end{aligned}
\end{equation*}

According to the bias covariance decomposition, we can write the risk of an estimator $\bm{\hat{\theta}}$, estimating $\bm{\theta_0}$, as
\begin{equation*}
\label{eq:risk}
\text{Risk}(\thetahv;\thetanv)=\E\left(\|\thetahv-\thetanv\|_2^2\right)=\|\E(\thetahv)-\thetanv\|_2^2+\Tr(\Cov(\thetahv)).
\end{equation*}
Thus,
\begin{equation*}
\begin{aligned}
&\text{Risk}\left(\alphahfv(t);\alphanv\right)=\|\E(\alphahfv(t))-\alphanv\|_2^2+\Tr(\Cov(\alphahfv))\\
&=\|-\exp(-t\K)\alphanv\|_2^2+\Tr\left(\sigma_\varepsilon^2\K^{-2}\left(\I-\exp(-t\K)\right)^2\right)\\
&=\alphanv^\top\U e^{-t\Ss}\U^\top\U e^{-t\Ss}\U^\top\alphanv +\sigma_\varepsilon^2\Tr\left(\U\left(\frac{1-e^{-t\Ss}}{\Ss}\right)^2\U^\top\right)\\
&=\alphanv^\top\U e^{-2t\Ss}\U^\top\alphanv +\sigma_\varepsilon^2\Tr\left(\left(\frac{1-e^{-t\Ss}}{\Ss}\right)^2\right)\\
&=\sum_{i=1}^n((\U_{:,i})^\top\alphanv)^2 e^{-2ts_i}+\sigma_\varepsilon^2\sum_{i=1}^n\left(\frac{1-e^{-ts_i}}{s_i}\right)^2,
\end{aligned}
\end{equation*}
where in the third equality, we have used the cyclic property of the trace.

Equivalently, for the risk of $\alphahrv$, for $\lambda=1/t$,
\begin{equation*}
\begin{aligned}
&\text{Risk}\left(\alphahrv(1/t);\alphanv\right)=\|\E(\alphahrv(1/t))-\alphanv\|_2^2+\Tr(\Cov(\alphahrv))\\
&=\|-1/t\cdot\left(\K+1/t\cdot\I\right)^{-1}\alphanv\|_2^2+\Tr\left(\sigma_\varepsilon^2\left(\K+1/t\cdot\I\right)^{-2}\right)\\
&=\alphanv^\top\U \frac{1/t}{\Ss+1/t}\U^\top\U \frac{1/t}{\Ss+1/t}\U^\top\alphanv +\sigma_\varepsilon^2\Tr\left(\U\left(\frac1{\Ss+1/t}\right)^2\U^\top\right)\\
&=\alphanv^\top\U \frac1{(t\cdot\Ss+1)^2}\U^\top\alphanv +\sigma_\varepsilon^2\Tr\left(\frac{t^2}{(t\cdot\Ss+1)^2}\right)\\
&=\sum_{i=1}^n((\U_{:,i})^\top\alphanv)^2 \frac{1}{(ts_i+1)^2}+\sigma_\varepsilon^2\sum_{i=1}^n\frac{t^2}{(ts_i+1)^2}.
\end{aligned}
\end{equation*}
For $x\geq 0$, the following two inequalities are shown by \cite{ali2019continuous}: 
$$e^{-x}\leq1/(x+1)\text{ and }1-e^{-x}\leq1.2985\cdot x/(x+1).$$
Using these, we obtain
\begin{equation*}
\begin{aligned}
\text{Risk}\left(\alphahfv(t);\alphanv\right)=&\sum_{i=1}^n\left(((\U_{:,i})^\top\alphanv)^2 e^{-2ts_i}+\sigma_\varepsilon^2\left(\frac{1-e^{-ts_i}}{s_i}\right)^2\right)\\
\leq&\sum_{i=1}^n\left(((\U_{:,i})^\top\alphanv)^2 \frac1{(ts_i+1)^2}+\sigma_\varepsilon^2\left(1.2985\frac {ts_i}{ts_i+1}\right)^2\frac1{s_i^2}\right)\\
=&\sum_{i=1}^n\left(((\U_{:,i})^\top\alphanv)^2 \frac1{(ts_i+1)^2}+\sigma_\varepsilon^21.2985^2\frac {t^2}{(ts_i+1)^2}\right)\\
\leq&1.6862\sum_{i=1}^n\left(((\U_{:,i})^\top\alphanv)^2 \frac1{(ts_i+1)^2}+\sigma_\varepsilon^2\frac {t^2}{(ts_i+1)^2}\right)\\
=&1.6862\cdot\text{Risk}\left(\alphahrv(1/t);\alphanv\right),
\end{aligned}
\end{equation*}
which proves part \emph{(a)}.

The calculations for part \emph{(b)} are basically identical to those for part \emph{(a)}. We obtain
\begin{equation*}
\begin{aligned}
&\text{Risk}\left(\fhfv(\X,t);\fnv(\X)\right)=\sum_{i=1}^n((\U_{:,i})^\top\alphanv)^2 s_i^2e^{-2ts_i}+\sigma_\varepsilon^2\sum_{i=1}^n\left(1-e^{-ts_i}\right)^2,\\
&\text{Risk}\left(\fhrv(\X,1/t);\fnv(\X)\right)=\sum_{i=1}^n((\U_{:,i})^\top\alphanv)^2 \frac{s_i^2}{(ts_i+1)^2}+\sigma_\varepsilon^2\sum_{i=1}^n\frac{t^2s_i^2}{(ts_i+1)^2},
\end{aligned}
\end{equation*}
and thus
\begin{equation*}
\begin{aligned}
&\text{Risk}\left(\fhfv(\X,t);\fnv(\X)\right)=\sum_{i=1}^n\left(((\U_{:,i})^\top\alphanv)^2 s_i^2e^{-2ts_i}+\sigma_\varepsilon^2\left(1-e^{-ts_i}\right)^2\right)\\
\leq&\sum_{i=1}^n\left(((\U_{:,i})^\top\alphanv)^2 \frac{s_i^2}{(ts_i+1)^2}+\sigma_\varepsilon^2\left(1.2985\frac {ts_i}{ts_i+1}\right)^2\right)\\
\leq&1.2985^2\sum_{i=1}^n\left(((\U_{:,i})^\top\alphanv)^2 \frac{s_i^2}{(ts_i+1)^2}+\sigma_\varepsilon^2\frac {t^2s_i^2}{(ts_i+1)^2}\right)\\
=&1.6862\cdot\text{Risk}\left(\fhrv(\X,1/t);\fnv(\X)\right),
\end{aligned}
\end{equation*}
which proves part \emph{(b)}.

For the out-of-sample prediction risk, since $\kv(\xsv)$ and $\K$ do not commute, calculations become slightly different. We now obtain
\begin{equation*}
\begin{aligned}
&\text{Risk}\left(\fhf(\xsv,t);f_0(\xsv,t)\right)\\
=&\left(\E(\fhf(\xsv,t))-f_0(\xsv,t)\right)^2+\Var(\fhf(\xsv,t))\\
=&\left(-\kv(\xsv)^\top\exp(-t\K)\alphanv\right)^2+\sigma_\varepsilon^2\kv(\xsv)^\top\K^{-2}\left(\I-\exp(-t\K)\right)^2\kv(\xsv)\\
=&\alphanv^\top\exp(-t\K)\kv(\xsv)\kv(\xsv)^\top\exp(-t\K)\alphanv\\
&+\sigma_\varepsilon^2\Tr\left(\K^{-2}\left(\I-\exp(-t\K)\right)^2\kv(\xsv)\kv(\xsv)^\top\right)\\
=&\Tr\left(\exp(-t\K)\alphanv\alphanv^\top\exp(-t\K)\kv(\xsv)\kv(\xsv)^\top\right)\\
&+\sigma_\varepsilon^2\Tr\left(\K^{-2}\left(\I-\exp(-t\K)\right)^2\kv(\xsv)\kv(\xsv)^\top\right),\\
\end{aligned}
\end{equation*}
where we have used the fact that the trace of a scalar is the scalar itself, and the cyclic property of the trace.

Analogously, for the risk of $\fhr(\xsv)$,
\begin{equation*}
\begin{aligned}
&\text{Risk}\left(\fhr(\xsv,1/t);f_0(\xsv)\right)\\
=&\left(\E(\fhr(\xsv,1/t))-f_0(\xsv)\right)^2+\Var(\fhr(\xsv,1/t))\\
=&\left(-1/t\cdot\kv(\xsv)^\top\left(\K+1/t\cdot\I\right)^{-1}\alphanv\right)^2+\sigma_\varepsilon^2\kv(\xsv)^\top\left(\K+1/t\cdot\I\right)^{-2}\kv(\xsv)\\
=&1/t^2\cdot\alphanv^\top\left(\K+1/t\cdot\I\right)^{-1}\kv(\xsv)\kv(\xsv)^\top\left(\K+1/t\cdot\I\right)^{-1}\alphanv\\
&+\sigma_\varepsilon^2\Tr\left(\left(\K+1/t\cdot\I\right)^{-2}\kv(\xsv)\kv(\xsv)^\top\right)\\
=&\Tr\left(1/t^2\cdot\left(\K+1/t\cdot\I\right)^{-1}\alphanv\alphanv^\top\left(\K+1/t\cdot\I\right)^{-1}\kv(\xsv)\kv(\xsv)^\top\right)\\
&+\sigma_\varepsilon^2\Tr\left(\left(\K+1/t\cdot\I\right)^{-2}\kv(\xsv)\kv(\xsv)^\top\right).
\end{aligned}
\end{equation*}

Taking expectation over $\alphanv$, for $\alphanv\sim(\bm{0},\Sigmaalpha)$, we obtain
\begin{equation*}
\begin{aligned}
\E_{\alphanv}&\left(\text{Risk}\left(\fhf(\xsv,t);f_0(\xsv)\right)\right)\\
&=\Tr\left(\left(\Sigmaalpha\exp(-t\K)^2+\sigma_\varepsilon^2\K^{-2}\left(\I-\exp(-t\K)\right)^2\right)\kv(\xsv)\kv(\xsv)^\top\right),\\
\E_{\alphanv}&\left(\text{Risk}\left(\fhr(\xsv,1/t);f_0(\xsv)\right)\right)\\
&=\Tr\left(\left(\Sigmaalpha1/t^2\cdot\left(\K+1/t\cdot\I\right)^{-2}+\sigma_\varepsilon^2\left(\K+1/t\cdot\I\right)^{-2}\right)\kv(\xsv)\kv(\xsv)^\top\right).
\end{aligned}
\end{equation*}

Finally comparing the risks, we obtain
\begin{equation*}
\begin{aligned}
\E_{\alphanv}&\left(\text{Risk}(\fhf(\xsv,t);f_0(\xsv))\right)\\
=&\Tr\left(\left(\Sigmaalpha\exp(-t\K)^2+\sigma_\varepsilon^2t^2(t\K)^{-2}\left(\I-\exp(-t\K)\right)^2\right)\kv(\xsv)\kv(\xsv)^\top\right)\\
\leq& \Tr\left(\left(\Sigmaalpha\left(\I+t\K\right)^{-2}+\sigma_\varepsilon^2t^2\left(1.2985 \left(\I+t\K\right)^{-1}\right)^2\right)\kv(\xsv)\kv(\xsv)^\top\right)\\
=& \Tr\left(\left(\Sigmaalpha\cdot1/t^2\cdot\left(1/t\cdot\I+\K\right)^{-2}\right.\right.\\
&\left.\left.+\sigma_\varepsilon^2 \cdot 1.2985^2 \left(1/t\cdot\I+\K\right)^{-2}\right)\kv(\xsv)\kv(\xsv)^\top\right)\\
\leq& 1.6862\cdot\Tr\left(\left(\Sigmaalpha\cdot1/t^2\cdot\left(1/t\cdot\I+\K\right)^{-2}\right.\right.\\
&\left.\left.+\sigma_\varepsilon^2 \left(1/t\cdot\I+\K\right)^{-2}\right)\kv(\xsv)\kv(\xsv)^\top\right)\\
=&1.6862\cdot\E_{\alphanv}\left(\text{Risk}(\fhr(\xsv,1/t);f_0(\xsv))\right),\\
\end{aligned}
\end{equation*}
which proves part \emph{(c)}.

\end{proof}

\begin{proof}[Proof of Proposition \ref{thm:sgn_inf}]~\\
The proofs of the two parts are very similar, differing only in the details.

For part \emph{(a)}, when $\X^\top\X$ is a diagonal matrix with elements $\{s_{ii}\}_{i=1}^p$,
$$\left\|\yv-\X\betav\right\|_2^2=\yv^\top\yv-2\yv^\top\X\betav+\betav^\top\X^\top\X\betav=\sum_{i=1}^p\left(y_i^2-2(\X^\top\yv)_i\beta_i+s_{ii}\beta_i^2\right).$$
The constraint $\|\betav\|_\infty=\max_i|\beta_i| \leq c$ is equivalent to $|\beta_i| \leq c$ for $i=1,2,\dots p$.
Thus, $$\left\|\yv-\X\betav\right\|_2^2 \text{ s.t. } \|\betav\|_\infty\leq c \iff\sum_{i=1}^p\left(y_i^2-2(\X^\top\yv)_i\beta_i+s_{ii}\beta_i^2\right) \text{ s.t. } |\beta_i|\leq c\ \forall i,$$
which decomposes element-wise. In the absence of the constraint,\\ $\hat{\beta}_i=\left((\X^\top\X)^{-1}\X^\top\yv\right)_i=\left(\X^\top\yv\right)_i/s_{ii}$.

Assume $\left(\X^\top\yv\right)_i/s_{ii}\geq 0$. Then, the optimal value for $\beta_i$ is $\left(\X^\top\yv\right)_i/s_{ii}$, unless $\left(\X^\top\yv\right)_i/s_{ii}>c$, then, due to convexity, the optimal value is $c$, i.e.\ $\hat{\beta}_i=\min\left(\left(\X^\top\yv\right)_i/s_{ii},c\right)$. Accounting also for the case of $\left(\X^\top\yv\right)_i/s_{ii}< 0$, we obtain
$$\hat{\beta}_i(c)=\sgn\left(\left(\X^\top\yv\right)_i/s_{ii}\right)\cdot\min\left(\left|\left(\X^\top\yv\right)_i/s_{ii}\right|,c\right).$$

The sign gradient flow solution is calculated as follows:
\begin{equation*}
\begin{aligned}
&\sgn\left(-\frac{\partial}{\partial\betahv(t)}\left(\left\|\yv-\X\betahv(t)\right\|_2^2\right)\right)=\sgn\left(\X^\top\yv-\X^\top\X\betahv(t)\right)\\
&=\sgn\left(\begin{bmatrix}\vrule\\(\X^\top\yv)_i-s_{ii}\betah_i(t)\\ \vrule\end{bmatrix}\right)
=\begin{bmatrix}\vrule\\\sgn\left((\X^\top\yv)_i-s_{ii}\betah_i(t)\right)\\ \vrule\end{bmatrix}.
\end{aligned}
\end{equation*}
Since element $i$ in the vector only depends on $\betahv$ through $\betah_i$, 
\begin{equation*}
\begin{aligned}
\frac{\partial\betahv(t)}{\partial t}&=\sgn\left(-\frac{\partial}{\partial\betahv(t)}\left(\left\|\yv-\X\betahv(t)\right\|_2^2\right)\right)\\
\iff\frac{\partial\betah_i(t)}{\partial t}&= \sgn((\X^\top \yv)_i-s_{ii}\betah_i(t))\\
&=
\begin{cases}
1&\text{if }(\X^\top\yv)_i>s_{ii}\betah(t)\\
-1&\text{if }(\X^\top\yv)_i<s_{ii}\betah(t)\\
0&\text{if }(\X^\top\yv)_i=s_{ii}\betah(t)\\
\end{cases},
\ i=1,2,\dots p.
\end{aligned}
\end{equation*}
Thus, with $\betah_i(0)=0$, $\betah_i(t)=\pm t$, depending on the sign of $(\X^\top\yv)_i/s_{ii}$. Once $\betah_i=(\X^\top\yv)_i/s_{ii}$ (and $\sgn(s_{ii}\betah_i-(\X^\top\yv)_i)=0$), then $\betah_i=(\X^\top\yv)_i/s_{ii}$. That is,
$$\betah_i(t)=\sgn\left((\X^\top\yv)_i/s_{ii}\right)\cdot\min\left(t,\left|(\X^\top\yv)_i/s_{ii}\right|\right).$$

For part \emph{(b)}, when $\K$ is a diagonal matrix with elements $\{k_{ii}\}_{i=1}^n$,
$$\left\|\yv-\K\alphav\right\|_{\K^{-1}}^2=\left\|\left(\sqrt{\K}\right)^{-1}\yv-\sqrt{\K}\alphav\right\|_2^2=\sum_{i=1}^n\left(\frac{y_i}{\sqrt{k_{ii}}}-\sqrt{k_{ii}}\alpha_i\right)^2.$$
The constraint $\|\alphav\|_\infty=\max_i|\alpha_i| \leq c$ is equivalent to $|\alpha_i| \leq c$, for $i=1,2,\dots n$.
Thus, $$\left\|\yv-\K\alphav\right\|_{\K^{-1}}^2 \text{ s.t. } \|\alphav\|_\infty\leq c \iff\sum_{i=1}^n\left(\frac{y_i}{\sqrt{k_{ii}}}-\sqrt{k_{ii}}\alpha_i\right)^2 \text{ s.t. } |\alpha_i|\leq c\ \forall i,$$
which decomposes element-wise. Assume $y_i/k_{ii}\geq 0$. Then, the optimal value for $\alpha_i$ is $y_i/k_{ii}$, unless $y_i/k_{ii}>c$, then, due to convexity, the optimal value is $c$, i.e.\ $\hat{\alpha}_i=\min(y_i/k_{ii},c)$. Accounting also for the case of $y_i/k_{ii}< 0$, we obtain
$$\hat{\alpha}_i(c)=\sgn(y_i/k_{ii})\cdot\min(|y_i/k_{ii}|,c).$$

The sign gradient flow solution is calculated as follows:
\begin{equation*}
\begin{aligned}
&\sgn\left(-\frac{\partial}{\partial\alphahv(t)}\left(\left\|\yv-\K\alphahv(t)\right\|_{\K^{-1}}^2\right)\right)=\sgn\left(\yv-\K\alphahv(t)\right)\\
&=\sgn\left(\begin{bmatrix}\vrule\\y_i-k_{ii}\alphah_i(t)\\ \vrule\end{bmatrix}\right)=\begin{bmatrix}\vrule\\\sgn(y_i-k_{ii}\alphah_i(t))\\ \vrule\end{bmatrix}.
\end{aligned}
\end{equation*}
Since element $i$ in the vector only depends on $\alphahv$ through $\alphah_i$, 
\begin{equation*}
\begin{aligned}
&\frac{\partial\alphahv(t)}{\partial t}=\sgn\left(-\frac{\partial}{\partial\alphahv(t)}\left(\left\|\yv-\K\alphahv(t)\right\|_{\K^{-1}}^2\right)\right)\\
\iff&\frac{\partial\alphah_i(t)}{\partial t}= \sgn(y_i-k_{ii}\alphah_i(t))
=\begin{cases}
1&\text{if }y_i>k_{ii}\alphah_i(t)\\
-1&\text{if }y_i<k_{ii}\alphah_i(t)\\
0&\text{if }y_i=k_{ii}\alphah_i(t)\\
\end{cases},
\ i=1,2,\dots n.
\end{aligned}
\end{equation*}
Thus, with $\alphah_i(0)=0$, $\alphah_i(t)=\pm t$, depending on the sign of $y_i/k_{ii}$. Once $\alphah_i=y_i/k_{ii}$ (and $\sgn(y_i-k_{ii}\alphah_i)=0$), then $\alphah_i=y_i/k_{ii}$. That is,
$$\alphah_i(t)=\sgn(y_i/k_{ii})\cdot\min(t,|y_i/k_{ii}|).$$
\end{proof}

\begin{proof}[Proof of Proposition \ref{thm:other_objs}]~\\
We first calculate the gradients of $\|\yv-\Ph\betav\|_2^2$, $\|\yv-\Ph\betav\|_1$ and $\|\yv-\Ph\betav\|_\infty$, respectively:
$$\frac{\partial}{\partial \betav}\left(\left\|\yv-\Ph\betav\right\|_2^2\right)=\frac{\partial}{\partial \betav}\left((\yv-\Ph\betav)^\top(\yv-\Ph\betav)\right)=-\Ph^\top(\yv-\Ph\betav).$$
Since $\|\vv\|_1=\sgn(\vv)^\top\vv$,
$$\frac{\partial}{\partial \betav}\left(\left\|\yv-\Ph\betav\right\|_1\right)=\frac{\partial}{\partial \betav}\left(\sgn(\yv-\Ph\betav)^\top(\yv-\Ph\betav)\right)=-\Ph^\top\sgn(\yv-\Ph\betav).$$
Let $\mathbb{I}_\infty(\vv)$ be a vector denoting the sign of the largest (absolute) value in $\vv$, such that
$$\mathbb{I}_\infty(\vv)_d=
\begin{cases}
\sgn(v_d)&\text{ if } d=\argmax_{d'}|v_{d'}|\\
0&\text{ else.}
\end{cases}$$
Then $\|\vv\|_\infty=\mathbb{I}_\infty(\vv)^\top\vv$, and 
$$\frac{\partial}{\partial \betav}\left(\left\|\yv-\Ph\betav\right\|_\infty\right)=\frac{\partial}{\partial \betav}\left(\mathbb{I}_\infty(\yv-\Ph\betav)^\top(\yv-\Ph\betav)\right)=-\Ph^\top\mathbb{I}_\infty(\yv-\Ph\betav).$$

The three update rules for gradient descent in $\betav$ are thus

\begin{equation}
\label{eq:eq_gd_beta}
\begin{aligned}
\betahv_{k+1}&=\betahv_k+\eta\cdot\Ph^\top\left(\yv-\Ph\betahv_k\right), && \text{ for } \|\yv-\Ph\betav\|_2^2\\
\betahv_{k+1}&=\betahv_k+\eta\cdot\Ph^\top\sgn\left(\yv-\Ph\betahv_k\right), && \text{ for } \|\yv-\Ph\betav\|_1\\
\betahv_{k+1}&=\betahv_k+\eta\cdot\Ph^\top\mathbb{I}_\infty\left(\yv-\Ph\betahv_k\right), && \text{ for } \|\yv-\Ph\betav\|_\infty.
\end{aligned}
\end{equation}

Since the gradient of $\|\yv-\K\alphav\|_{\K^{-1}}^2$ is $-(\yv-\K\alphav)$, the update rules for gradient descent, sign gradient descent and coordinate descent are, respectively
\begin{equation}
\label{eq:eq_gd_alpha}
\begin{aligned}
\alphahv_{k+1}&=\alphahv_k+\eta\cdot\left(\yv-\K\alphahv_k\right)\\
\alphahv_{k+1}&=\alphahv_k+\eta\cdot\sgn\left(\yv-\K\alphahv_k\right)\\
\alphahv_{k+1}&=\alphahv_k+\eta\cdot\mathbb{I}_\infty\left(\yv-\K\alphahv_k\right).\\
\end{aligned}
\end{equation}
By multiplying each term in Equation \ref{eq:eq_gd_alpha} by $\Ph^\top$, and using $\betav=\Ph^\top\alphav$ and $\K=\Ph\Ph^\top$, and thus $\K\alphav=\Ph\betav$, we obtain Equation \ref{eq:eq_gd_beta}.
\end{proof}

\begin{proof}[Proof of Proposition \ref{thm:conv}]~\\
For KGD, let $L_2(\alphahv):=\|\yv-\K\alphahv\|_2^2$. Then $L_2'(\alphahv)=-\K(\yv-\K\alphav)$, $L_2''(\alphahv)=\K^2$ and higher order derivatives are zero. With $\dalphahv=\eta\cdot(\yv-\K\alphahv)$, according to Taylor's theorem
\begin{equation*}
\begin{aligned}
&L_2(\alphahv_{k+1})-L_2(\alphahv_k)
=L_2(\alphahv_k+\dalphahv_k)-L_2(\alphahv_k)\\
&=L_2(\alphahv_k)+L_2'(\alphahv_k)^\top\dalphahv+\frac12\dalphahv^\top L_2''(\alphahv_k)\dalphahv-L_2(\alphahv_k)\\
&=-\eta\cdot(\yv-\K\alphav)^\top\K(\yv-\K\alphahv)+\frac{\eta^2}2(\yv-\K\alphahv_k)^\top\K^2(\yv-\K\alphahv)\\
&\leq-\eta\cdot(\yv-\K\alphav)^\top\K(\yv-\K\alphahv)+\frac{\eta^2}2\smax(\K)(\yv-\K\alphahv_k)^\top\K(\yv-\K\alphahv).
\end{aligned}
\end{equation*}
Factoring out $\eta\cdot(\yv-\K\alphav)^\top\K(\yv-\K\alphahv)$, and using that $(\yv-\K\alphav)^\top\K(\yv-\K\alphahv)\geq0$, since $\K$ is positive semi-definite, we see that
$L_2(\alphahv_{k+1})-L_2(\alphahv_k)\leq 0$ if $\eta\leq\frac2{\smax(\K)}$.

For KSGD, $L_1(\alphahv):=\|\yv-\K\alphahv\|_1=\sgn(\yv-\K\alphahv)^\top(\yv-\K\alphahv)$. Then $L_1'(\alphahv)=-\K\cdot\sgn(\yv-\K\alphav)$ and higher order derivatives are zero. With $\dalphahv=\eta\cdot\sgn(\yv-\K\alphahv)$, according to Taylor's theorem, for $\eta\leq \min_{1\leq i\leq n}\left(\left|\yv-\K\alphahv_k\right|_i\right)$
\begin{equation*}
\begin{aligned}
L_1(\alphahv_{k+1})-L_1(\alphahv_k)
&=L_1(\alphahv_k+\dalphahv_k)-L_1(\alphahv_k)
=L_1(\alphahv_k)+L'(\alphahv_k)^\top\dalphahv\\
&=-\eta\cdot\sgn(\yv-\K\alphav)^\top\K\cdot\sgn(\yv-\K\alphahv)\leq0.\\
\end{aligned}
\end{equation*}

For KCD, let $\mathbb{I}_\infty(\vv)$ be a vector denoting the sign of the largest (absolute) value in $\vv$, such that
$$\mathbb{I}_\infty(\vv)_d=
\begin{cases}
\sgn(v_d)&\text{ if } d=\argmax_{d'}|v_{d'}|\\
0&\text{ else,}
\end{cases}$$
so that $\|\vv\|_\infty=\mathbb{I}_\infty(\vv)^\top\vv$. Furthermore, let $L_\infty(\alphahv):=\|\yv-\K\alphahv\|_\infty=\mathbb{I}_\infty(\yv-\K\alphahv)^\top(\yv-\K\alphahv)$. Then $L_\infty'(\alphahv)=-\K\cdot\mathbb{I}_\infty(\yv-\K\alphav)$ and higher order derivatives are zero. With $\dalphahv=\eta\cdot\mathbb{I}_\infty(\yv-\K\alphahv)$, according to Taylor's theorem, for $\eta\leq \min_{1\leq i\leq n}\left(\left|\yv-\K\alphahv_k\right|_i\right)$
\begin{equation*}
\begin{aligned}
L_\infty(\alphahv_{k+1})-L_\infty(\alphahv_k)
&=L_\infty(\alphahv_k+\dalphahv_k)-L_\infty(\alphahv_k)
=L_\infty(\alphahv_k)+L'(\alphahv_k)^\top\dalphahv\\
&=-\eta\cdot\mathbb{I}_\infty(\yv-\K\alphav)^\top\K\cdot\mathbb{I}_\infty(\yv-\K\alphahv)\leq0.\\
\end{aligned}
\end{equation*}
\end{proof}

\begin{proof}[Proof of Proposition \ref{thm:equ_f}]~\\
Let $\fpv:=[\fhv^\top,\fhsv{^\top}]^\top$, $\ypv:=[\yv^\top,\ytv^\top]^\top$, and let
$$\Ih:=\begin{bmatrix}\I_{n\times n} & \nv_{n\times \ns}\end{bmatrix}\in \R^{n\times (n+\ns)}$$
denote the training data selection matrix, so that $\fv=\Ih\fpv$ and $\K=\Ih\cdot[\K^\top,\Ks{^\top}]^\top=\Ih\Kss\Ih^\top$.

Differentiating Equation \ref{eq:equ_f1} with respect $\fpv$ and setting the gradient to $\nv$, we obtain
\begin{equation*}
\begin{aligned}
\label{eq:fh_temp}
\nv&=\frac{\partial}{\partial \fhpv}\left(\frac12\left\|\yv-\Ih\fpv\right\|_2^2+\frac\lambda2\|\fpv\|^2_{(\Kss)^{-1}}\right)\\
&=\Ih^\top\left(\Ih\fhpv-\yv\right)+\lambda(\Kss)^{-1}\fhpv\\
\iff\fhpv&=\left(\Ih^\top\Ih+\lambda(\Kss)^{-1}\right)^{-1}\Ih^\top\yv.
\end{aligned}
\end{equation*}
According to the matrix inversion lemma,
$$(\D-\C\A^{-1}\B)^{-1}\C\A^{-1}=\D^{-1}\C(\A-\B\D^{-1}\C)^{-1}.$$
For $\A=\I$, $\B=\Ih$, $\C=\Ih^\top$ and $\D=\lambda(\Kss)^{-1}$, we obtain
\begin{equation*}
\begin{aligned}
\left(\lambda(\Kss)^{-1}+\Ih^\top\Ih\right)^{-1}\Ih^\top&=1/\lambda\underbrace{\Kss\Ih^\top}_{=[\K^\top,\Ks{^\top}]^\top}\left(\I+1/\lambda\underbrace{\Ih\Kss\Ih^\top}_{=\K}\right)^{-1}\\
&=\begin{bmatrix}\K\\\Ks\end{bmatrix}(\lambda\I+\K)^{-1},
\end{aligned}
\end{equation*}
and thus
$$\fhpv=\begin{bmatrix}\K\\\Ks\end{bmatrix}\left(\K+\lambda\I\right)^{-1}\yv.$$

Before differentiating Equation \ref{eq:equ_f2} with respect to $\fpv$, we first note that according to the definition of $\ytv$,
$$\ypv-\fpv=\begin{bmatrix}\yv-\Ih\fpv\\\nv\end{bmatrix}=\begin{bmatrix}\I_{n\times n} \\ \nv_{\ns\times n} \end{bmatrix}(\yv-\Ih\fpv)=\Ih^\top\yv-\Ih^\top\Ih\fpv.$$
Now, differentiating Equation \ref{eq:equ_f2} with respect $\fpv$ and setting the gradient to $\nv$ we obtain
\begin{equation*}
\begin{aligned}
\nv&=\frac{\partial}{\partial \fhpv}\left(\frac12\left\|\ypv-\fhpv\right\|_{\Kss}^2+\frac\lambda2\|\fhpv\|^2_2\right)\\
&=-\Kss\ypv+\Kss\fhpv+\lambda\fhpv\\
&=-\Kss(\ypv-\fhpv)+\lambda\fhpv
=-\Kss\left(\Ih^\top\yv-\Ih^\top\Ih\fhpv\right)+\lambda\fhpv\\
&=-\Kss\Ih^\top\yv+(\Kss\Ih^\top\Ih+\lambda\I)\fhpv\\
\iff\fhpv&=\left(\begin{bmatrix}\K\\\Ks\end{bmatrix}\Ih+\lambda\I\right)^{-1}\begin{bmatrix}\K\\\Ks\end{bmatrix}\yv,
\end{aligned}
\end{equation*}

Using the matrix inversion lemma,
$$(\D-\C\A^{-1}\B)^{-1}\C\A^{-1}=\D^{-1}\C(\A-\B\D^{-1}\C)^{-1},$$
with $\A=\I$, $\B=\Ih$, $\C=[\K^\top,\Ks{^\top}]^\top$ and $\D=\lambda\I$, we obtain
\begin{equation*}
\begin{aligned}
\left(\lambda\I+\begin{bmatrix}\K\\\Ks\end{bmatrix}\Ih\right)^{-1}\begin{bmatrix}\K\\\Ks\end{bmatrix}&=1/\lambda\begin{bmatrix}\K\\\Ks\end{bmatrix}\left(\I+1/\lambda\underbrace{\Ih\begin{bmatrix}\K\\\Ks\end{bmatrix}}_{=\K}\right)^{-1}\\
&=\begin{bmatrix}\K\\\Ks\end{bmatrix}(\lambda\I+\K)^{-1},
\end{aligned}
\end{equation*}
and thus
$$\fhpv=\begin{bmatrix}\K\\\Ks\end{bmatrix}\left(\K+\lambda\I\right)^{-1}\yv.$$
\end{proof}

\begin{proof}[Proof of Equation \ref{eq:kgf_s_f}]~\\
With $\etahv(t)=\fhv(t)-\yv$ we obtain $\frac{d\etahv(t)}{dt}=\frac{d\fhv(t)}{dt}$ and the first part of Equation \ref{eq:kgd_diff_eq_f} can be written as
$$\frac{d\etahv(t)}{dt}=-\K\etahv(t)\iff \etahv(t)=\exp\left(-t\K\right)\etahnv.$$
Now,
$$\fhv(0)=\bm{0}\implies\etahnv=\etahv(0)=-\yv\implies \etahv(t)=-\exp(-t\K)\yv$$
Solving for $\fhv(t)$, we obtain
$$\fhv(t)=\left(\I-\exp(-t\K)\right)\yv.$$

Finally, 
\begin{equation*}
\begin{aligned}
\frac{d\fhsv(t)}{dt}&=\Ks\left(\yv-\fhv(t)\right)=\Ks\exp(-t\K)\yv\\
\implies \fhsv(t)&=\bm{c}-\Ks\K^{-1}\exp(-t\K)\yv.
\end{aligned}
\end{equation*}
Now,
$$\fhsv(0)=\bm{0}\implies\bm{c}=\Ks\K^{-1}\yv\implies \fhsv(t)=\Ks\K^{-1}(\I-\exp(-t\K))\yv.$$
\end{proof}

\begin{proof}[Proof of Equation \ref{eq:kgf_s_f} with momentum and Nesterov accelerated gradient]~\\
Analogous to the case of $\alphahv$ in the proof of the Remark \ref{rm:nest}, for momentum and Nesterov accelerated gradient, Equation \ref{eq:kgd_diff_eq_f} generalizes into
$$(1-\gamma)\cdot\begin{bmatrix}\frac{d\fhv(t)}{dt}\\\frac{d\fhv(t)}{dt}\end{bmatrix}=\begin{bmatrix}\K\\\Ks\end{bmatrix}\left(\yv-\fhv(t)\right).$$
Solving the differential equations in the same way as for Equation \ref{eq:kgf_s_f}, we obtain
$$\begin{bmatrix}\fhv(t)\\\fhsv(t)\end{bmatrix}=\begin{bmatrix}\I\\\Ks\K^{-1}\end{bmatrix}\left(\I-\exp\left(-\frac t{1-\gamma}\K\right)\right)\yv.$$
\end{proof}

\begin{proof}[Proof of Proposition \ref{thm:sgn_inf_f}]~\\
When $\Kss$ is a diagonal matrix with elements $\{k_{ii}\}_{i=1}^{n+\ns}$, $k_{ii}>0$,
$$\left\|\ypv-\fpv\right\|_{\Kss}^2=\sum_{i=1}^{n+\ns}\left(k_{ii}(\yp_i-\fp_i)\right)^2.$$
The constraint $\|\fpv\|_\infty=\max_i|\fp_i| \leq c$ is equivalent to $|\fp_i| \leq c\text{ for } i=1,2,\dots n+\ns$.\\
Thus, $$\left\|\ypv-\fpv\right\|_{\Kss}^2 \text{ s.t. } \|\fpv\|_\infty\leq c \iff\sum_{i=1}^{n+\ns}\left(k_{ii}(\yp_i-\fp_i)\right)^2 \text{ s.t. } |\fp_i|\leq c\ \forall i,$$
which decomposes element-wise to
\begin{equation}
\label{eq:fs_opt1}
\fhp_i=\argmin_{\fp_i\in \R}\left(k_{ii}(\yp_i-\fp_i)\right)^2 \text{ s.t. } |\fp_i|\leq c.
\end{equation}
We solve Equation \ref{eq:fs_opt1} separately for the two cases $i\leq n$ and $i>n$. For $i>n$, we have defined $\yp_i$ to be a copy of $\fp_i$, which means that the reconstruction error is always 0 and Equation \ref{eq:fs_opt1} reduces to
\begin{equation}
\label{eq:fs_opt2}
\begin{aligned}
\fhp_i&=\argmin_{\fp_i\in \R}\left(k_{ii}(\yp_i-\fp_i)\right)^2 \text{ s.t. } |\fp_i|\leq c\\
&=\argmin_{\fp_i\in \R}\left(k_{ii}(\yp_i-\fp_i)\right)^2 +\lambda_i|\fp_i|
=\lambda_i|\fp_i|,
\end{aligned}
\end{equation}
for $\lambda_i=\max(0,k_{ii}(\yp_i-c))$, where the first equality is due to Lagrangian duality, and the second is due to the reconstruction error being 0.

First assume $\lambda_i=0$, in which case Equation \ref{eq:fs_opt2} reduces to the ill-posed problem $\fh_i=\argmin_{\fp_i\in \R} 0$. However, in this case, we can use Equation \ref{eq:krr_s_fs} with $\lambda=0$. Since a diagonal $\Kss$ implies $\Ks=\nv$, we obtain, for $i>n$, $\fhp_i=0$. If, on the other hand, $\lambda_i>0$, then $\fhp_i=\argmin_{\fp_i\in \R} \lambda_i|\fp_i|=0$. Thus, for $i>n$, regardless of $\lambda_i$, $\fhp_i=0$.

For $i\leq n$, we have $\yp_i=y_i$. Assume $y_i\geq 0$. Then, the optimal value for $\fp_i$ is $y_i$, unless $y_i>c$, then, due to convexity, the optimal value is $c$, i.e.\ $\fhp_i=\min(y_i,c)$. Accounting also for the case of $y_i< 0$, we obtain
$$\fhp_i(c)=
\begin{cases}
\sgn(y_i)\cdot\min(|y_i|,c) &\text{if }i\leq n\\
0 &\text{if }i> n.\\
\end{cases}
$$

The sign gradient flow solution is calculated as follows:
\begin{equation*}
\begin{aligned}
&\sgn\left(-\frac{\partial}{\partial\fhpv(t)}\left(\left\|\ypv-\fhpv(t)\right\|_{\Kss}^2\right)\right)=\sgn\left(\Kss(\ypv-\fhpv(t))\right)\\
&=\sgn\left(\begin{bmatrix}\vrule\\k_{ii}(\yp_i-\fhp_i(t))\\ \vrule\end{bmatrix}\right)
=\sgn\left(\begin{bmatrix}[k_{ii}(y_i-\fh_i(t))]_{i=1}^n \\\nv\end{bmatrix}\right)\\
&=\begin{bmatrix}[\sgn(k_{ii}(y_i-\fh_i(t)))]_{i=1}^n\\\nv\end{bmatrix},
\end{aligned}
\end{equation*}
where $[\sgn(k_{ii}(y_i-\fh_i(t)))]_{i=1}^n\in \R^n$ and $\nv\in \R^{\ns}$.
Since element $i$ in the vector depends on $\fhpv$ only through $\fhp_i$, and since $k_{ii}>0$,
\begin{equation*}
\begin{aligned}
\frac{\partial\fhpv(t)}{\partial t}&=\sgn\left(-\frac{\partial}{\partial\fhpv(t)}\left(\left\|\ypv-\fhpv(t)\right\|_{\Kss}^2\right)\right)\\
\iff\frac{\partial \fhp_i(t)}{\partial t}
&= \begin{cases}
\sgn(k_{ii}(y_i-\fh_i(t))) &\text{if }i\leq n\\
0 &\text{if }i> n
\end{cases}\\
&= \begin{cases}
1&\text{if }i\leq n \text{ and } y_i>\fhp_i(t)\\
-1&\text{if }i\leq n \text{ and } y_i<\fhp_i(t)\\
0&\text{if }i>n \text{ or } y_i=\fhp_i(t).
\end{cases}
\end{aligned}
\end{equation*}
Thus, since $\fhp_i(0)=0$, $\fhp_i(t)=0$, for $i>n$. For $i\leq n$, $\fhp_i(t)=\pm t$, depending on the sign of $y_i$. Once $\fhp_i=y_i$ (and $\sgn(k_{ii}(y_i-\fhp_i)=0$), then $\fhp_i=y$:
$$\fhp_i(t)=
\begin{cases}
\sgn(y_i)\cdot\min(t,|y_i|) &\text{if }i\leq n\\
0 &\text{if }i> n.\\
\end{cases}
$$
\end{proof}
\end{appendix}

\newpage
\clearpage
\newpage
\bibliography{refs}
\bibliographystyle{apalike}

\end{document}